\documentclass[journal,transmag]{IEEEtran}

\usepackage{mathrsfs,amsthm,bm,enumerate}
\usepackage[bottom]{footmisc}
\usepackage{setspace}
\usepackage{multirow}
\usepackage{graphicx}
\usepackage{braket}
\usepackage{tabularx}
\usepackage{multirow,booktabs}
\usepackage{subfigure}
\usepackage{rotating}
\usepackage{cite}
\usepackage{color}
\usepackage{epstopdf}
\usepackage{float}
\usepackage{amsmath}
\usepackage{amsfonts}
\usepackage{algorithm}
\usepackage{algorithmic}
\usepackage{array}

\ifCLASSOPTIONcompsoc
 \usepackage[caption=false,font=normalsize,labelfont=sf,textfont=sf]{subfig}
\else
 \usepackage[caption=false,font=footnotesize]{subfig}
\fi

\newtheorem{theorem}{Theorem}[section]

\newtheorem{lemma}[theorem]{Lemma}

\newtheorem{definition}[theorem]{Definition}

\def\S{\mathcal{S}}
\def\P{{\mathscr P}}

\def\T{\mathcal{T}}
\def\Q{\mathcal{Q}}

\def\X{{\mathcal{X}}}

\def\Y{\mathcal{Y}}

\def\H{{\mathcal{H}}}

\def\T{\mathcal{T}}

\def\0{{\bf 0}}

\begin{document}
  \graphicspath{{./PIC/},{./FIG/}}
\title{{``Sparse + Low-Rank''} Tensor Completion Approach for Recovering Images and Videos}


\author{\IEEEauthorblockN{Chenjian Pan\IEEEauthorrefmark{1} and
Chen Ling\IEEEauthorrefmark{1} and Hongjin He\IEEEauthorrefmark{2} and Liqun Qi\IEEEauthorrefmark{3,4} and Yanwei Xu\IEEEauthorrefmark{3}
}
\IEEEauthorblockA{\IEEEauthorrefmark{1}Department of Mathematics, School of Science, Hangzhou Dianzi University,
Hangzhou, 310018, China.}
\IEEEauthorblockA{\IEEEauthorrefmark{2} School of Mathematics and Statistics, Ningbo University, Ningbo 315211, China.}
 \IEEEauthorblockA{\IEEEauthorrefmark{3}Theory Lab, Central Research Institute, 2012 Labs, Huawei Technologies Co., Ltd., Shatin, New Territory, Hong Kong, China. }
 \IEEEauthorblockA{\IEEEauthorrefmark{4}Department of Applied Mathematics, The Hong Kong Polytechnic University, Hung Hom, Kowloon, Hong Kong, China.}
\thanks{Manuscript received October XX, 2021; revised July XX, 2022.
Corresponding author: H. He (email: hehongjin@nbu.edu.cn).}}

\markboth{Manuscript,~Vol.~xx, No.~xx, July~2022}%
{Pan \MakeLowercase{\textit{et al.}}: ``'Sparse + Low-Rank' Tensor Completion Approach}
%



\IEEEtitleabstractindextext{%
\begin{abstract}
Recovering color images and videos from highly undersampled data is a fundamental and challenging task in face recognition and computer vision. By the multi-dimensional nature of color images and videos, in this paper, we propose a novel tensor completion approach, which is able to efficiently explore the sparsity of tensor data under the discrete cosine transform (DCT). Specifically, we introduce two {``sparse + low-rank''} tensor completion models as well as two implementable algorithms for finding their solutions. {The first one is a DCT-based sparse plus weighted nuclear norm induced low-rank minimization model. The second one is a DCT-based sparse plus $p$-shrinking mapping induced low-rank optimization model}. Moreover, we accordingly propose two implementable augmented Lagrangian-based algorithms for solving the underlying optimization models. A series of numerical experiments including color image inpainting and video data recovery demonstrate that our proposed approach performs better than many existing state-of-the-art tensor completion methods, especially for the case when the ratio of missing data is high.
\end{abstract}

}

\maketitle

\IEEEdisplaynontitleabstractindextext

%
\IEEEpeerreviewmaketitle

\section{Introduction}\label{Intro}
\IEEEPARstart{W}{ith} the rapid developments of sensor technologies and video surveillance systems, the collected images and videos are naturally stored as multi-way arrays, which are also called tensors. As we know, the complete information of images and videos is very important to make a good decision for intelligent devices or data users. However, during the acquisition process, observed image and video data often contain missing entries (or pixels). In this situation, estimating these missing information of images and videos, which is also called tensor completion, is a fundamental and challenging problem in the communities of image processing and computer vision, e.g., see \cite{K06,HKBH13,CPT04,GSZW11} and references therein. 

It is well-known that tensor is a higher-order extension of matrix, so the tensor completion is also a natural generalization of matrix completion. Generally speaking, when the missing entries are sparse in an incomplete matrix, we could efficiently explore the local information around the unknown components to get a relatively ideal completion to the matrix, e.g., see \cite{AK99,DAB99,GL13}. However, when dealing with a highly undersampled matrix, it is incredibly difficult, even if not impossible, to accurately estimate these missing entries from an incomplete matrix without any global or prior information. Comparatively, the complex structure of tensors makes that recovering the missing entries from a highly undersampled tensor puts forward more theoretical and computational challenges at the interface of statistics and optimization, e.g., see \cite{SGCH19}. In the past decades, it is well-documented that exploiting the inherent global information, e.g., low-rank and sparsity, of the data is able to greatly improve the estimation quality of highly undersampled matrices and tensors, e.g., see \cite{LMWY13, GRY11, LFLY18,ZLLZ18, CHL14}, to name just a few.

As a direct extension of the low-rank matrix completion, the canonical {\it low-rank tensor completion} (LRTC) model is expressed mathematically as
\begin{equation}\label{LRTC-optim}
\begin{array}{cl}
\min\limits_{\mathcal{X}} &{\rm rank}(\mathcal{X})\\
\text{s.t.}&{\mathscr P}_{\Omega}(\mathcal{X})={\mathscr P}_{\Omega}(\mathcal{H}),
\end{array}
\end{equation}	
where both $\mathcal{X}$ and $\mathcal{H}$ are $N$-th order tensors, ${\rm rank}(\cdot)$ denotes the rank function, $\Omega$ is the index set corresponding to the observed entries of the incomplete tensor $\mathcal{H}$, and ${\mathscr P}_{\Omega}(\cdot)$ is the linear operator that keeps known elements in $\Omega$ while setting the others to be zeros. It has been documented in \cite{CR09,HL13} that directly minimizing the rank function of tensors, even for matrices, is an NP-hard problem. Moreover, unlike the unique definition of rank for matrices, there are diverse definitions for tensors such as CANDECOMP/PARAFAC (CP) rank, Tucker rank and Tensor Train (TT) rank (see \cite{KB09, Ose11}). Consequently, the diversity of tensor ranks often makes engineers difficult to choose an appropriate surrogate for characterizing the low-rankness of their problems. 

Roughly speaking, most state-of-the-art tensor completion methods can be grouped into two categories. (i) The first type refers to the tensor nuclear norm minimization approach, which unfolds the underlying tensor as a series of matrices so that we could efficiently employ the matrix nuclear norm as an approximation to the rank of a tensor. Therefore, many efficient algorithms tailored for matrix completion could be naturally extended to tensor completion based on the so-called tensor nuclear norm definitions, e.g., see \cite{JHZJD18, LMWY13, LFLY18, HOM16, HSCW14, BPTD17, JHZML16,XZLC19} and references therein. Mathematically, a seminal tensor nuclear norm minimization model introduced in \cite{LMWY13} takes the form
\begin{equation}\label{LNNTC-optim0}
\begin{array}{cl}
\min\limits_{\mathcal{X}}&\sum_{n=1}^N\alpha_n\|X_{(n)}\|_*\\
\text{s.t.}&{\mathscr P}_{\Omega}(\mathcal{X})={\mathscr P}_{\Omega}(\mathcal{H}),
\end{array}
\end{equation}
where $\alpha_n\geq 0~(n\in [N]:=\{1,2,\cdots,N\})$ can be regarded as weight parameters often satisfying $\sum_{n=1}^N\alpha_n\approx 1$,  matrix $X_{(n)}$ corresponds to the mode-$n$ unfolding of tensor $\mathcal{X}$ for every $n\in [N]$, and $\|\cdot\|_*$ represents the well-known nuclear norm referring to the sum of all singular values of a matrix. Note that a vector consisting of the ranks of unfolding matrices $X_{(n)}$'s $(n=1,2,\cdots N)$ is called the Tucker rank of the tensor $\mathcal{X}$ (see \cite{KB09}). Hence, the objective function in \eqref{LNNTC-optim0} is indeed a convexly weighted sum of the components of Tucker rank. 
Compared with the direct vectorization and matricization approches for tensor data, model \eqref{LNNTC-optim0} efficiently exploits the multi-mode structure, thereby possibly achieving a better approximation to the low-rankness of tensor data. Moreover, these unfolding completion models (i.e., tensor nuclear norm minimization models) are beneficial for algorithmic design due to the convexity of tensor nuclear norm functions. Recently, the state-of-the-art {\it alternating direction method of multipliers} (ADMM) and Douglas-Rachford splitting method have been successfully applied to tensor completion, e.g., see \cite{LMWY13,GRY11,ZA17, LFCLLY16, LFCLLY20}. 
(ii) The second type of tensor completion approach is based on tensor decomposition. Actually, we observe that many tensor nuclear norm based models utilize the so-called unfolding (or matricization) technique. However, Yuan and Zhang \cite{YZ16} showed that the matricization for higher-order tensors ignores the nature of tensors and potentially destroys some inherent properties of the data, thereby leading to suboptimal procedure, which encourages researchers to develop new completion methods from tensor decomposition perspectives. In the literature, some popular tensor decomposition techniques include CP decomposition, Tucker decomposition \cite{KB09}, T-SVD \cite{KM11} and Tensor-Train decomposition \cite{Ose11}. Like the matrix decomposition, the aforementioned tensor decomposition forms have also been verified as powerful dimensionality reduction tools, which can promote the low-rankness of tensor data to some extent. Empirically, these tensor decomposition based completion models have received great successes in big data analysis, pattern recognition, and traffic data recovery, e.g., see  \cite{ZZC15, ZLLZ18, CHL14, FJ15, YZC16, WTLZC19, LNZ20,  LSFCC16} and references therein.

When regarding color images and videos as general tensors, the aforementioned tensor completion approaches are certainly applicable to images and video inpainting problems. However, it often fails to achieve an ideal restoration quality for highly undersampled image data since only the low-rank property is exploited in general tensor completion models {(see Figs. \ref{RGB_image_visual} and \ref{RGB_image_visual_stru})}. As we know, the {\it total variation} (TV) regularization (see \cite{ROF92}) is a widely used tool in the community of image processing to preserve sharp discontinuities (edges) of an image, while removing noise and other unwanted fine scale detail. Hence, researchers judiciously incorporated TV regularization into low-rank tensor completion models for the purpose of exploiting the inherent structure of images. The promising numerical performance can be found in recent works, e.g., \cite{VMG12,JHZML16,QLCGZ20,LNZ20,ZNZJH19,JLLS18,QBNZ21}. From the mathematical modeling perspective, the underlying images {usually are sparse under the TV transform}. Consequently, we see that an $\ell_1$-norm is used to promote the sparsity of an image in the TV transform domain, which demonstrates that combining the low-rank and TV-based sparse prior information simultaneously can greatly improve the restoration quality. {However, as stated in \cite{WXYXT18}, the traditional TV regularizer is still not perfect, since it only guarantees an estimation presenting a locally smooth visualization, while ignoring the global perception and structure.  Such a drawback motives us to find new sparse surrogates to recover more accurate data. Actually,} in the image processing literature, the well-known {\it discrete cosine transform} (DCT) has been widely used for the most popular image compression standard JPEG, since DCT is powerful to obtain a (at least approximately) sparse representation of an image. 
For example, we consider a widely used gray video, i.e., \textbf{sidewalk} (see Fig. \ref{facade_sparse}), which is a third order tensor of size $220\times 352\times 30$. By performing DCT on each dimension of such a gray video, a large number of elements of approach to zeros, thereby leading to a high level sparsity when setting a truncated number (TN), e.g., $0.05$ and $0.1$, to drop all values less than such a TN. However, the video recovered directly from the truncated tensor is still close to the original one, which means that a small sample information is possible to recover an ideal (at least identifiable) one. Such an experimental observation also encourages us to consider a sparsity regularization on the coefficient tensor under the DCT procedure. Here, we refer the reader to \cite{AI07,DHQYWCZ19,HLHS17,LSS13,WXYXT18,MG18,ZW91} for recent applications of DCT on image recovery {and some completion methods considering the property of the tensor data in transformed domains, e.g., see \cite{JHZJD18,JNZH20,JZZN21,ZXJWN20}}. Compared to the TV-based tensor completion model, to the best of our knowledge, DCT-based approaches received much less considerable attention on color images and videos inpainting from tensor completion angle. Hence, we are motivated to make a further study on showing the ability of the DCT technique for recovering images and videos from highly undersampled data.
\begin{figure}[!htbp]
	\centering
	\includegraphics[width=0.15\textwidth]{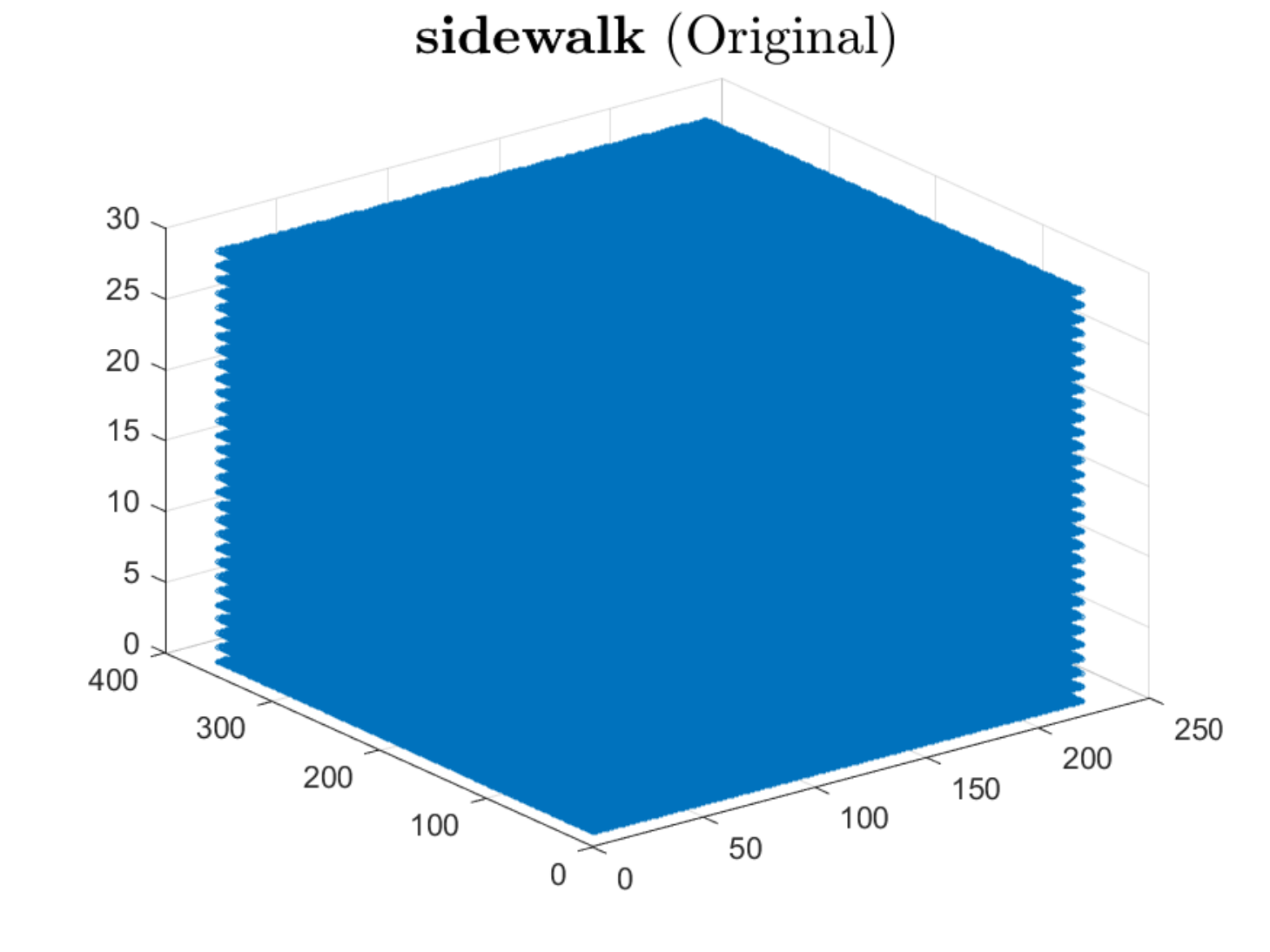}
	\includegraphics[width=0.15\textwidth]{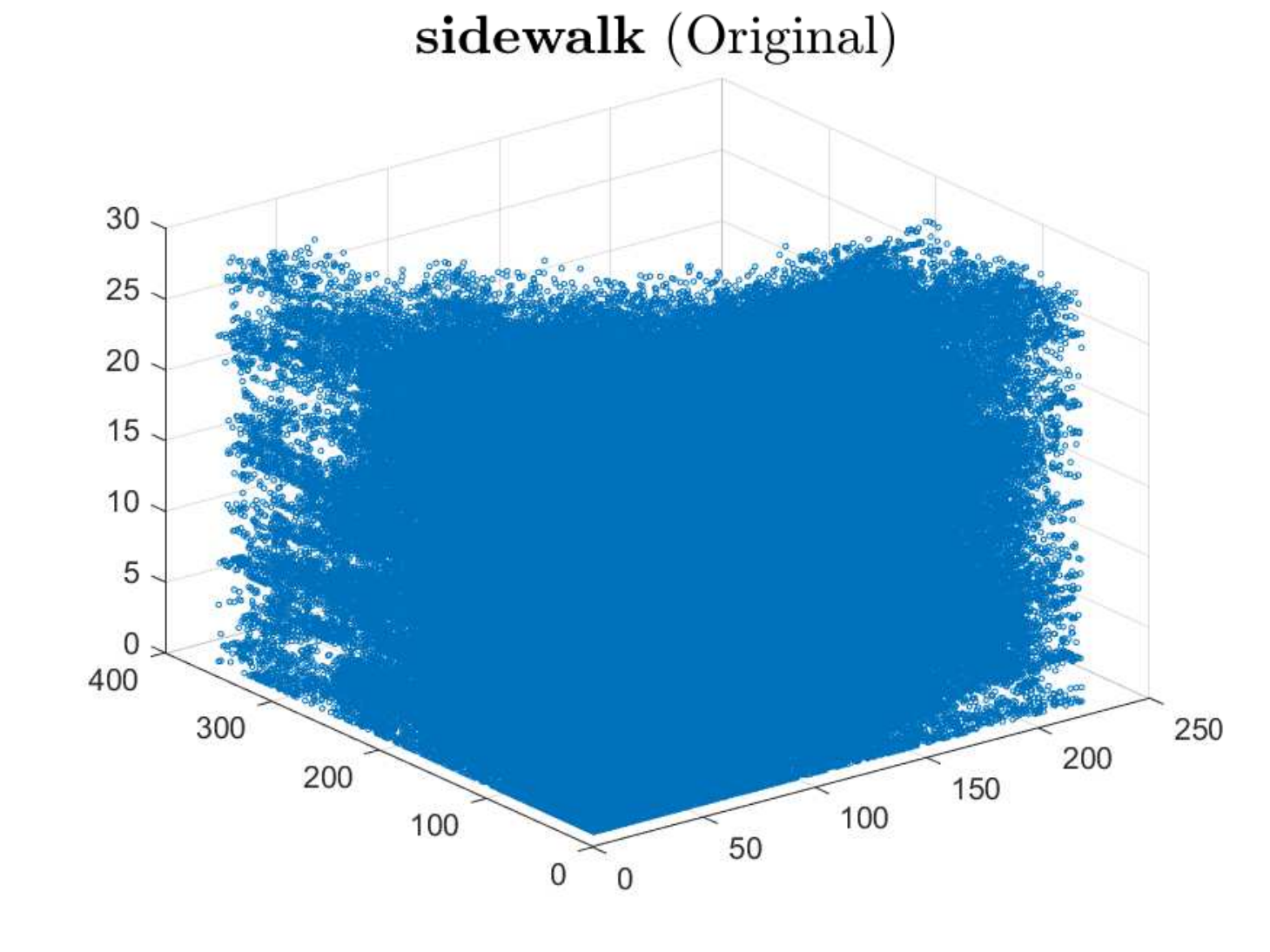}
	\includegraphics[width=0.15\textwidth]{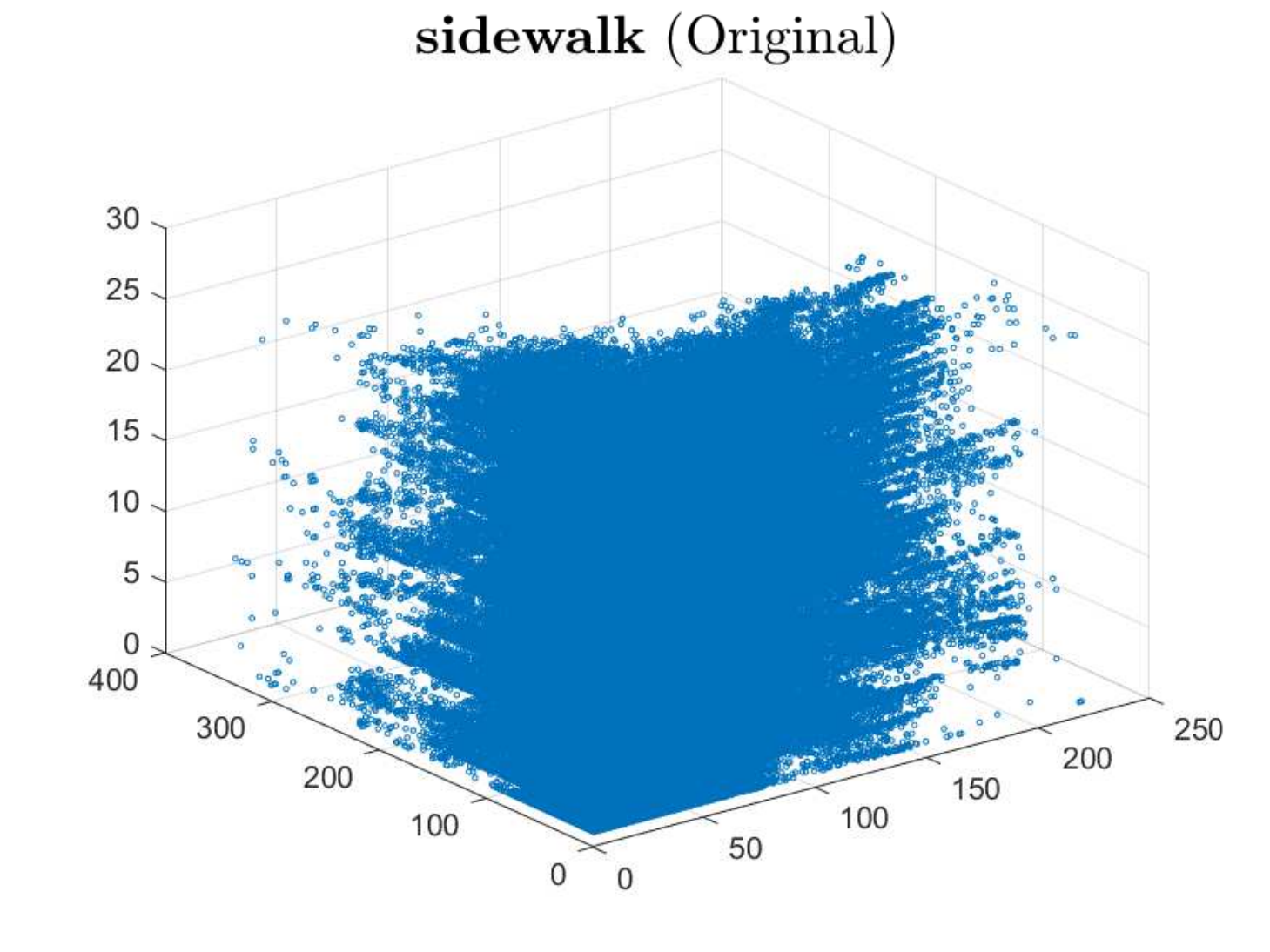}\\
		\text{~~~~~~~~Original~~~~~~~~~~~~~~TN=0.05~~~~~~~~~~~~~TN=0.1~~~~~~~~}
	\includegraphics[width=0.15\textwidth]{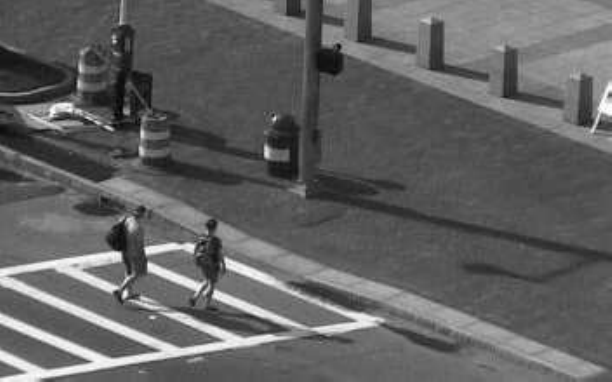}
	\includegraphics[width=0.15\textwidth]{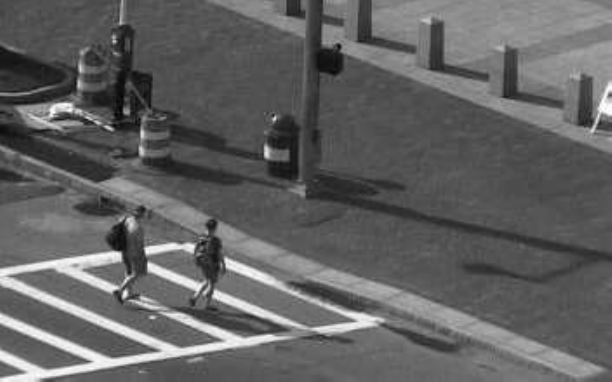}
	\includegraphics[width=0.15\textwidth]{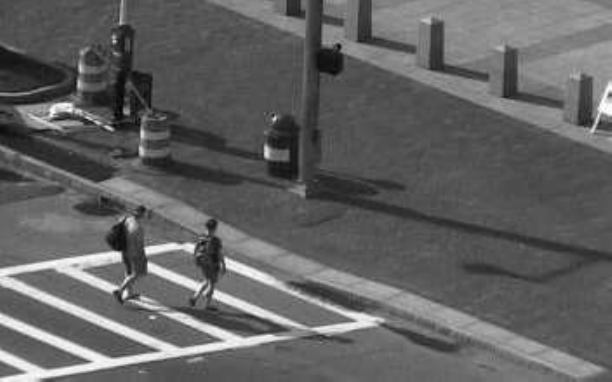}
	\text{~~~~~~~~Original~~~~~~~~~~~PSNR 35.24~~~~~~~~~PSNR 31.28~~~~~~~~}
		\caption{The first row show the visualization of the sparsity of the gray video \textbf{sidewalk} under DCT, where the white represents $0$. The second row illustrates the comparisons of the 10-th frame of the original video and the frames reconstructed from the truncated tensor.}
		\label{facade_sparse}
\end{figure}

In this paper, we introduce {two ``sparse + low-rank''} tensor completion approaches for recovering color images and videos. Our contribution is four-fold.
\begin{itemize}
\item Most nuclear norm based tensor completion models equally treat all singular values of the unfolding matrices. For some real-world data sets, small singular values are often sensitive to the incomplete (or corrupted by noise) observed information, thereby possibly leading to a suboptimal low-rank tensor. In this situation, we will propose {a minimization model, whose objective function is the sum of a DCT-based sparse regularization term and a weighted nuclear norm induced low-rank regularization term so that the singular values have different weights to enhance the global low-rank structure}.
\item In many unfolding approaches, the most popular low-rank surrogate of tensors is the nuclear norm of unfolding matrices. Recently, both theoretical and empirical results show that the so-called $p$-shrinkage thresholding algorithm outperforms the classical iterative soft thresholding algorithm induced by nuclear norm for low-rank and sparse recovery problems, e.g., see \cite{Ch14,Char09,SLH19}. Accordingly, the second contribution of this paper {is that we introduce a nonconvex optimization model, which minimizes the sum of a DCT-based sparse term and a $p$-shrinking mapping induced low-rank term. }
\item We develop two implementable augmented Lagrangian-based algorithms for the proposed tensor completion models, where all subproblems of our algorithms have closed-form solutions.
\item  We conduct a series of experiments on color images and surveillance videos. Our computational results demonstrate that our proposed {``sparse + low-rank'' tensor completion approaches outperform many state-of-the-art model-driven tensor completion methods, especially for highly undersampled cases.}
\end{itemize}
 
 The structure of this paper is as follows. In Section \ref{NarPrel}, we summarize some notations and recall some basic definitions including (weighted) nuclear norm of matrices, $p$-shrinkage mapping and proximal operators. In Section \ref{ModAlg}, we will first introduce a unified  ``sparse + low-rank''  tensor completion approach. Then, we split this section into two parts and propose two kinds of ``sparse + low-rank'' tensor completion models, respectively. Based on the augmented Lagrangian function, we also propose two implementable algorithms for the underlying tensor completion models. Numerically,  in Section \ref{ExpTests}, we conduct the performance of our approaches on images and videos recovery from highly undersampled data. Finally, some concluding remarks are provided in Section \ref{ConRemark}.
 
\section{Notations and Preliminaries}\label{NarPrel}
In this section, we summarize some notations and definitions on $p$-shrinkage mapping and weighted nuclear norm that will be used throughout this paper.

Tensor is a multidimensional array, which is an extension of matrix. The space of all $N$-th order real tensors is denoted by $\mathbb{R}^{I_1\times I_2\times\cdots\times I_N}$, where the order of a tensor is also called way or mode. Given an $N$-th order $\mathcal{A}\in \mathbb{R}^{I_1\times I_2\times\cdots\times I_N}$, we denote the $(i_1, i_2, \ldots, i_N )$-th component of $\mathcal{A}$ by $a_{i_1i_2\ldots i_N}$. So the $N$-th order tensor $\mathcal{A}$ is also denoted by $\mathcal{A}=(a_{i_1i_2\ldots i_N})$. Throughout this paper, tensors of order $N\geq 3$ are denoted by calligraphical letters, e.g., $\mathcal{A}, \mathcal{B},\ldots$. Generally, we use capital letters (e.g., $A, B,\ldots$), boldfaced lowercase letters (e.g., ${\bm a}, {\bm b},\ldots$), and lowercase letters (e.g., $a,b,\ldots$) to denote matrices, vectors, and scalars, respectively. For any two $N$-th order real tensors $\mathcal{A}=(a_{i_1i_2\ldots i_N})$ and $\mathcal{B}=(b_{i_1i_2\ldots i_N})$, the inner product between $\mathcal{A}$ and $\mathcal{B}$ is given by 
$$\langle\mathcal{A},\mathcal{B}\rangle:=\sum_{i_1,i_2,\ldots,i_N}a_{i_1i_2\ldots i_N} b_{i_1i_2\ldots i_N}.$$ 
Consequently, the Frobenius norm of tensor $\mathcal{A}$ associated with the above inner product is given by $\|\mathcal{A}\|_F=\sqrt{\langle\mathcal{A},\mathcal{A}\rangle}$. Given an $N$-th order tensor $\mathcal{A}$, the mode-$n$ matricization (or unfolding) of  $\mathcal{A}$ is denoted by $A_{(n)}$, and the $(i_1,i_2,\ldots,i_N)$-th entry of tensor $\mathcal{A}$ is mapped to the $(i_n,j)$-th entry of matrix $A_{(n)}$  in the lexicographical order, where
$$
j=1+\sum_{1\leq l\leq N,l\neq n} (i_l-1)J_l~~~~{\rm with}~~~\displaystyle J_l=\prod_{1\leq t\leq l-1,t\neq n}I_t.
$$

Given a matrix $A\in \mathbb{R}^{m\times n}$, we write its {\it singular value decomposition} (SVD) as $A=U\Lambda V^\top$, where $U$ and $V$ are orthogonal matrices and $\Lambda$ is a diagonal matrix, i.e., 
$$\Lambda={\rm diag}(\sigma(A))={\rm diag}\left(\sigma_1(A),\sigma_2(A),\cdots,\sigma_r(A)\right)$$
with $r:=\min\{m,n\}$, and the diagonals $\sigma_i(A)$'s are singular values of $A$ satisfying $\sigma_1(A)\geq \sigma_2(A)\geq \cdots \geq \sigma_r(A)$ without loss of generality. With the above preparation, we first recall the definition of (weighted) nuclear norm of a matrix (see \cite{DHQYWCZ19}).
\begin{definition}\label{weighted-nuclear-norm}
	For any given matrix $A\in\mathbb{R}^{m\times n}$, its nuclear norm is defined by 
	$$\|A\|_*=\sum_{i=1}^{r} \sigma_i(A)$$
	and the weighted nuclear norm is given by
	$$\|A\|_{*,{\bm w}}=\sum_{i=1}^{r}{w}_i\sigma_i(A)$$
	where $r:=\min\{m,n\}$ and $\sigma_i(A)$ is the $i$-th largest singular value of $A$, and ${\bm w}=(w_1,\cdots,w_r)$ is a weight vector, which, as suggested in \cite{HOM16},  is given by 
	$${w}_i = \delta / (\sigma_i(A)+\epsilon),\quad i=1,2,\cdots,r,$$
	where $\delta\in\mathbb{R}_+$ is a constant and $\epsilon$ is a small positive number to avoid division by $0$. In algorithmic implementation, we can iteratively update ${w}_i$'s.
\end{definition}

Below, we recall the $p$-shrinkage mapping introduced in \cite{Char09}.
\begin{definition}\label{def_pshrink}
Given ${\bm a}\in \mathbb{R}^n$, $\mu > 0$ and $p\leq 1$, the $p$-shrinkage thresholding operator is defined in component-wise by
\begin{equation}\label{pshrink}
{\rm \bf pshrink}({\bm a}, \mu, p) = \text{sign}({\bm a})\odot \max\left\{|{\bm a}|-\mu|{\bm a}|^{p-1},\;0\right\},
\end{equation}
where `${\rm sign}(\cdot)$' and `$|\cdot |$' are the sign function and absolute value function in component-wise, respectively, and `$\odot$' represents the component-wise product between two vectors. In particular, when setting $p=1$, the $p$-shrinkage operator \eqref{pshrink} immediately reduces to the well-known soft-thresholding, which is denoted by
\begin{equation}\label{soft-shrink}
{\rm shrink}({\bm a}, \mu) = \text{sign}({\bm a})\odot \max\left\{|{\bm a}|-\mu,\;0\right\}.
\end{equation}
\end{definition}
Clearly, the smaller $p$ is, the less ${\rm shrink}({\bm a}, \mu, p) $ shrinks large inputs. We refer the reader to \cite{SLH19} for an illustration to the philosophy of the $p$-shrinkage operator. 

For any given ${\bm a}\in\mathbb{R}^n$, the proximal operator associated with $\theta:\mathbb{R}^n\to\mathbb{R}$ about $\tau>0$ is defiend as 
	$$\text{prox}_{\theta,\tau}({\bm a}):= \arg\min_{{\bm x}\in\mathbb{R}^n}\left\{\theta({\bm x})+\frac{\tau}{2}\|{\bm x}-{\bm a}\|^2\right\}.$$
Consequently, as shown in \cite[Theorem 4]{GL11} (see also \cite{LTYL16}), for any $\mu>0$, $A\in {\mathbb R}^{m\times n}$ and a weight vector ${\bm w}$, the global optimal solution to the following optimization problem
\begin{equation*}
\min_{X} \|X\|_{*,{\bm w}} + \frac{1}{2\mu}\|X-A\|^2_F
\end{equation*}
is given by the weighted SVD thresholding 
$$X^*=U{\rm \bf wshrink}(\Lambda,\mu,{\bm w})V^\top,$$
where $A=U\Lambda V^\top$ is the SVD of $A$ and 
\begin{equation}\label{wsoft-shrink}
{\rm \bf wshrink}(\Lambda, \mu,{\bm w}) = \max\left\{\sigma(A)-\mu {\bm w},\;0\right\}.
\end{equation}
Moreover, it follows from \cite{Ch14} that the $p$-shrinkage mapping defined in Definition \ref{def_pshrink} can also be interpreted as the proximal operator of a penalty function $\Phi_p^\mu(\cdot):\mathbb{R}^n\to\mathbb{R}$, i.e.,
	$${\rm \bf pshrink}({\bm a}, \mu, p)= \arg\min_{{\bm x}\in\mathbb{R}^n} \left\{\Phi_p^\mu({\bm x})+\frac{1}{2\mu}\|{\bm x}-{\bm a}\|^2\right\},$$
	where $\Phi_p^\mu({\bm x}):=\sum_{i=1}^{n}\phi_p^\mu(x_i)$ with $\phi_p^\mu(x_i)$ being even, concave, nondecreasing and continuous on $[0,\infty]$, differentiable on $(0,\infty)$, nondifferentiable at 0 with the subdifferential being $\partial \phi_p^\mu(0) = [-1,1]$. We refer the reader to \cite[Theorem 1]{Ch14} (also \cite{WC16,SLH19}) for more details. Notice that the $p$-shrinkage operator is available to matrix variables. Specifically, for the optimization problem
\begin{equation}\label{Phi-prox}
\min_{X} \Phi_p^\mu(X) + \frac{1}{2\mu}\|X-A\|^2_F,
\end{equation}
we have its globally optimal solution given by
\begin{equation}\label{solve-p}
	X^*=U{\rm \bf pshrink}(\Lambda,\mu,p)V^\top,
\end{equation}
where $A=U\Lambda V^\top$ is the SVD of $A$ and 
\begin{equation*}
{\rm \bf pshrink}(\Lambda, \mu,p) ={\rm \bf pshrink}(\sigma(A), \mu,p).
\end{equation*}

\section{Models and Algorithm}\label{ModAlg}
In this section, we aim to develop {DCT-based sparse plus nonconvex functions induced low-rank} tensor completion approaches to images and videos recovery. Mathematically, we first introduce a unified  ``{sparse + low-rank}'' tensor completion model for $N$-th order tensors as follows:
\begin{equation}\label{question-2}
\begin{array}{cl}
\min\limits_{\mathcal{X}} &{\rm rank}(\mathcal{X})+\lambda\|\mathcal{{\mathscr D}(X)}\|_0\\
\text{s.t.}& {\mathscr P}_{\Omega}(\mathcal{X}) = {\mathscr P}_{\Omega}(\H),
\end{array}
\end{equation}
where ${\mathscr D}(\cdot)$ denotes the multi-dimensional DCT operator satisfying orthogonality, $\|\cdot\|_0$ represents the so-called $\ell_0$-norm to characterize the sparsity of DCT's coefficients, and $\lambda>0$ is a tuning parameter. Due to the appearances of rank function and $\ell_0$ norm, the optimization model \eqref{question-2} is a highly non-convex and NP-hard problem. Usually, we can find an approximate solution of \eqref{question-2} by employing the (non-) convex relaxations of rank function and $\ell_0$-norm, respectively. Therefore, in this section, we first introduce a DCT-based {sparse plus weighted nuclear norm (WNN) induced} low-rank tensor completion model. Then, we propose a DCT-based {sparse plus $p$-shrinking mapping} induced low-rank tensor completion model. Additionally, we will propose two implementable augmented Lagrangian-based splitting methods to solve the underlying models.

\subsection{{DCT-based sparse + WNN induced low-rank minimization model}}
As shown in the literature, directly minimizing the sum of the rank function and $\ell_0$-norm of tensor $\mathcal{X}$ in \eqref{question-2} is not an easy task. So, how to approximate the rank of tensors is crucial for improving the completion quality. In this subsection, we use the sum of the weighted nuclear norms of the unfolding matrices to approximate the rank of $\mathcal{X}$ and employ the $\ell_1$-norm to replace the $\ell_0$-norm for promoting the sparsity of images under the DCT. Accordingly, we call the model {DCT-based sparse + WNN induced low-rank} minimization model, which is expressed as the form
\begin{equation}\label{WTNN-question-1}
\begin{array}{cl}
\displaystyle\min_{\mathcal{X}}&\displaystyle\sum_{i=1}^{N}\alpha_i\|X_{(i)}\|_{*,{\bm w}}+\lambda\|\mathcal{{\mathscr D}(X)}\|_1\\
\text{s.t.}& \P_{\Omega}(\mathcal{X}) = \P_{\Omega}(\H).
\end{array}
\end{equation}
Apparently, two parts of the objective function in model \eqref{WTNN-question-1} are nonsmooth, but both have promising structures so that 
their proximal operators enjoy explicit forms. To exploit the structure of model \eqref{WTNN-question-1}, we here first introduce auxiliary variables to separate the two nonsmooth terms in the objective function, thereby leading to a separable optimization model, i.e.,
\begin{equation}\label{ADMM-WTNN}
\begin{array}{cl}
\displaystyle\min_{\mathcal{X}} &\sum_{i=1}^{N}\alpha_i\|Y_{i,(i)}\|_{*,{\bm w}}+\lambda\|\T\|_1\\
\text{s.t.}&\mathcal{X} = \mathcal{Y}_{i},\quad  \text{for all}~i \in [N],\\
&\T = {\mathscr D}(\X),\\
&\P_{\Omega}(\X) = \P_{\Omega}(\H).
\end{array}
\end{equation}
where $Y_{i,(i)}$ represents the mode-$i$ unfolding of tensor $\mathcal{Y}_i$ for $i\in[N]$.  Clearly, model \eqref{ADMM-WTNN} is an equality-constrained optimization problem, then its augmented Lagrangian function reads as
\begin{align}\label{WTNN-Augment_lag}
&{\bm L}(\Y_i,\X,\T,\S_i,\Q) \nonumber \\
&:=\sum_{i=1}^{N}\left( \alpha_i\|Y_{i,(i)}\|_{*,{\bm w}}+\left\langle \S_i,\X-\Y_i \right\rangle +\frac{\beta}{2}\|\X-\Y_i\|_F^2 \right) \nonumber\\
&\quad +\lambda\|\T\|_1+\left\langle \Q,\T-{\mathscr D}(\X) \right\rangle+\frac{\beta}{2}\|\T-{\mathscr D}(\X)\|_F^2,
\end{align}
where $\mathcal{S}_i$ and $\mathcal{Q}$ are Lagrangian multipliers associated to $X_{(i)}=Y_{i,(i)}$ and $\T={\mathscr D}(\X)$, respectively,  and $\beta$ are positive penalty parameters. {By \eqref{WTNN-Augment_lag}, we follow the alternating spirit of ADMM to design an implementable algorithm. Theoretically, all primal variables can be updated in any order due to their equal roles. However, their updating order is slightly sensitive to the numerical performance. Notice that the weight of the low-rank term is usually greater than the sparse part, we empirically first update $\mathcal{Y}_i$ and then calculate $\mathcal{X}$ due to newly introduced equality constraint. Thirdly, we update $\mathcal{T}$ for the reason that we can immediately utilize the latest information of $\mathcal{X}$. Finally, we update the Lagrangian multipliers $\mathcal{S}_i$ and $\mathcal{Q}$ simultaneously. Therefore, }we accordingly update the variables in \eqref{WTNN-Augment_lag} via the order $\mathcal{Y}_i\to\mathcal{X}\to\mathcal{T}\to\mathcal{S}_i\to\mathcal{Q}$. Specifically, for given $(\mathcal{X}^k,\mathcal{T}^k,\mathcal{S}_i^k,\mathcal{Q}^k)$, we update the $(k+1)$-th iterates via the following process:
\begin{itemize}
\item The update of $\mathcal{Y}_i$ ($i\in[N]$) reads as
\begin{align}\label{WTNN-solve-Y}
\Y^{k+1}_i & = \arg\min_{\mathcal{Y}_i} {\bm L}\left({\Y_i,\X^k,\T^k,\S^k_i,\Q^k}\right) \nonumber \\
&=\arg\min_{\mathcal{Y}_i} \left\{\alpha_i\|Y_{i,(i)}\|_{*,{\bm w}} +\frac{\beta_k}{2}\left\|\Y_{i}-\widehat{\Y}_{i}^k\right\|_F^2\right\}\nonumber \\
&={\rm{fold}}\left(U_k{\rm \bf wshrink}\left(\Lambda^k,\; \frac{\alpha_i}{\beta_k},\;{\bm w}_i\right)V_k^\top  \right),
\end{align}
where `fold($\cdot$)' corresponds to the inverse operator of mode-$i$ unfolding, `${\rm \bf wshrink}(\cdot,\cdot,\cdot)$' is given by \eqref{wsoft-shrink} and
$$\widehat{\Y}_{i}^k=\X^k+\frac{1}{\beta_k}{\S}_{i}^k\quad\text{and}\quad\widehat{Y}_{i,(i)}^k=U_k\Lambda^k V_k^\top.$$
\item {With the latest $\Y^{k+1}_i $, we obtain $\mathcal{X}^{k+1}$ via solving 
\begin{equation*}
\min_{\mathcal{X}} \left\{{\bm L}\left({\Y^{k+1}_i,\X,\T^k,\S^k_i,\Q^k}\right) \;|\; \P_{\Omega}(\mathcal{X}) = \P_{\Omega}(\mathcal{H})\right\},
\end{equation*}	
which can be expressed explicitly by
\begin{equation}\label{WNN-X-Update}
\left\{
\begin{array}{l}
\mathcal{X}^{k+1}_{\Omega} = \mathcal{H}_{\Omega},\\
\X^{k+1}_{\Omega^c} =\frac{1}{N+1} \widehat{\mathcal{X}}^{k}_{\Omega^c}, 
\end{array}\right.
\end{equation}
where $\Omega^c$ denotes the complementary set of $\Omega$ and
$$\widehat{\mathcal{X}}^{k}= \sum_{i=1}^{N}\left(\mathcal{Y}_i^{k+1}-\frac{1}{\beta_k}\mathcal{S}_i^k\right)+{\mathscr D}^{-1}\left(\mathcal{T}^k+\frac{1}{\beta_k}\mathcal{Q}^k\right).$$}
\item  For the $\mathcal{T}$-subproblem, we have 
\begin{align}\label{WTNN-solve-T}
\T^{k+1}  &= \arg\min_{\mathcal{T}} {\bm L}\left({\Y^{k+1}_i,\X^{k+1},\T,\S^k_i,\Q^k}\right) \nonumber \\
&= \arg\min_{\mathcal{T}} \left\{ \lambda\|\mathcal{T}\|_1+\frac{\beta_k}{2}\left\|\T-\widehat{{\mathcal T}}^{k} \right\|_F^2\right\} \nonumber \\
&= \text{shrink}\left(\widehat{{\mathcal T}}^{k},\frac{\lambda}{\beta_k}\right),
\end{align}
where `shrink($\cdot,\cdot$)' is given by \eqref{soft-shrink} and
$$\widehat{{\mathcal T}}^{k}={\mathscr D}(\X^{k+1})-\frac{1}{\beta_k}\Q^k.$$
\item With the above latest $(\mathcal{Y}_i^{k+1},\mathcal{X}^{k+1},\mathcal{T}^{k+1})$, we update the two Lagrangian multipliers $\S^{k+1}_i $ and $\Q^{k+1}$ via
\begin{equation}\label{wnn-ups}
\S^{k+1}_i  = \S^{k}_i+\beta_k\left(\X^{k+1} - \Y^{k+1}_i\right) 
\end{equation}
and
\begin{equation}\label{wnn-upq}
\Q^{k+1}  =  \Q^{k} +\beta_k\left(\T^{k+1}-{\mathscr D}(\X^{k+1})\right),
\end{equation}
respectively.
\end{itemize}

With the above preparations, we formally summarize the updating schemes for model \eqref{ADMM-WTNN} in Algorithm \ref{ADMM-solve2}.
\begin{algorithm}[!htbp]
	\caption{ADMM for Model (\ref{ADMM-WTNN})}\label{ADMM-solve2}
	\hspace*{0.02in} {\bf Input:} 
	Initial starting points $\X^0,\Y^0,\T^0,\S_0,\Q_0$ and $\beta_0>0$.\\
	\hspace*{0.2in}{1: Update $\Y_i^{k+1}$ simultaneously via (\ref{WTNN-solve-Y}) for $i=[N]$};\\
	\hspace*{0.2in}{2: Update $\X^{k+1}$ via \eqref{WNN-X-Update}};\\
	\hspace*{0.2in}{3: Update $\T^{k+1}$ via (\ref{WTNN-solve-T})};\\
	\hspace*{0.2in}{4: Update $\S_i^{k+1}$ simultaneously via \eqref{wnn-ups} for $i=[N]$};\\
	\hspace*{0.2in}{5: Update $\Q^{k+1}$ via \eqref{wnn-upq}};\\
	\hspace*{0.2in}{6: Update $\beta_{k+1}= \varrho \beta_{k}$ with $\varrho>1$;}\\
	\hspace*{0.2in}{7: Until a termination criterion is fulfilled}.\\
	\hspace*{0.02in} {\bf Output:} 
	$\X^*$.
\end{algorithm}

\subsection{{DCT-based sparse + $p$-shrinking mapping induced low-rank optimization model}}
In model \eqref{WTNN-question-1}, we employ the sum of the weighted nuclear norms of the unfolding matrices $X_{(i)}$ to approximate the low-rankness of the data. However, we do not know whether such an approximation is enough ideal for color images and videos tensor data. Recently, the so-called $p$-shrinking algorithm is widely used in low-rank and sparse recovery problems. Hence, in this subsection, we further consider the combination of $p$-shrinking low-rank penalty and DCT-driven sparse regularization, and propose the following tensor completion model:
\begin{equation}\label{question-1}
\begin{array}{cl}\min\limits_{\mathcal{X}} &\sum_{i=1}^{N} \alpha_i \Phi_p^u(X_{(i)})+\lambda\|\mathcal{{\mathscr D}(X)}\|_1\\
\text{s.t.}& \P_{\Omega}(\mathcal{X}) = \P_{\Omega}(\H),
\end{array}
\end{equation}
where $p\leq 1$ and $\Phi_p^u(\cdot)$ is a nonconvex penalty function defined in \eqref{Phi-prox}.  Compared to model \eqref{WTNN-question-1}, the only difference is that we here employ the $\Phi_p^u(X_{(i)})$ to replace the weighted nuclear norm term $\|X_{(i)}\|_{*,\bm{w}}$. So, both \eqref{WTNN-question-1} and \eqref{question-1} share the same structure except the low-rank-inducing term. {For simplicity, we use the symbols in \eqref{ADMM-WTNN} again}, then model \eqref{question-1} can be rewritten as the following separable optimization problem, i.e.,
\begin{equation}\label{question1}
\begin{array}{cl}
 \min\limits_{\mathcal{X}} & \sum_{i=1}^{N} \alpha_i \Phi_p^u(Y_{i,(i)})+\lambda\|\mathcal{T}\|_1\\
\text{s.t.}& \mathcal{X} = \mathcal{Y}_i,\quad \text{for all }~i \in [N],\\
& \mathcal{T} = {\mathscr D}(\mathcal{X}),\\
&\P_{\Omega}(\mathcal{X}) = \P_{\Omega}(\H).
\end{array}
\end{equation}
Correspondingly, the augmented Lagrangian function for \eqref{question1} reads as
\begin{align}\label{2Augment_lag}
&{\bm L}(\Y_i,\X,\T,\S_i,\Q) \nonumber \\
&:=\sum_{i=1}^{N}\left( \alpha_i\Phi_p^\mu(Y_{i,(i)})+\left\langle \S_i,\X-\Y_i \right\rangle +\frac{\beta}{2}\|\X-\Y_i\|_F^2 \right) \nonumber\\
&\quad +\lambda\|\T\|_1+\left\langle \Q,\T-{\mathscr D}(\X) \right\rangle+\frac{\beta}{2}\|\T-{\mathscr D}(\X)\|_F^2.
\end{align}
Based upon \eqref{2Augment_lag}, by invoking the update idea of Algorithm \ref{ADMM-solve2}, i.e., $\mathcal{Y}_i\to\mathcal{X}\to\mathcal{T}\to\mathcal{S}_i\to\mathcal{Q}$, we derive the detailed iterative schemes for solving model \eqref{question1}.
\begin{itemize}
\item Update $\mathcal{Y}_i^{k+1}$  via 
\begin{align}\label{M-kk}
\Y^{k+1}_i & = \arg\min_{\mathcal{Y}_i} {\bm L}\left({\Y_i,\X^k,\T^k,\S^k_i,\Q^k}\right) \nonumber \\
&=\arg\min_{\mathcal{Y}_i} \left\{\alpha_i\Phi_p^\mu(Y_{i,(i)}) +\frac{\beta_k}{2}\left\|\Y_{i}-\widehat{\Y}_{i}^k\right\|_F^2\right\}\nonumber \\
&={\rm{fold}}\left(\tilde{U}_k{\rm \bf pshrink}\left(\tilde{\Lambda}^k,\; \frac{\alpha_i}{\beta_k},\;p\right)\tilde{V}_k^\top  \right),
\end{align}
where `${\rm \bf pshrink}(\cdot,\cdot,\cdot)$' is given by \eqref{pshrink} and
$$\widehat{\Y}_{i}^k=\X^k+\frac{1}{\beta_k}\S_{i}^k\quad\text{and}\quad\widehat{Y}_{i,(i)}^k=\tilde{U}_k\tilde{\Lambda}^k\tilde{V}_k^\top.$$
\item {We obtain $\mathcal{X}^{k+1}$ via solving 
	\begin{equation*}
	\min_{\mathcal{X}} \left\{{\bm L}\left({\Y^{k+1}_i,\X,\T^k,\S^k_i,\Q^k}\right) \;|\; \P_{\Omega}(\mathcal{X}) = \P_{\Omega}(\mathcal{H})\right\},
	\end{equation*}	
	which can also be expressed explicitly by
	\begin{equation}\label{IpST-X-Update}
	\left\{
	\begin{array}{l}
	\mathcal{X}^{k+1}_{\Omega} = \mathcal{H}_{\Omega},\\
	\X^{k+1}_{\Omega^c} =\frac{1}{N+1} \widehat{\mathcal{X}}^{k}_{\Omega^c}
	\end{array}\right.
	\end{equation}
}
where 
$$\widehat{\mathcal{X}}^{k}= \sum_{i=1}^{N}\left(\mathcal{Y}_i^{k+1}-\frac{1}{\beta_k}\mathcal{S}_i^k\right)+{\mathscr D}^{-1}\left(\mathcal{T}^k+\frac{1}{\beta_k}\mathcal{Q}^k\right).$$
\item  Update $\mathcal{T}^{k+1}$ via 
\begin{align}\label{pshrink-T}
\T^{k+1}  &= \arg\min_{\mathcal{T}} {\bm L}\left({\Y^{k+1}_i,\X^{k+1},\T,\S^k_i,\Q^k}\right) \nonumber \\
&= \text{shrink}\left({\mathscr D}(\X^{k+1})-\frac{1}{\beta_k}\Q^k,\frac{\lambda}{\beta_k}\right).
\end{align}
\item Finally, update the two Lagrangian multipliers $\mathcal{S}_i^{k+1}$ and $\mathcal{Q}^{k+1}$ via \eqref{wnn-ups} and \eqref{wnn-upq}, respectively.
\end{itemize}

Formally, we summarize the iterative schemes for \eqref{question1} in Algorithm \ref{ADMM-solve}.
\begin{algorithm}[!htbp]
	\caption{ADMM for Model \eqref{question1}.}\label{ADMM-solve}
	\hspace*{0.02in} {\bf Input:} 
	Initial starting points $\X^0,\Y^0,\T^0,\S_0,\Q_0$.\\
	\hspace*{0.2in}{1: Update $\Y_i^{k+1}$ simultaneously via \eqref{M-kk} for $i=[N]$};\\
	\hspace*{0.2in}{2: Update $\X^{k+1}$ via \eqref{IpST-X-Update}};\\
	\hspace*{0.2in}{3: Update $\T^{k+1}$ via \eqref{pshrink-T}};\\
	\hspace*{0.2in}{4: Update $\S_i^{k+1}$ simultaneously via \eqref{wnn-ups} for $i=[N]$};\\
	\hspace*{0.2in}{5: Update $\Q^{k+1}$  via \eqref{wnn-upq}};\\
	\hspace*{0.2in}{6: Update $\beta_{k+1}= \varrho \beta_{k}, \varrho>1$;}\\
	\hspace*{0.2in}{7: Until a termination criterion is fulfilled}.\\
	\hspace*{0.02in} {\bf Output:} 
	$\X^*$.
\end{algorithm}

\subsection{Convergence results}
In this subsection, we follow the way used in \cite{OTBKK16} to prove the convergence result of Algorithm \ref{ADMM-solve2} for problem (\ref{ADMM-WTNN}) under some assumptions. Since Algorithm \ref{ADMM-solve} share the almost same iterative schemes expect the $\Y$-subproblem, we skip the convergence analysis of Algorithm \ref{ADMM-solve} for model (\ref{question1}) for the conciseness of the paper.

%

\begin{lemma}\label{LUpperb}
	Suppose that both $\{S^k\}:=\{\S_1^k,\dots,\S_N^k\}$ and $\{\Q^k\}$ are bounded, the	sequences $\{\Y^k\}:=\{\Y_1^k,\cdots,\Y^k_{N}\}$ and $\{\T^k\}$ produced by Algorithm \ref{ADMM-solve2} are bounded.
\end{lemma}

\begin{proof}
	From the definition of  ${\bm L}(\Y,\X,\T,\S,\Q)$, we have 
	\begin{align}\label{lineg}
	&{\bm L}(\Y^k,\X^k,\T^k,\S^k,\Q^k) - {\bm L}(\Y^k,\X^k,\T^k,\S^{k-1},\Q^{k-1})\nonumber \\
	&=\sum_{i=1}^{N}\left\langle \S_i^k-\S_i^{k-1},\X^k-\Y_i^k \right\rangle + \langle \Q^k-\Q^{k-1},\T^k-\mathscr{D}(\X^k) \rangle\nonumber\\
	&~~~+\frac{\beta_k-\beta_{k-1}}{2}\left( \sum_{i=1}^{N}\|\X^k-\Y_i^k\|_F^2+\|\T^k-\mathscr{D}(\X^k)\|_F^2 \right)\nonumber\\
	&=\frac{1}{\beta_{k-1}}\left( \sum_{i=1}^{N}\|\S_i^k-\S_i^{k-1}\|_F^2+\|\Q^k-\Q^{k-1}\|_F^2\right)\nonumber\\
	&~~~+\frac{\beta_k-\beta_{k-1}}{2\beta_{k-1}^2}\left( \sum_{i=1}^{N}\|\S_i^k-\S_i^{k-1}\|_F+\|\Q^k-\Q^{k-1}\|_F^2 \right)\nonumber\\
	&=\frac{\beta_k+\beta_{k-1}}{2\beta_{k-1}^2}\left( \sum_{i=1}^{N}\|\S_i^k-\S_i^{k-1}\|_F^2 +\|\Q^k-\Q^{k-1}\|_F^2 \right),
	\end{align}
	where the second equality follows from (\ref{wnn-ups}) and (\ref{wnn-upq}). Then, it follows from the iterative schemes of Algorithm \ref{ADMM-solve2} and (\ref{lineg}) that
	\begin{align}\label{LUpper}
	&{\bm L}(\Y^{k+1},\X^{k+1},\T^{k+1},\S^k,\Q^k)\nonumber\\
	&\leq{\bm L}(\Y^k,\X^k,\T^k,\S^k,\Q^k)\nonumber\\
	&={\bm L}(\Y^k,\X^k,\T^k,\S^{k-1},\Q^{k-1})\nonumber\\
	&~~~+\frac{\beta_k+\beta_{k-1}}{2\beta_{k-1}^2}\left(  \sum_{i=1}^{N}\|\S_i^k-\S_i^{k-1}\|_F^2+\|\Q^k-\Q^{k-1}\|_F^2 \right).
	\end{align}
By invoking $\beta_k= \varrho \beta_{k-1}$ with $\varrho>1$, we immediately prove $\sum_{i=1}^{\infty}\frac{\beta_k+\beta_{k-1}}{2\beta_{k-1}^2}<\infty$. As a consequence of (\ref{LUpper}), the sequence $\{ {\bm L}(\Y^k,\X^k,\T^k,\S^k,\Q^k) \}$ is upper-bounded due to the boundedness of $\{\S^k\}$ and $\{Q^k\}$. Rearranging terms of $\{ {\bm L}(\Y^k,\X^k,\T^k,\S^k,\Q^k) \}$ immediately yields
	\begin{align*}
	&\sum_{i=1}^{N}\left(\alpha_i\|Y^k_{i,(i)}\|_{*,\bm w}\right)+\lambda\|\T^k\|_1\nonumber\\
	&={\bm L}(\Y^k,\X^k,\T^k,\S^{k-1},\Q^{k-1})-\frac{\beta_{k-1}^{-1}}{2}\left( \|\Q^k\|_F^2-\|\Q^{k-1}\|_F^2 \right)\nonumber\\
	&~~~-\frac{\beta_{k-1}^{-1}}{2}\sum_{i=1}^{N}\left( \|\S_i^k\|_F^2-\|\S_i^{k-1}\|_F^2 \right),
	\end{align*}
	which, together with the boundedness of the sequences $\{{\bm L}(\Y^k,\X^k,\T^k,\S^{k-1},\Q^{k-1})\}$, $\{\S^k\}$ and $\{\Q^k\}$, implies that  $\{Y^k_{i,(i)}\}$ and $\{\T^k\}$ for all $i\in [N]$ are bounded, since $\|\cdot\|_{*,\bm w}$ and $\|\cdot\|_1$ are nonnegative. 
\end{proof}

\begin{theorem}\label{kkt-theorem}
	Let $\left\{ \left( \Y^k,\X^k,\T^k,\S^k,\Q^k \right)   \right\}$ be a sequence generated by Algorithm \ref{ADMM-solve2}.
	Suppose that both $\{\S^k\}$ and $\{\Q^k\}$ are bounded, and further satisfy $\lim\limits_{k\rightarrow\infty} \|\S^{k+1}-\S^k\|_F=0$ and  $\lim\limits_{k\rightarrow\infty} \|\Q^{k+1}-\Q^k\|_F = 0$, respectively. Then,
	\begin{itemize}
		\item[(i).] $\{\X^k\},\{\Y^k\},\{\T^k\}$ are bounded Cauchy sequences.
		\item[(ii).] any accumulation point $\left\{ \left( \Y^\infty,\X^\infty,\T^\infty,\S^\infty,\Q^\infty \right) \right\}$ satisfies the KKT condition of model \eqref{ADMM-WTNN}.
	\end{itemize}
\end{theorem}

\begin{proof}\label{kkt-proof}
	For any $i\in [N]$, it easily follows from the iterative scheme \eqref{wnn-ups} and $\lim\limits_{k\rightarrow\infty}\beta_k = \infty$ that the sequence $\{\X^k\}$ is bounded due to Lemma \ref{LUpperb}. Moreover, by the properties of $\{\S^k\}$ and $\{\Q^k\}$, we have
		\begin{equation}\label{cq}
			\left\{\begin{array}{l}
			\lim_{k\rightarrow\infty}\|\X^{k+1}-\Y_i^{k+1}\|_F=0, \\
		\lim_{k\rightarrow\infty}\|\T^{k+1}-\mathscr{D}(\X^{k+1})\|_F=0,
			\end{array}\right.
		\end{equation} 
	  which further implies that $\{\Y^k\}$ and $\{\T^k\}$ approach to feasible solutions.


	Now, we show that $\{\X^k\}$, $\{\Y^k\}$, $\{\T^k\}$ are Cauchy sequences. Here, we only show $\{\Y^k\}$ being a Cauchy sequence, while the other sequences can be proved in a similar way.
	
	Invoking the first-order optimality condition of the $\Y$-subproblm yields
	\begin{align*}
		0&\in\alpha_i\partial_C\|Y_{i,(i)}^{k+1}\|_{*,\bm w}-\beta_k\left(X_{(i)}^k-Y_{i,(i)}^{k+1}+\frac{1}{\beta_{k}}S_{i,(i)}^k \right),
	\end{align*}
	which, together with (\ref{wnn-ups}), leads to
	\begin{align}\label{fir-op-Y}
		X_{(i)}^k-X_{(i)}^{k+1}\in\frac{\alpha_i\partial_C\|Y_{i,(i)}^{k+1}\|_{*,\bm w}-S_{i,(i)}^{k+1}}{\beta_{k}}.
	\end{align}
	By the boundeness of both $\{Y_{i,(i)}^k\}$ and $\{\S_i^k\}$ and properties of $\|\cdot\|_{*,\bm w}$, it follows from \eqref{fir-op-Y} that $\|X_{(i)}^k-X_{(i)}^{k+1}\|_F=O(\beta_k^{-1})$, which together with $\sum_{i=0}^{\infty}\beta_k^{-1}=\frac{\varrho}{\beta_0(\varrho-1)}<\infty$, implies that $\{\X^k\}$ is a Cauchy sequence, and it immediately has a limit point.

Let $\X^{\infty}$, $\Y^{\infty}$,  $\T^{\infty}$,  $\S^{\infty}$,  $\Q^{\infty}$ be the limit points of $\{\X^k\}$, $\{\Y^k\}$, $\{\T^k\}$, $\{\S^k\}$, $\{\Q^k\}$, respectively. It first follows from \eqref{cq} that 
\begin{equation*}
\X^{\infty}=\Y_i^{\infty}, (i=1,\ldots,N)\;\;\text{and} \;\;\T^{\infty}=\mathscr{D}(\X^{\infty}).
\end{equation*}
Rearranging \eqref{fir-op-Y} arrives at 
	\begin{align}\label{kkt-Y}
		0&\in \alpha_i\partial_C\|Y_{i,(i)}^{k+1}\|_{*,\bm w}-S_{i,(i)}^{k+1}-\beta_k\left(X_{(i)}^k-X_{(i)}^{k+1}  \right)\nonumber \\
		&= \alpha_i\partial_C\|\Y_{i}^{k+1}\|_{*,\bm w}-\S_{i}^{k+1}-\beta_k\left(\X^k-\X^{k+1}  \right).
	\end{align}	
Consequently, taking limit on \eqref{kkt-Y} immediately yields $\S_{i}^\infty \in \alpha_i\partial_C\|\Y_{i}^{\infty}\|_{*,\bm w}$.
	

By the first-order optimality condition of the $\X$-subproblem, we have $\bar{X}$
\begin{equation*}
\left\{ \begin{array}{l}
\sum\limits_{i=1}^{N} \left( \S_i^k + \beta_k(\X^{k+1}-\Y_i^{k+1}) \right) \\
\qquad + \beta_k\left(\X^{k+1}-\mathscr{D}^{-1}(\T^k)\right)-\mathscr{D}^{-1}(\Q^k)=0,\\  
\P_{\Omega}(\mathcal{X}^{k+1}) = \P_{\Omega}(\H),
\end{array}\right.
\end{equation*}
Taking $k\to \infty$ on the above two equalities leads to 
\begin{equation*}
\sum\limits_{i=1}^{N} \S_i^\infty =\mathscr{D}^{-1}(\Q^\infty)\; \text{ and }\;
\P_{\Omega}(\mathcal{X}^{\infty}) = \P_{\Omega}(\H).
\end{equation*}
	
Finally, the first-order optimality condition of the $\T$-subproblem reads as
	\begin{align*}
		0&\in\partial \lambda\|\T^{k+1}\|_1+\beta_k\left( \T^{k+1}-\mathscr{D}(X^{k+1})+\beta_k^{-1}\Q^k  \right)\nonumber\\
		 & = \partial\lambda\|\T^{k+1}\|_1+\Q^k+\beta_k\left( \T^{k+1}-\mathscr{D}(\X^{k+1}) \right)\nonumber\\
		 &=\partial \lambda\|\T^{k+1}\|_1+\Q^{k+1}.
	\end{align*}
Consequently, we have $-\Q^{*}\in\partial\lambda\|\T^{*}\|_1$ when $k\rightarrow\infty$. Therefore, we conclude that the accumulation point satisfies the KKT condition of model \eqref{ADMM-WTNN}.
\end{proof}

Hereafter, we only give the time complexity analysis of Algorithm \ref{ADMM-solve2} and omit the analysis of Algorithm \ref{ADMM-solve} since both of them share the almost iterative scheme and the same SVD.

The main time cost of Algorithm \ref{ADMM-solve2} is consuming by performing SVD and multi-dimensional DCT. In each iteration, the complexity of subproblem $\Y^{k+1}_{i}~(i\in [N])$ is $O\left( \sum_{i=1}^{N}(I_i)^2\times\Pi_{j\neq i}I_j \right)$. The complexities of subproblem $\X^{k+1}$ and $\T^{k+1}$ are both $O\left( \Pi_{i=1}^NI_i\times\log\Pi_{i=1}^NI_i \right)$. In addition, the complexities to update $\S_i^{k+1}~(i\in [N])$ and $\Q^{k+1}$ are $O(N\Pi_{i=1}^NI_i)$ and $O(\Pi_{i=1}^NI_i)$. So, the time complexity of each iteration is $O\left( \sum_{i=1}^{N}(I_i)^2\times\Pi_{j\neq i}I_j+\Pi_{i=1}^NI_i\times\log\Pi_{i=1}^NI_i \right)$
Thus, the total time complexity of Algorithm \ref{ADMM-solve2} is $O\left( t\left( \sum_{i=1}^{N}(I_i)^2\times\Pi_{j\neq i}I_j+\Pi_{i=1}^NI_i\times\log\Pi_{i=1}^NI_i \right)\right)$, where $t$ is the number of iterations.

\section{Numerical Experiments}\label{ExpTests}
In this section, we are concerned with the numerical performance of our approach proposed in Section \ref{ModAlg} on images and videos recovery. Here, we will consider two kinds of images and videos data sets: 1) RGB images; 2) surveillance videos. All algorithms were implemented in MATLAB R2018b (64bit) and experiments were conducted on a laptop computer with Intel(R) Core(TM) i7-7500 CPU @2.70GHz and 8GB memory. Throughout this section, we denote Algorithms \ref{ADMM-solve2} and \ref{ADMM-solve} by `DCT-WNN' and `DCT-IpST' for simplicity, respectively. Moreover, we also compare the proposed algorithms with eight state-of-the-art tensor completion approaches as follows:

\begin{itemize}
\item IpST\cite{SLH19}: Iterative $p$-shrinkage thresholding algorithm for solving low Tucker rank tensor recovery problem, which only employs a nonconvex penalty function $\Phi_p^\mu(\cdot)$ given in \eqref{Phi-prox} to replace the traditional nuclear norm of unfolding matrices. 
\item TTNNL1\cite{HLHS17}: Using the truncated tensor nuclear norm for low-rank approximation, and a sparse regularization term combined with the 3-D DCT bases. 
\item F-TNN\cite{JNZH20}\footnote{https://github.com/TaiXiangJiang/Framelet-TNN}: Framelet representation of tensor nuclear norm for third-order tensor completion.
\item TNN-3DTV\cite{JLLS18}: Anisotropic total variation regularized low-rank tensor completion based on tensor nuclear norm, which utilizes the 3D total variation regularization to exploit the structure of images.
\item TCTF\cite{ZLLZ18}\footnote{https://panzhous.github.io/}: Tensor factorization for low-rank tensor completion, which uses the operation of T-product to decompose a tensor into two smaller tensors.
\item  WSTNN\cite{ZHZJM20}\footnote{https://yubangzheng.github.io/ybz/}: Tensor N-tubal rank and its convex relaxation for low-rank tensor recovery, which defines the weighted sum of the tensor nuclear norm.
\item SMF-LRTC\cite{ZHJZJM19}\footnote{https://github.com/zhaoxile/Low-rank-tensor-completion-via-smooth-matrix-factorization}: Low-rank tensor completion via smooth matrix factorization, which exploits the piecewise smoothness prior of tensors by introducing smoothness constraints on the factor matrices. Here, we only use it for video experiments, because such a model is designed for  video and hyperspectral image data sets.
\item DP3LRTC\cite{ZXJWN20}: Using tensor nuclear norm to characterize the global low-rankness prior and plugging a denoising neural network, which learned from a large number of natural images. It is a deep learning method.
\end{itemize}

For the fair comparison, we take
\begin{equation}
	\text{RelCha} = \frac{\|\X^{k+1}-\X^k\|_F}{\|\X_{\rm{true}}\|_F}\leq 10^{-4},
\end{equation}
as the stopping criterion for all methods, where $\X_{\rm{true}}$ is the true tensor.
Moreover, the following metrics are chosen to evaluate the recovery performance of the different algorithms.

 \begin{itemize}
\item  Peak Signal-to-Noise Ratio (PSNR \cite{ZLLZ18}) is defined as: $$\text{PSNR}=10\log_{10}\frac{(\#\mathcal{X}_{\rm true})\cdot(\X_{\rm true}^{\max} )^2}{\|\X^{*}-\X_{\rm true}\|_F^2}$$
to measure the quality of the restored tensor data by an algorithm, where $(\#\mathcal{X}_{\rm true})$ denotes the number of elements of $\mathcal{X}_{\rm true}$, $\X_{\rm true}^{\max}$ represents the largest element of $\X_{\rm true}$, and $\X^*$ corresponds to the restored tensor. 
\item  Structural Similarity (SSIM \cite{WBSS04})\footnote{{\sc Matlab} package: https://ece.uwaterloo.ca/$\sim$z70wang/research/ssim/.} is defined as:
\begin{equation*}
{\rm SSIM} = \displaystyle \frac{ (2\mu_X\mu_{X^*}+a_1)(2\sigma_{XX^*}+a_2)}{ (\mu_X^2+\mu^2_{X^*}+a_1)(\sigma_X^2+\sigma^2_{X^*}+a_2)},
\end{equation*}
where $X$ and $X^*$ denote the greyscale images for the original image and its recovered image, respectively; $a_1$ and $a_2$ are constants; $\mu_X$ and $\mu_{X^*}$ denote the average values, while $\sigma_X$ and $\sigma_{X^*}$ denote the standard deviation of $X$ and $X^*$, respectively; and $\sigma_{XX^*}$ denote the covariance matrix between $X$ and $X^*$.	
 \end{itemize}

Since there are some parameters in models and algorithms, for both models \eqref{ADMM-WTNN} and \eqref{question1}, we set $(\alpha_1,\alpha_2,\alpha_3) = (1/3,1/3,10^{-3})$ for image inpainting and $(\alpha_1,\alpha_2,\alpha_3)  = (1/3,1/3,1/3)$ for video inpainting, respectively. Moreover, we take $\lambda = 0.05$ and $\lambda = 10^{-2}$ for \eqref{ADMM-WTNN} and \eqref{question1}, respectively. For the algorithmic parameters, we set $p = 0.2$ for DCT-IpST throughout the experiments. In addition, we set $\varrho=1.2$ and $\beta^0=10^{-5}$, respectively. All parameters of the other compared algorithms were taken as the default values used in the paper.

\subsection{RGB image inpainting}\label{sec_RGB}
In this subsection, we consider the RGB image inpainting problem, where RGB images namely include Red, Green, and Blue channels and the number of channels corresponds to the modes of a third order tensor. Here, we conduct the numerical performance our approach on $16$ widely used images{\footnote{http://r0k.us/graphics/kodak/}}, which are summarized in Fig. \ref{RGB_image}. Here, the eight images in the first two rows of Fig. \ref{RGB_image} are $256\times 256\times 3$, and the remaining eight images are $512\times 768\times 3$ except image \textbf{woman} of size $768\times 512\times 3$. These images have no visual features such as low-rank and sparsity. However,  as shown in Fig. \ref{fig_rgbsparse}, the sparse structure is apparent when applying DCT to these images, which also sufficiently supports that the main idea of this paper is reasonable.

\begin{figure}[!htbp]
	\centering
	\text{~~~airplane~~~~~~~~baboon~~~~~~~~~facade~~~~~~~~~~house~~~~~}	
	\includegraphics[width=0.11\textwidth]{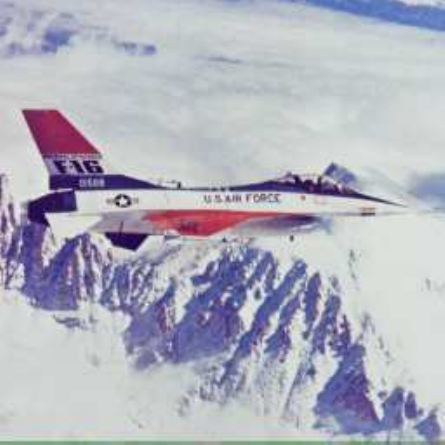}
	\includegraphics[width=0.11\textwidth]{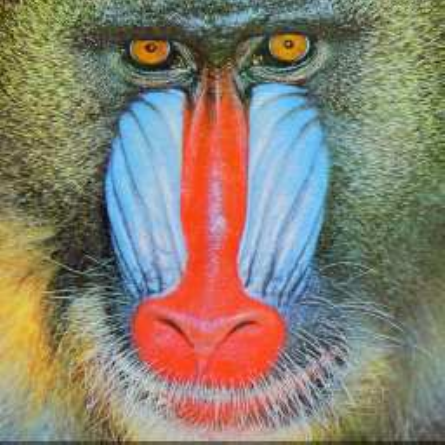}
	\includegraphics[width=0.11\textwidth]{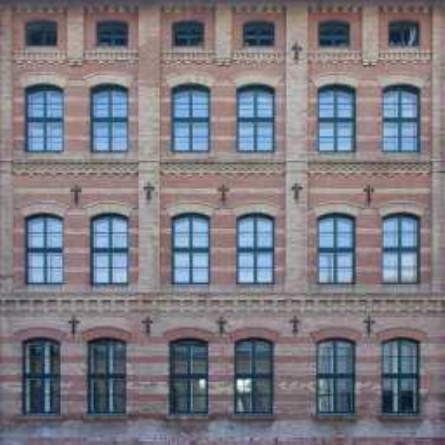}
	\includegraphics[width=0.11\textwidth]{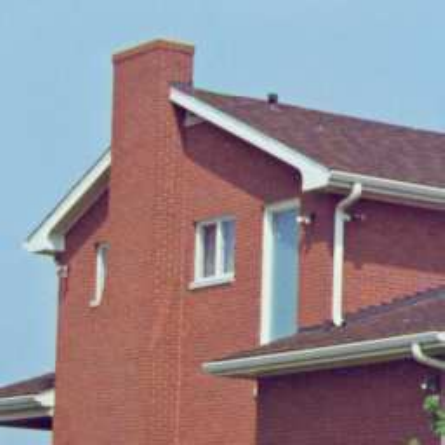}\\
	\text{~~~~peppers~~~~~~~~sailboat~~~~~~~~~~giant~~~~~~~~~~butterfly~~~}
	\includegraphics[width=0.11\textwidth]{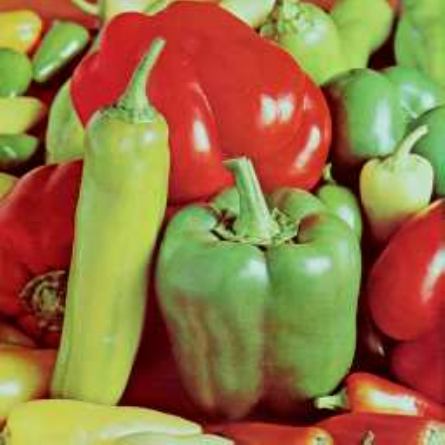}
	\includegraphics[width=0.11\textwidth]{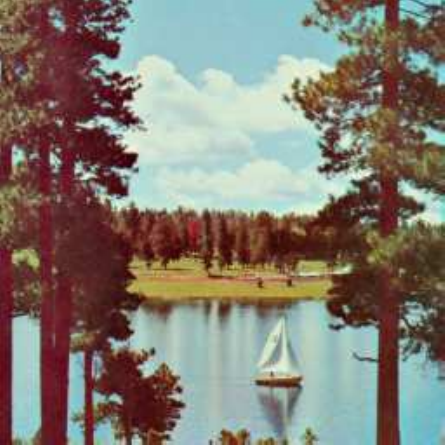}
	\includegraphics[width=0.11\textwidth]{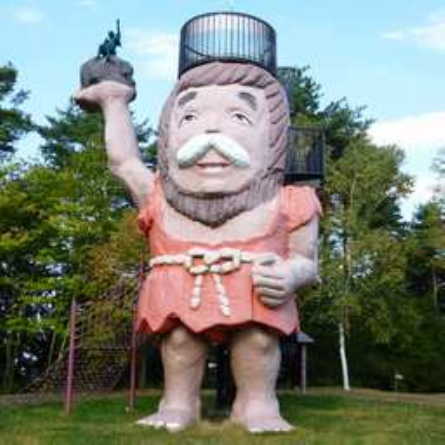}
	\includegraphics[width=0.11\textwidth]{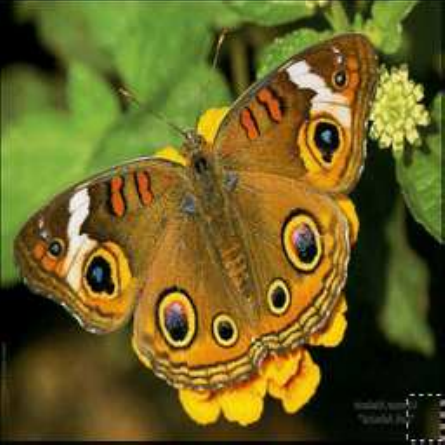}\\
	\text{~~~~~caps~~~~~~~~~~woman~~~~~~motorbikes~~~~~~~buildings~~}
	\includegraphics[width=0.12\textwidth]{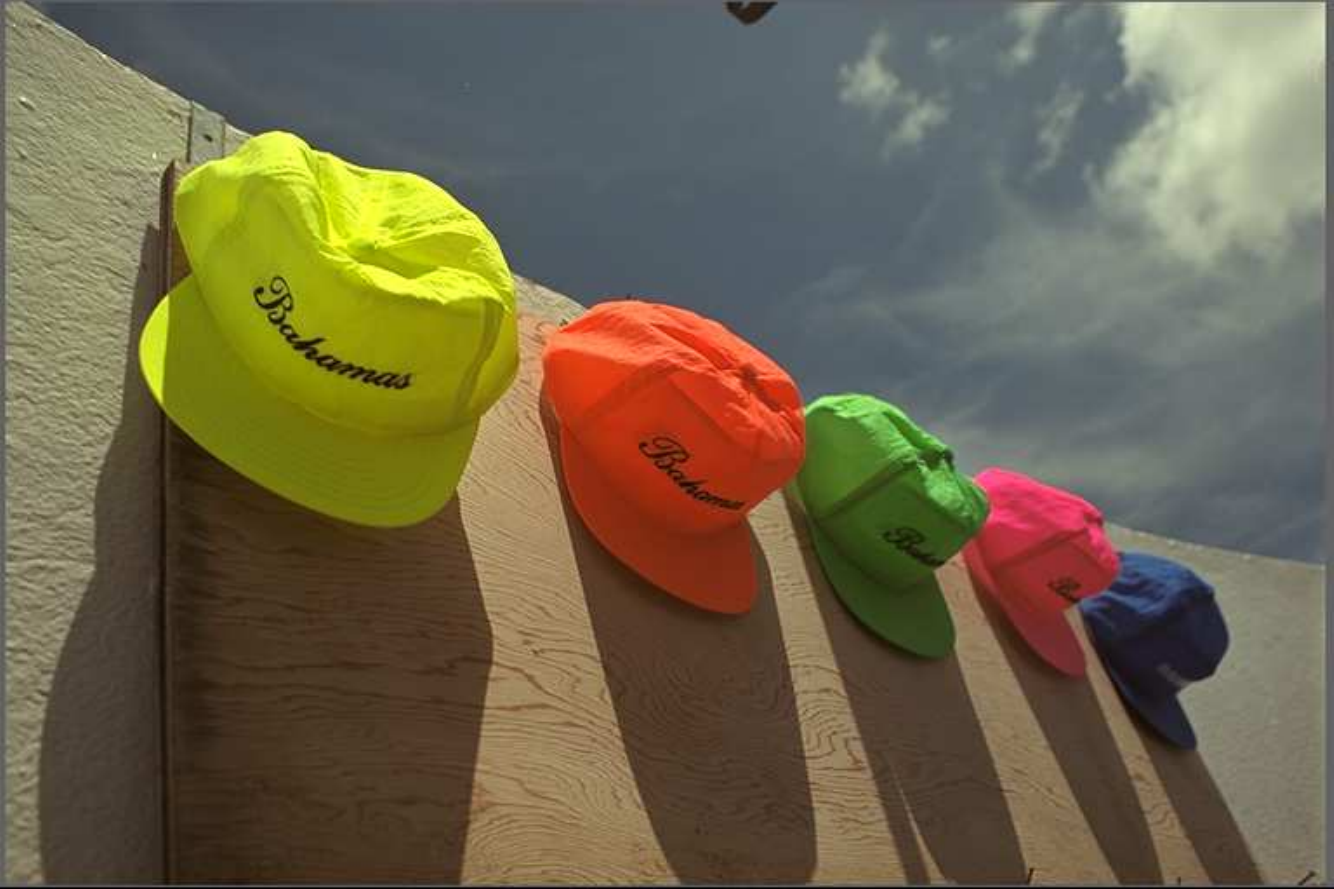}
	\includegraphics[width=0.09\textwidth,height=0.06\textheight]{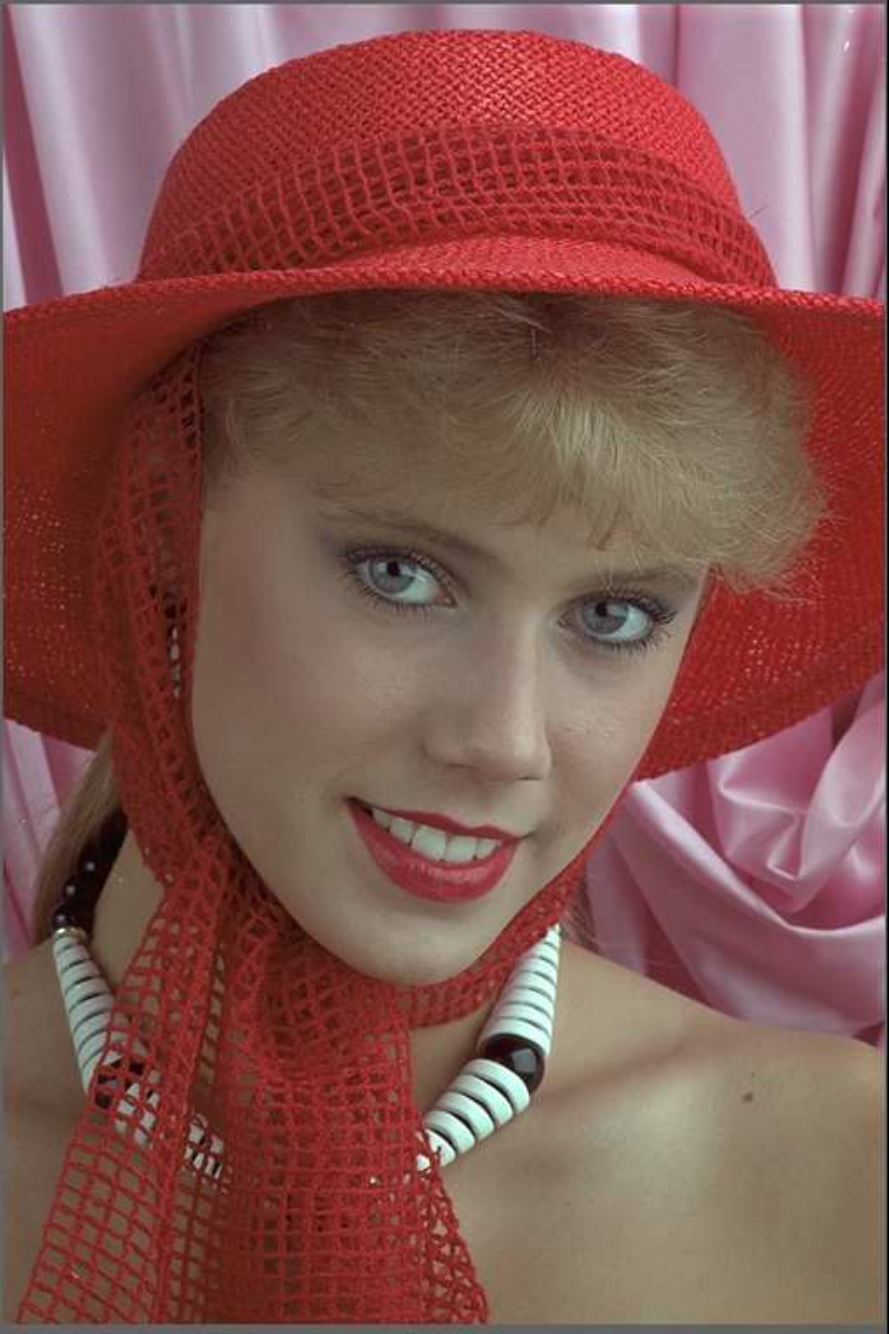}
	\includegraphics[width=0.12\textwidth]{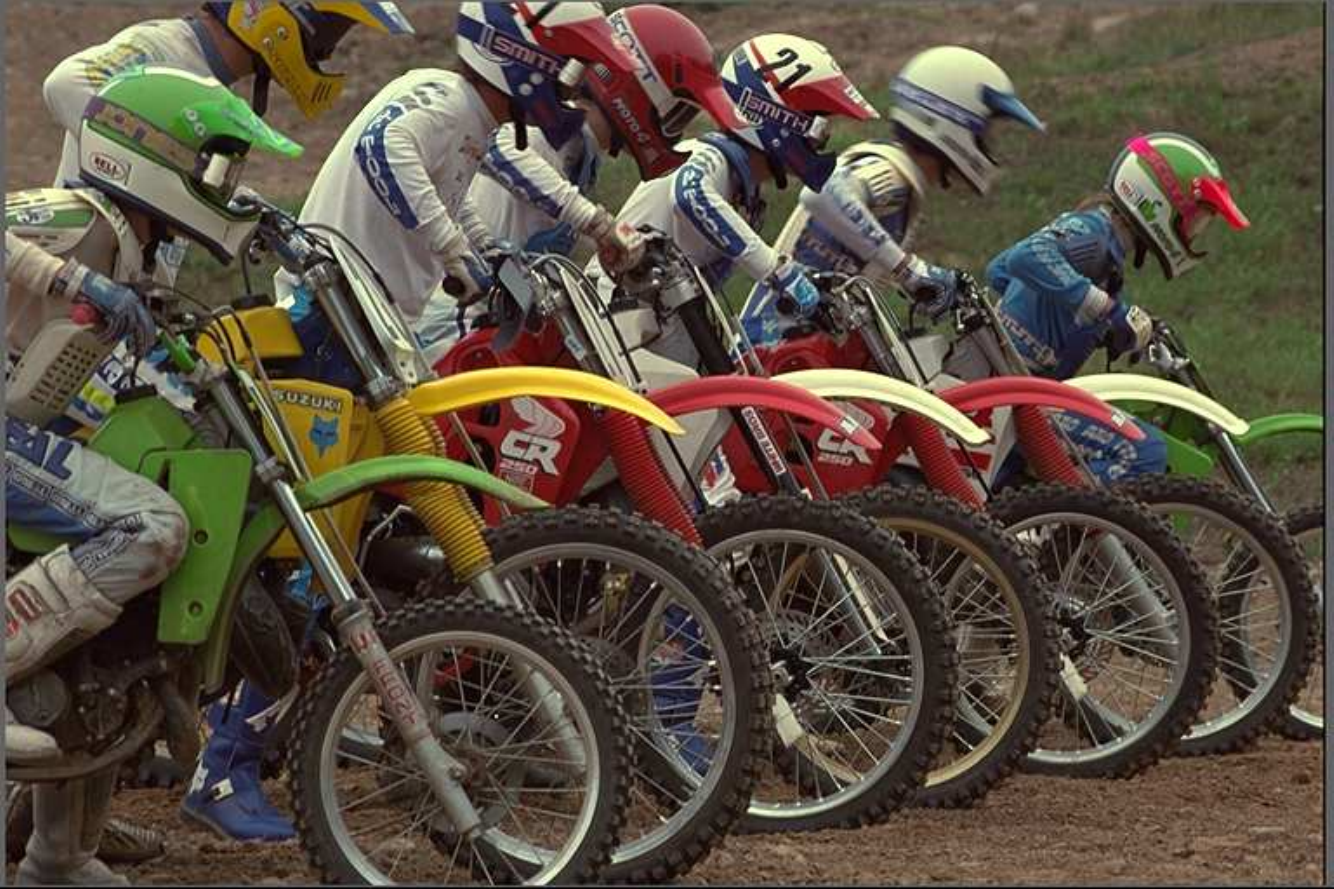}
	\includegraphics[width=0.12\textwidth]{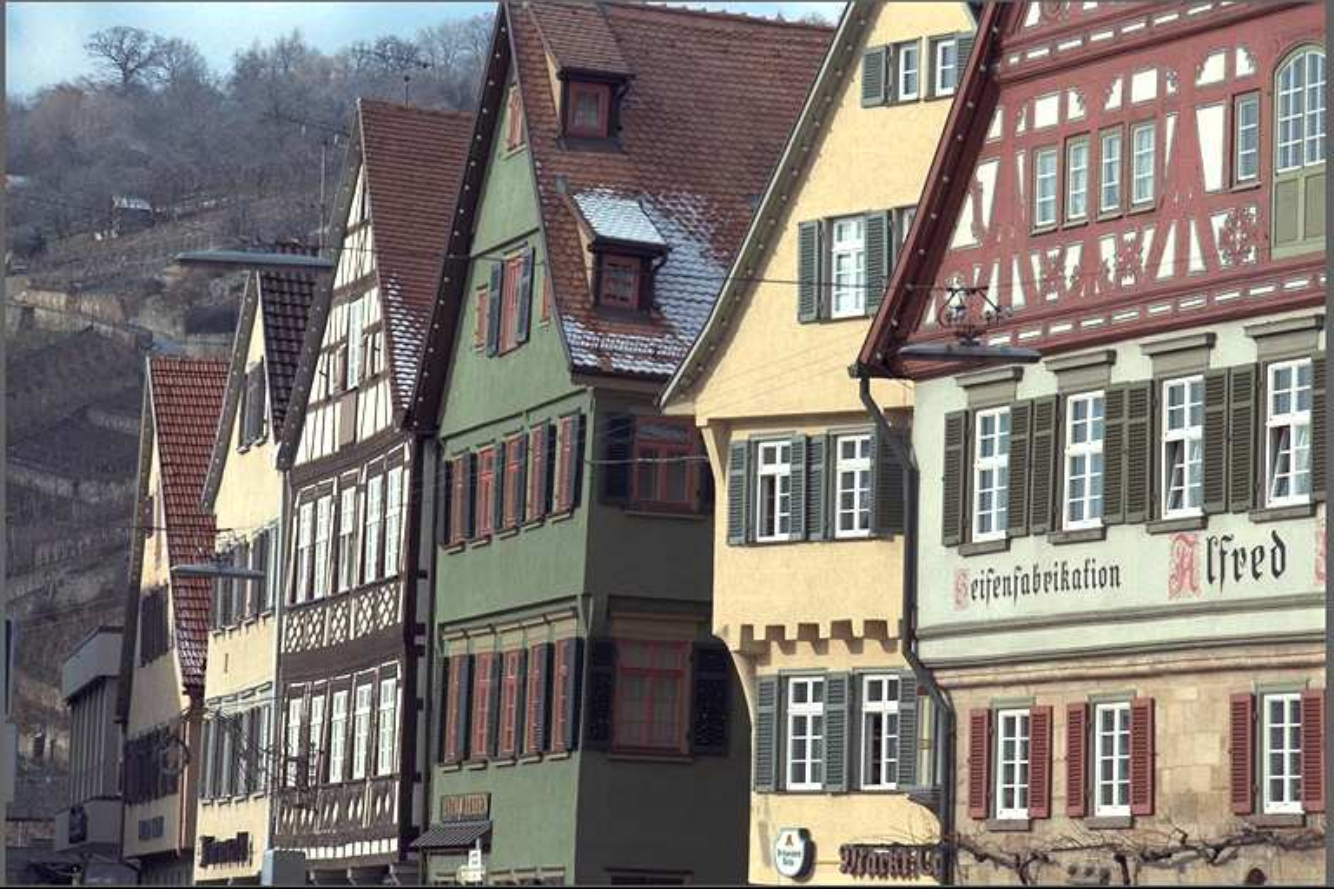}\\
	\text{~~~~valley~~~~~~~~~~~girl~~~~~~~~~~~aircraft ~~~~~~~~window~~}
	\includegraphics[width=0.11\textwidth]{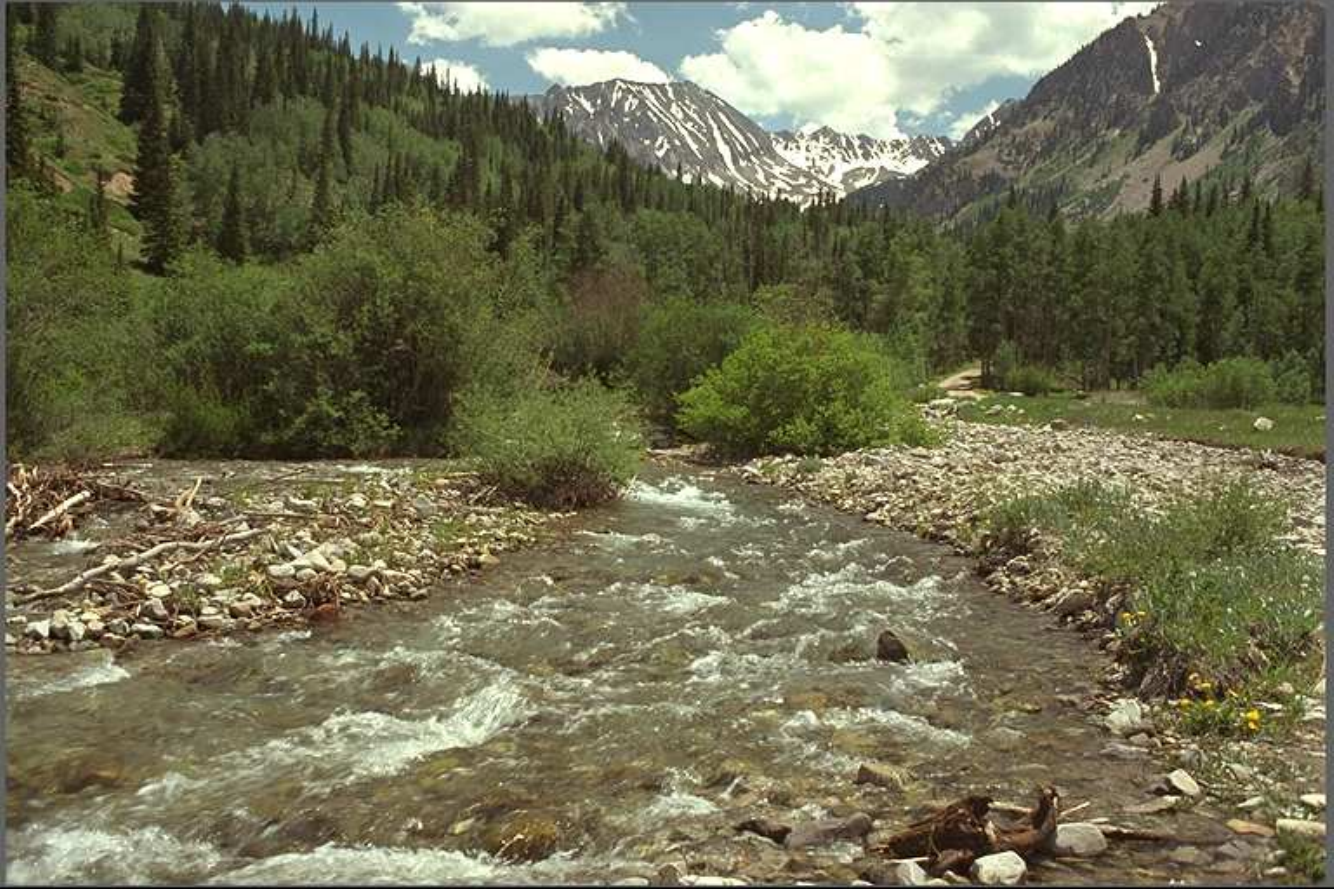}
	\includegraphics[width=0.11\textwidth]{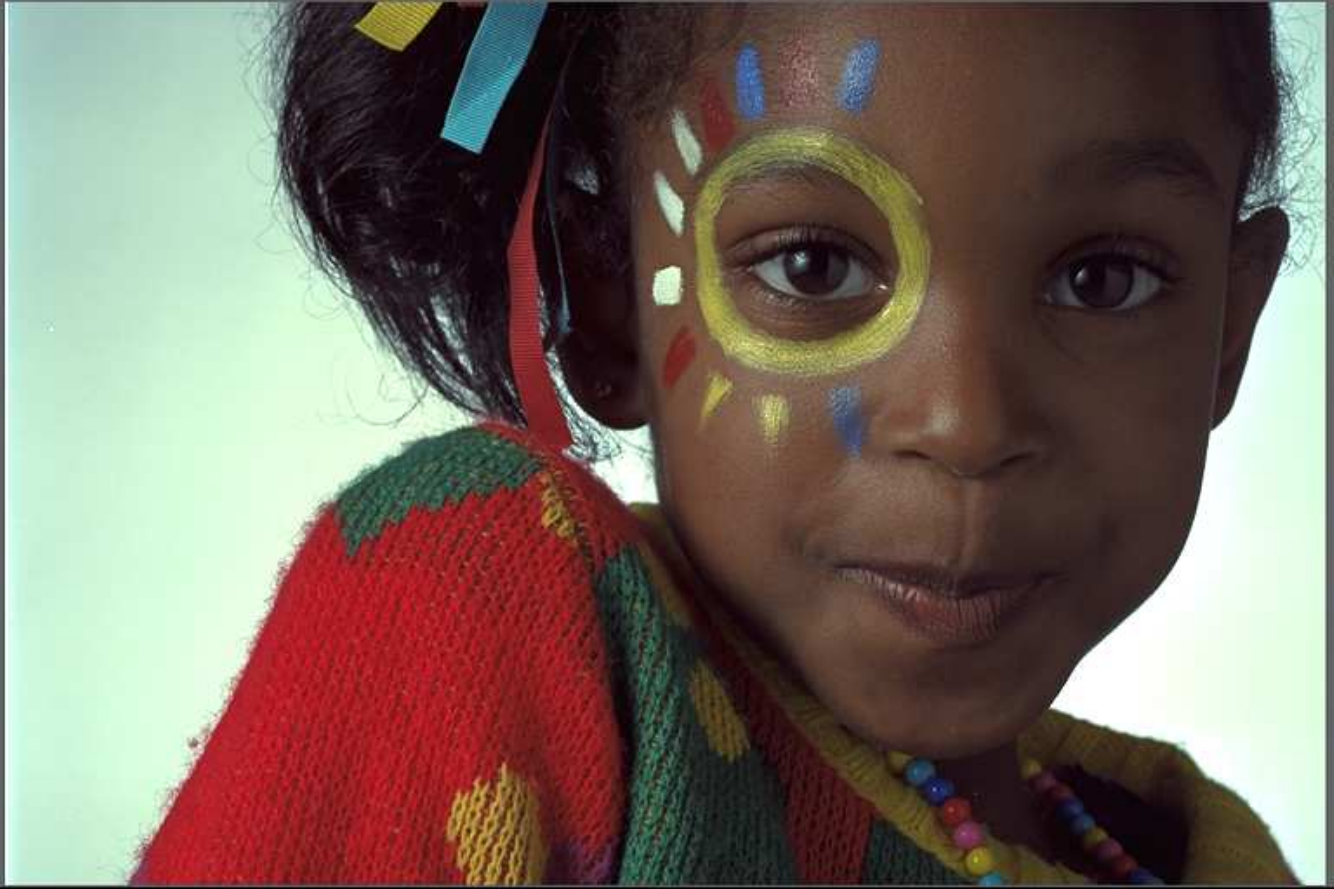}
	\includegraphics[width=0.11\textwidth]{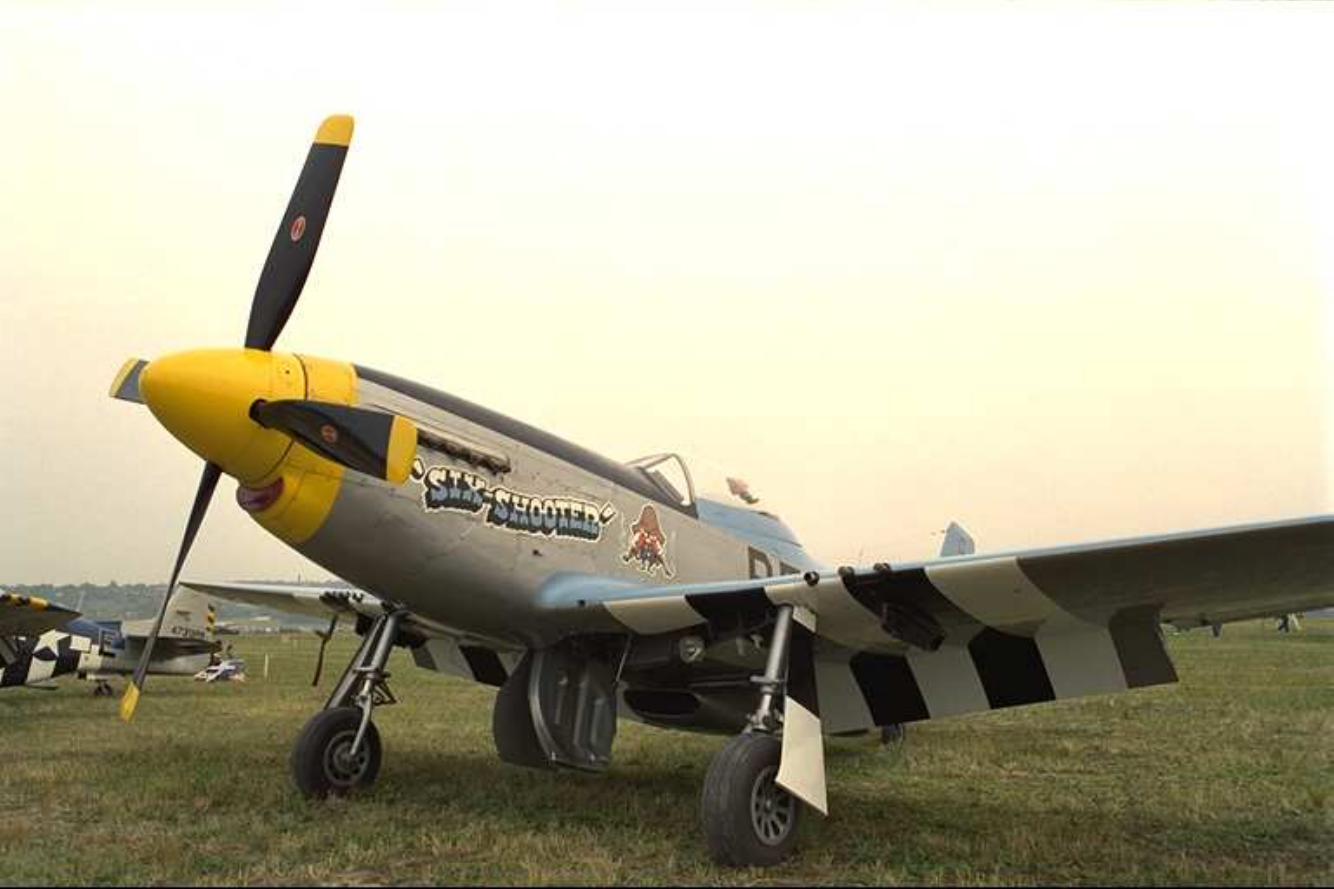}
	\includegraphics[width=0.11\textwidth]{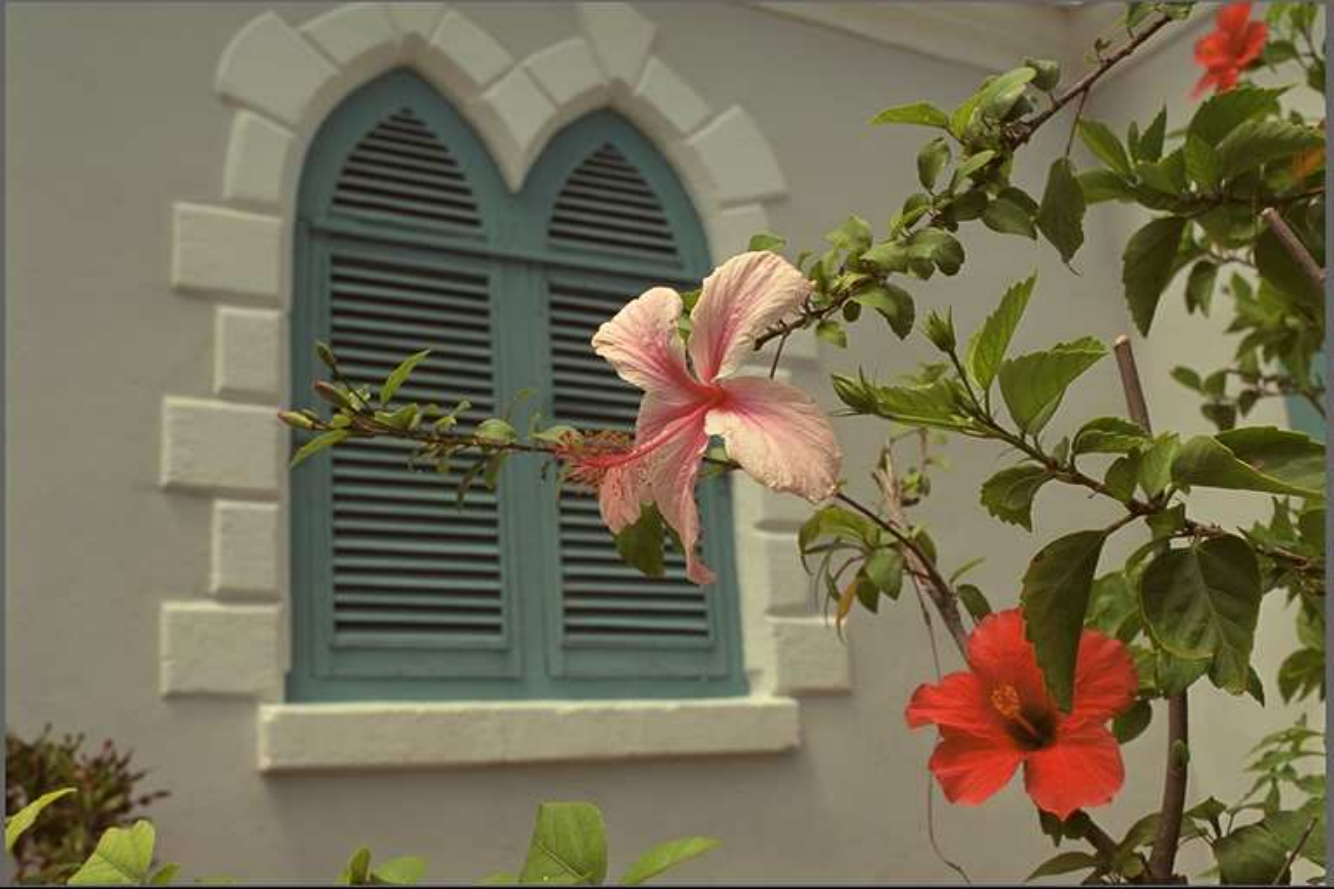}
	\caption{The sixteen RGB images for experiments.}
	\label{RGB_image}
\end{figure}

\begin{figure}[!htbp]
	\centering
	\includegraphics[width=0.5\textwidth]{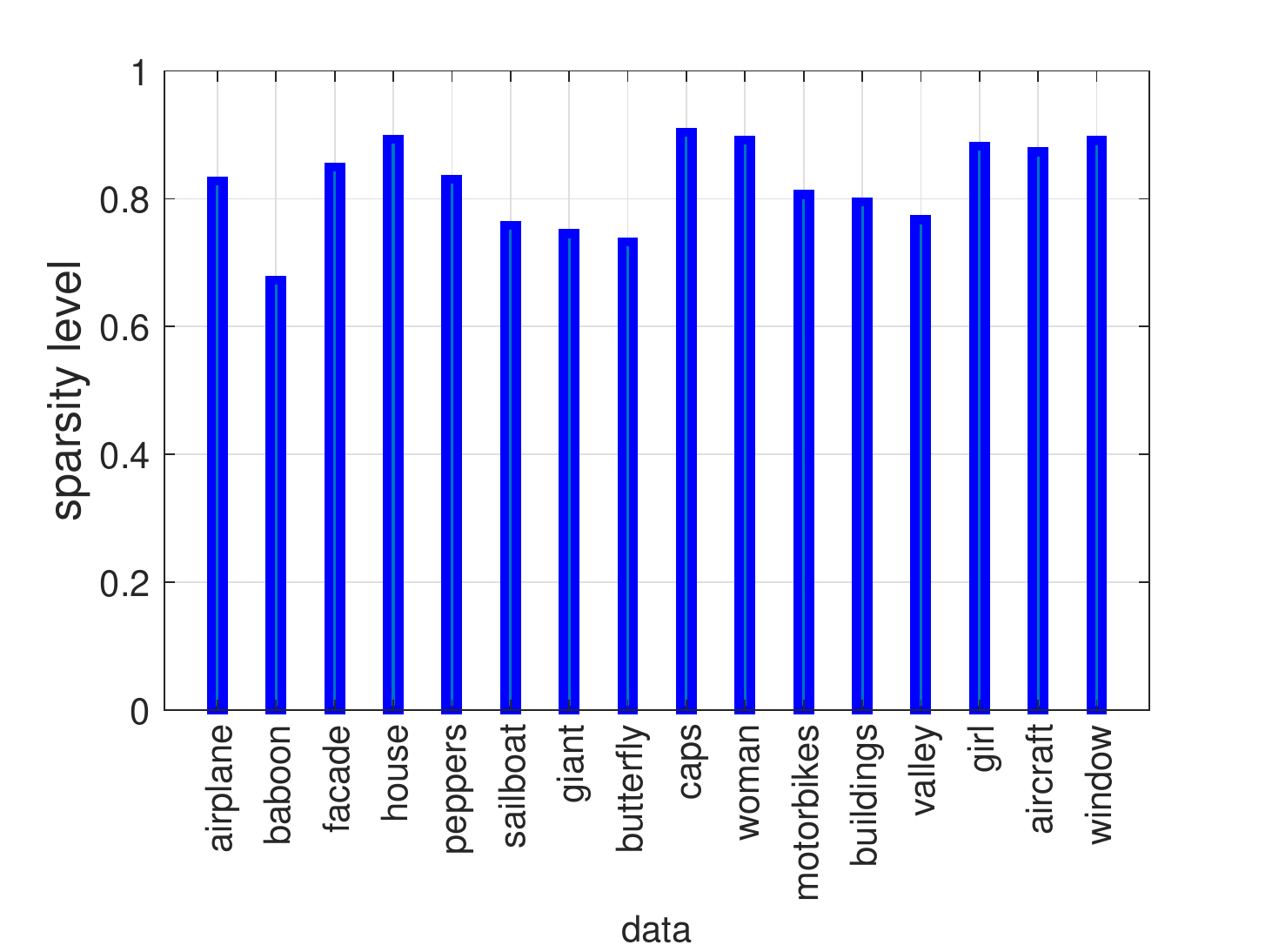}\\
	\caption{{Sparsity level of sixteen RGB images under DCT, where ``sparsity level'' is defined as the proportion of the number of zero elements in a tensor.}}\label{fig_rgbsparse}
\end{figure}

Now, we are concerned with the inpainting performance of our approach on the cases where the pixels of images are missed in a random way. In Fig. \ref{RGB_image_visual}, we first consider four scenarios on the sampling rate (`{\sffamily sr}'), i.e., {${\sf sr}=\{ 3\%, 5\%,10\%,20\%\}$, for five images. Clearly, except the deep learning method DP3LRTC}, the recovered images by our methods, i.e., DCT-WNN and DCT-IpST, are better than the other compared model-driven methods for the scenarios where {\sf sr}'s are lower than $10\%$ (see the first four columns of Fig. \ref{RGB_image_visual}). To investigate the full sensitivity of these methods to {\sf sr}, we further consider five scenarios on {\sf sr}'s from $3\%$ to $30\%$ for the images listed in Fig. \ref{RGB_image}, and plot PSNR values with respect to {\sf sr}'s in Fig. \ref{RGB_image_psnr}. It can be easily seen from Fig. \ref{RGB_image_psnr} that both DCT-WNN and DCT-IpST outperform the other compared methods except the deep learning method DP3LRTC in terms of taking higher PSNR values for all scenarios with ${\sf sr}\leq 20\%$.

\begin{figure}[h]
	\centering
	\text{~~~~~caps~~~~~~~~~~girl~~~~~~~~window ~~~~aircraft~~~~motorbikes~~~~~~}
	\includegraphics[width=.48\textwidth]{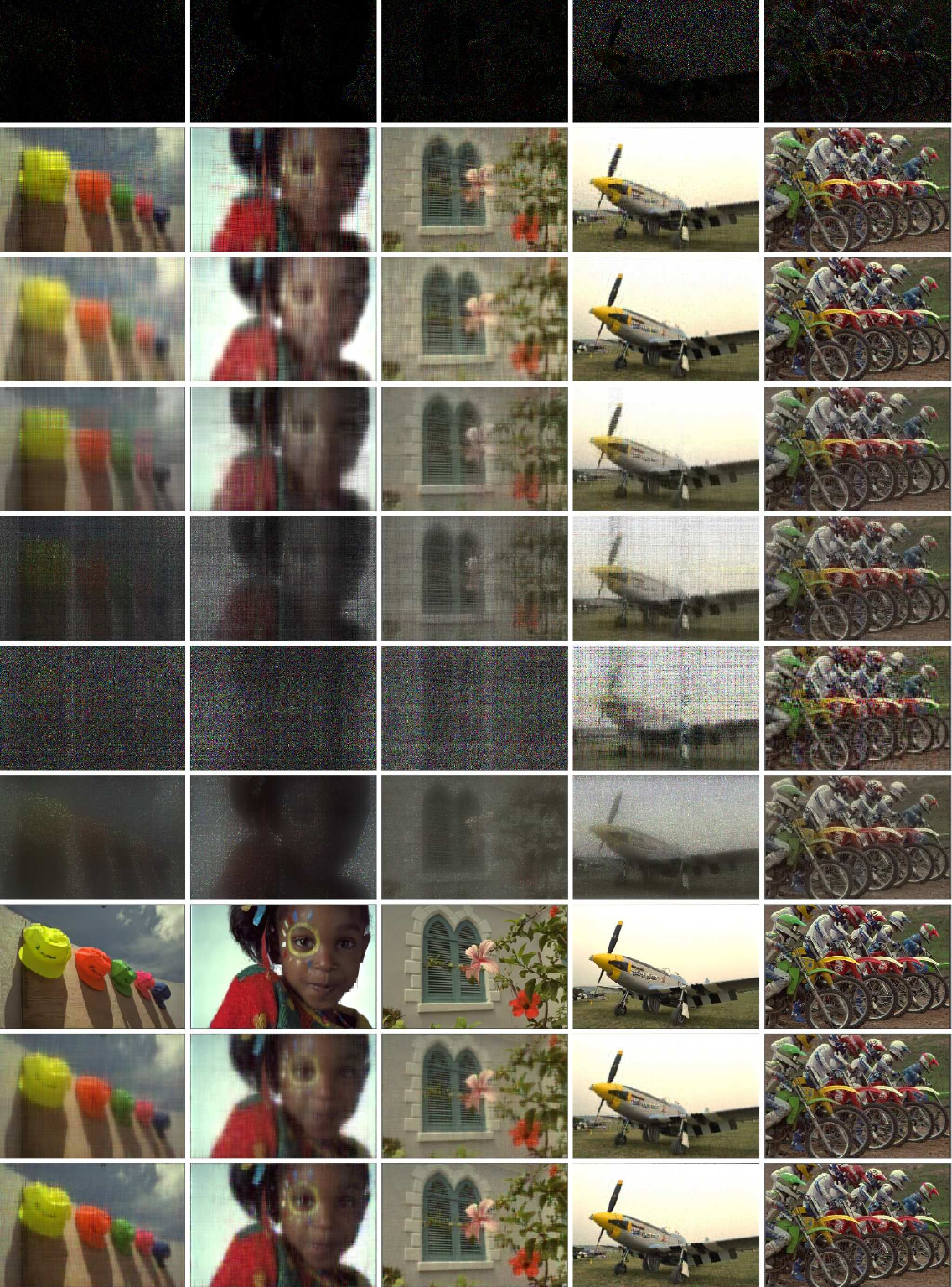}
	\caption{Images recovered by the nine algorithms for the cases with random missing information. The top row corresponds to the observed incomplete images. From left to right: downsampled images with ${\sf sr}=3\%, 3\%, 5\%, 10\% $ and $20\%$, respectively. From the second row to bottom: Images recovered by IpST, TTNNL1, F-TNN, TNN-3DTV, TCTF, WSTNN, DP3LRTC, DCT-IpST and DCT-WNN, respectively.}
	\label{RGB_image_visual}
\end{figure}

\begin{figure*}[!htbp]
	\centering
	\includegraphics[width=0.24\textwidth]{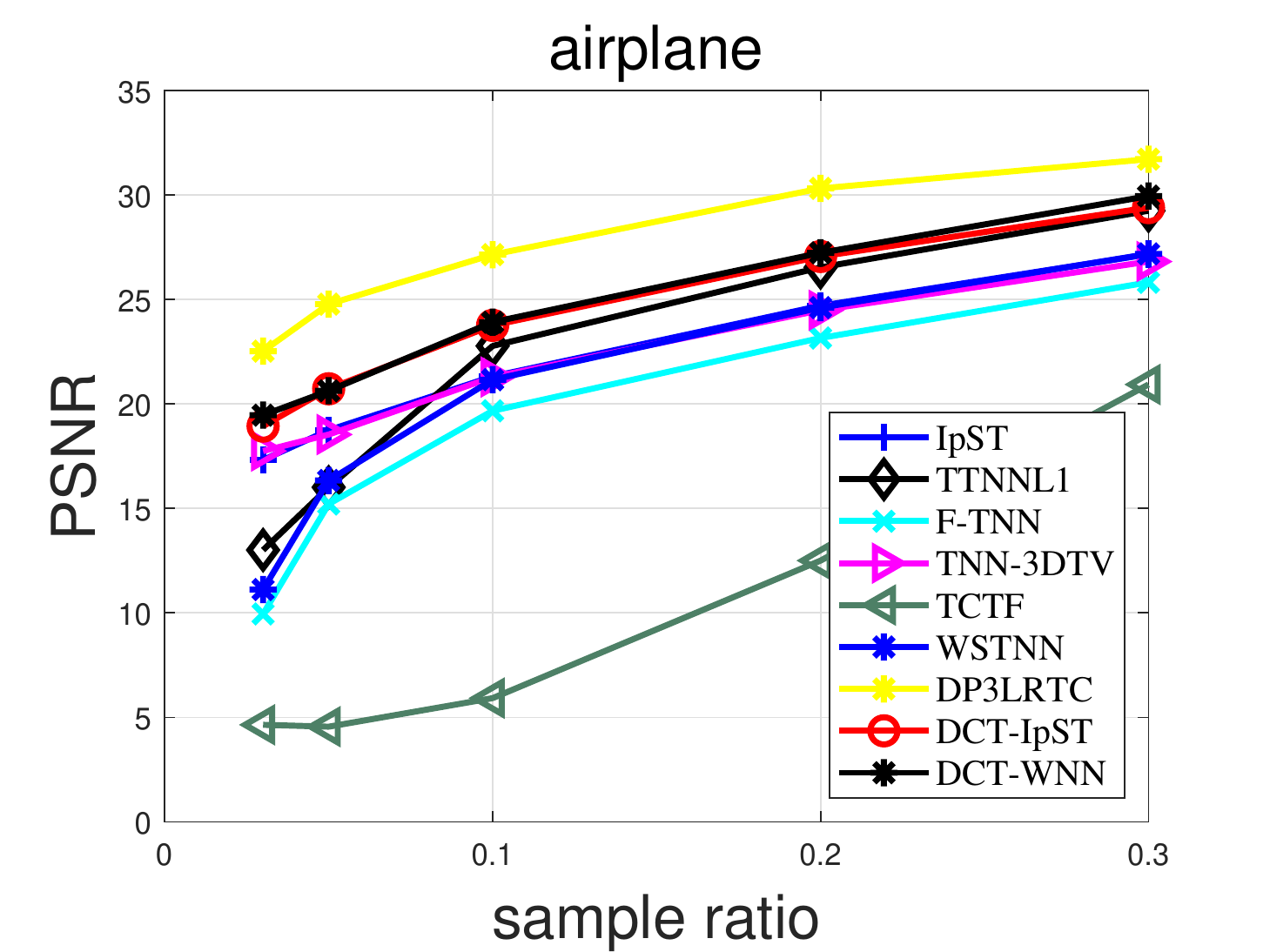}
	\includegraphics[width=0.24\textwidth]{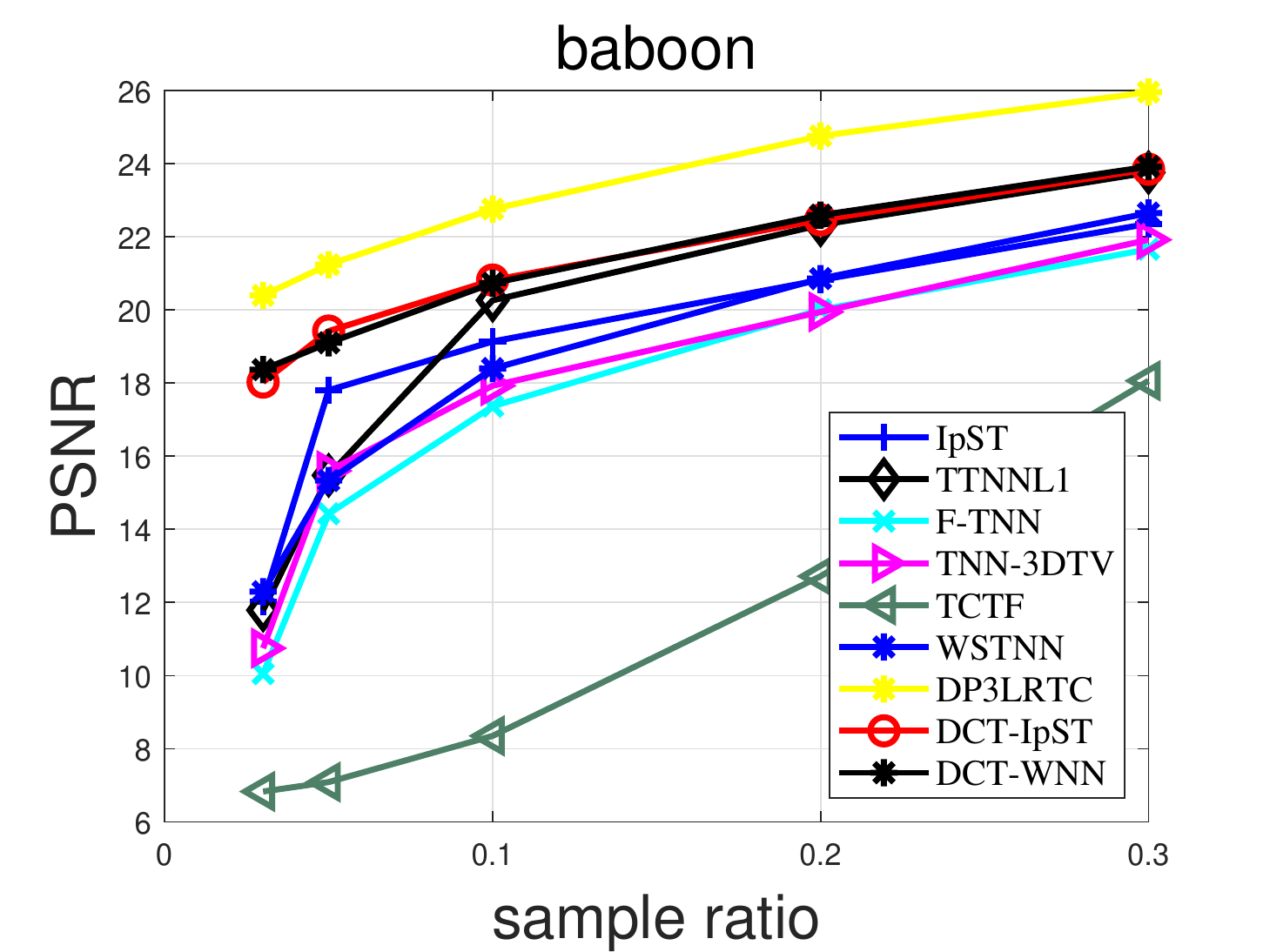}
	\includegraphics[width=0.24\textwidth]{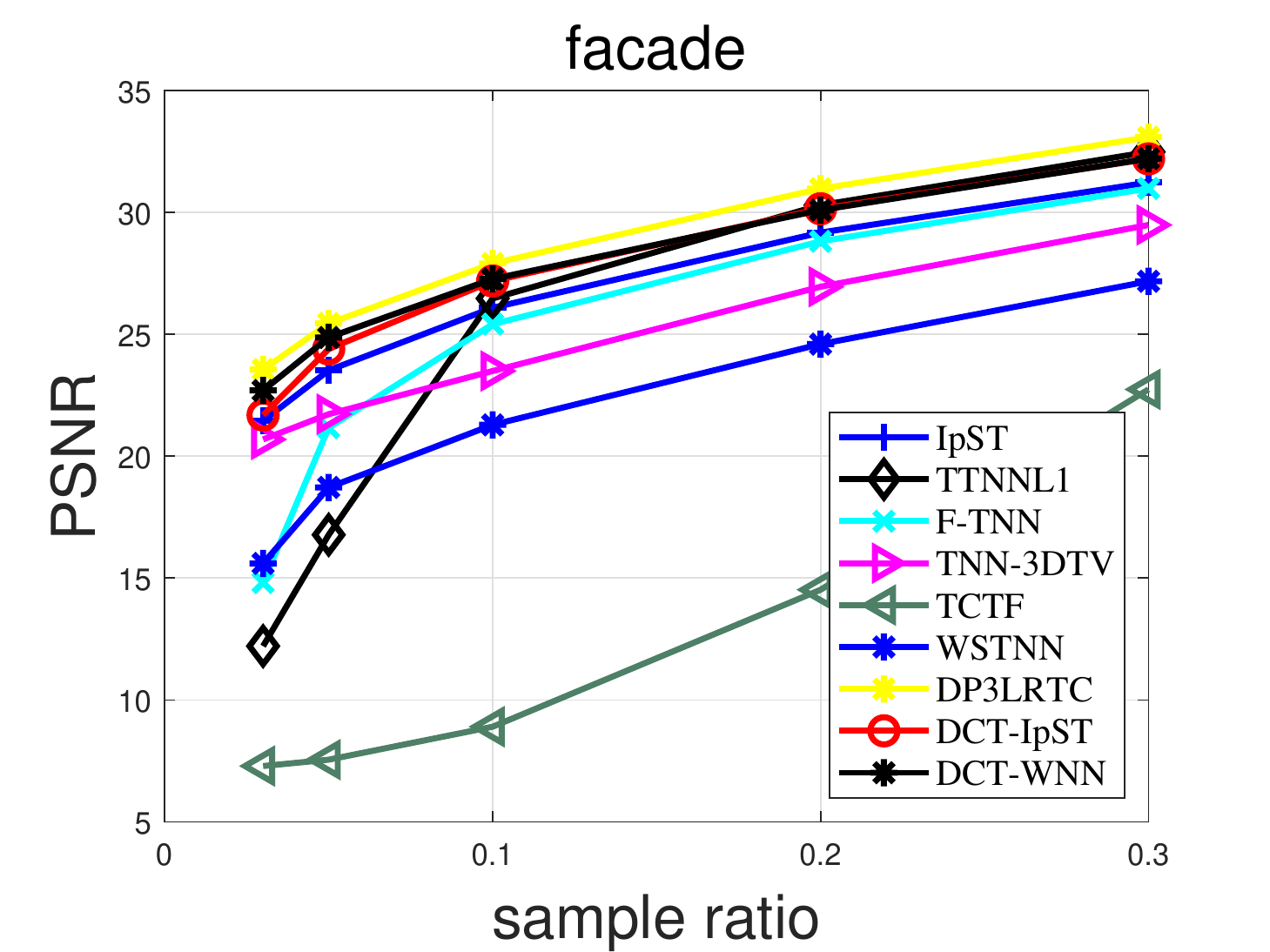}
	\includegraphics[width=0.24\textwidth]{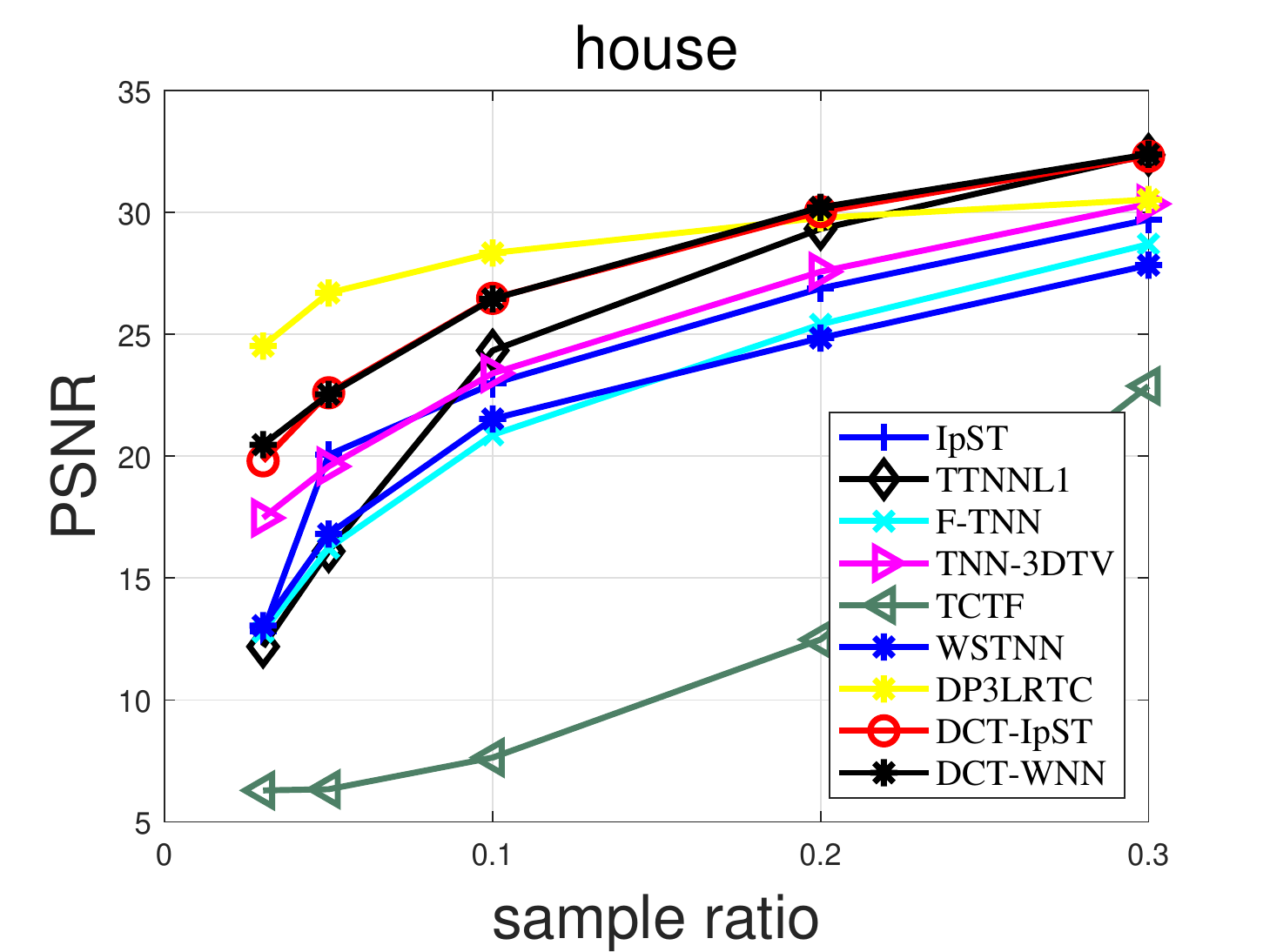}\\
	\includegraphics[width=0.24\textwidth]{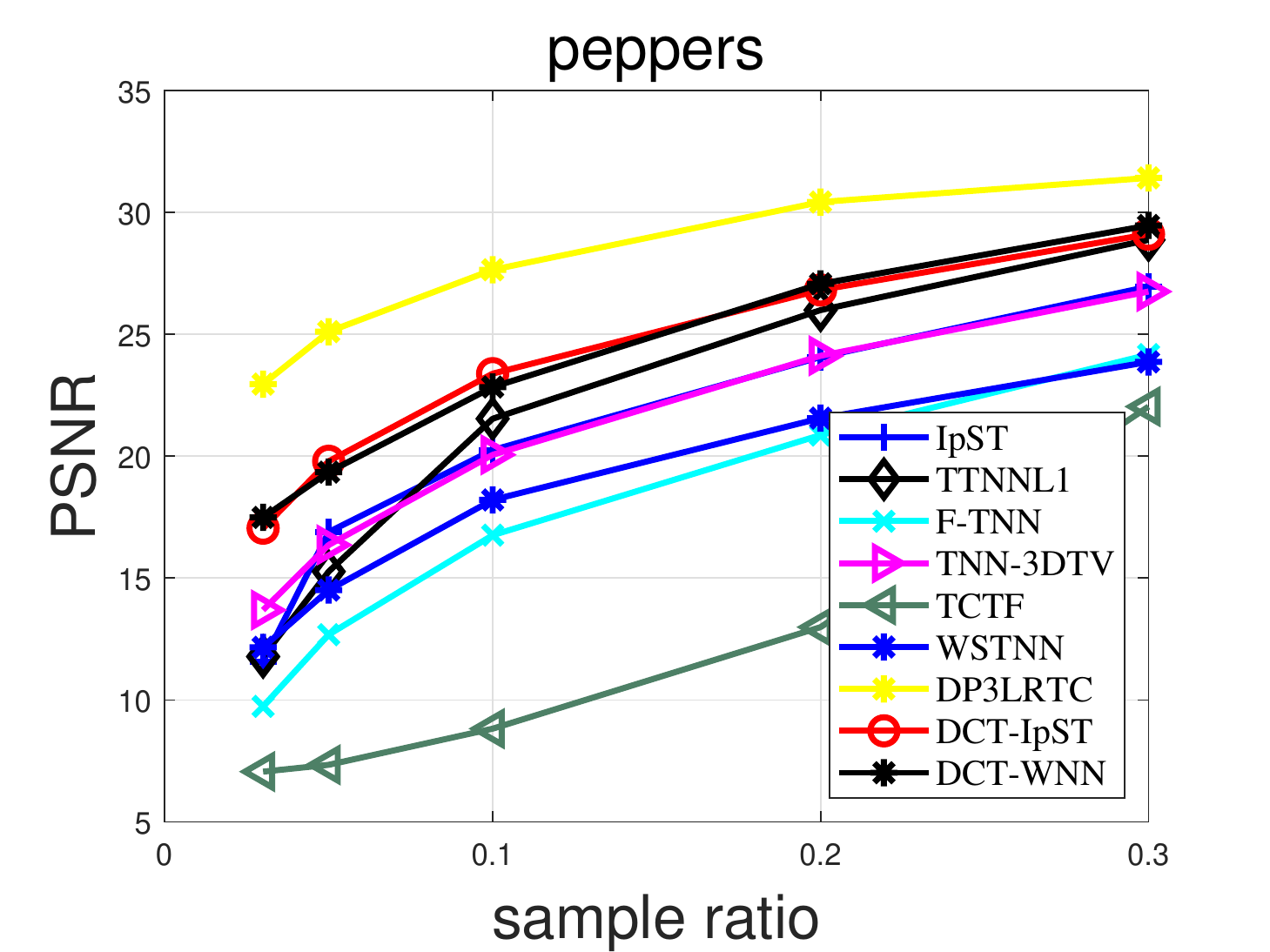}
	\includegraphics[width=0.24\textwidth]{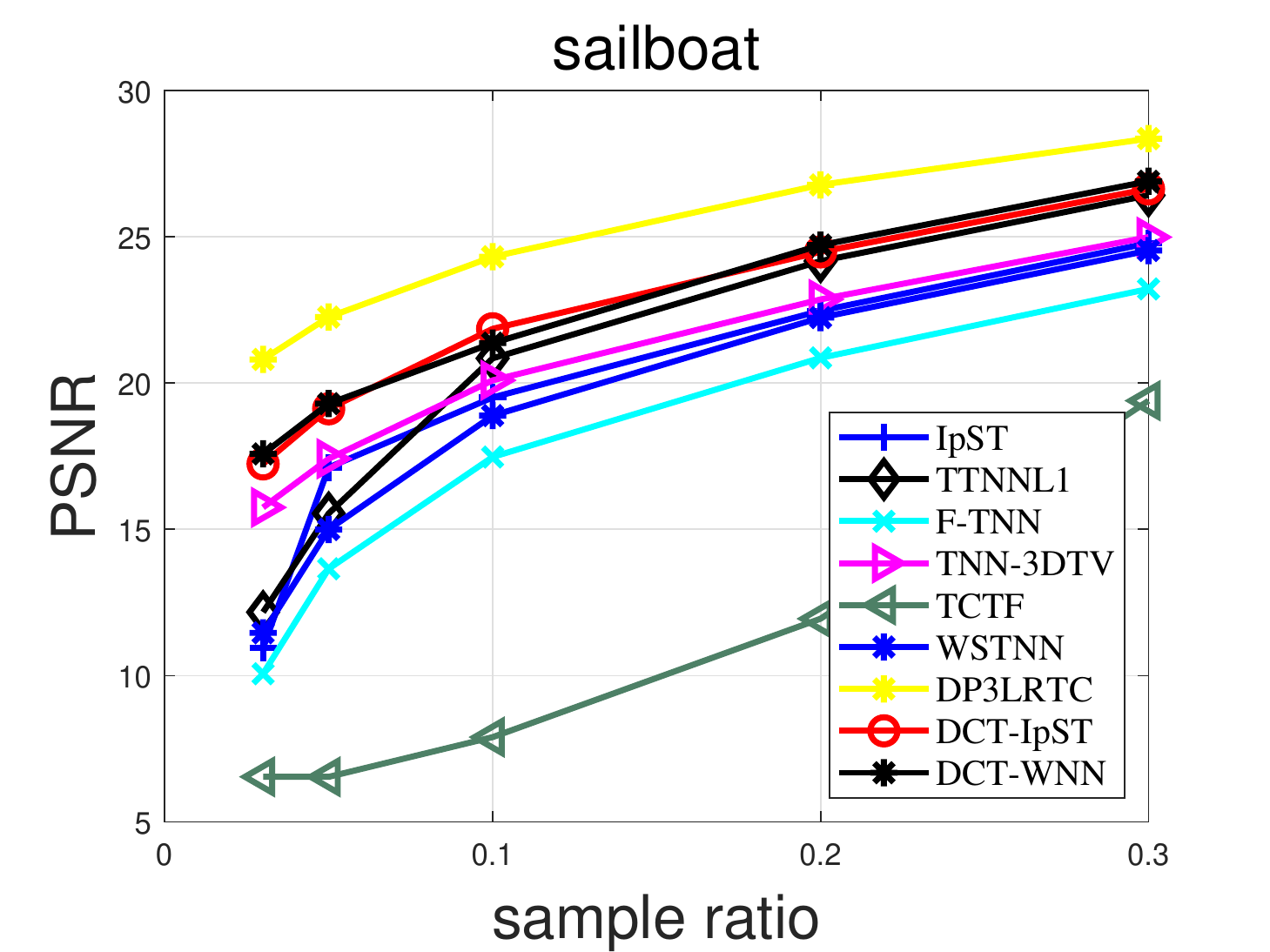}
	\includegraphics[width=0.24\textwidth]{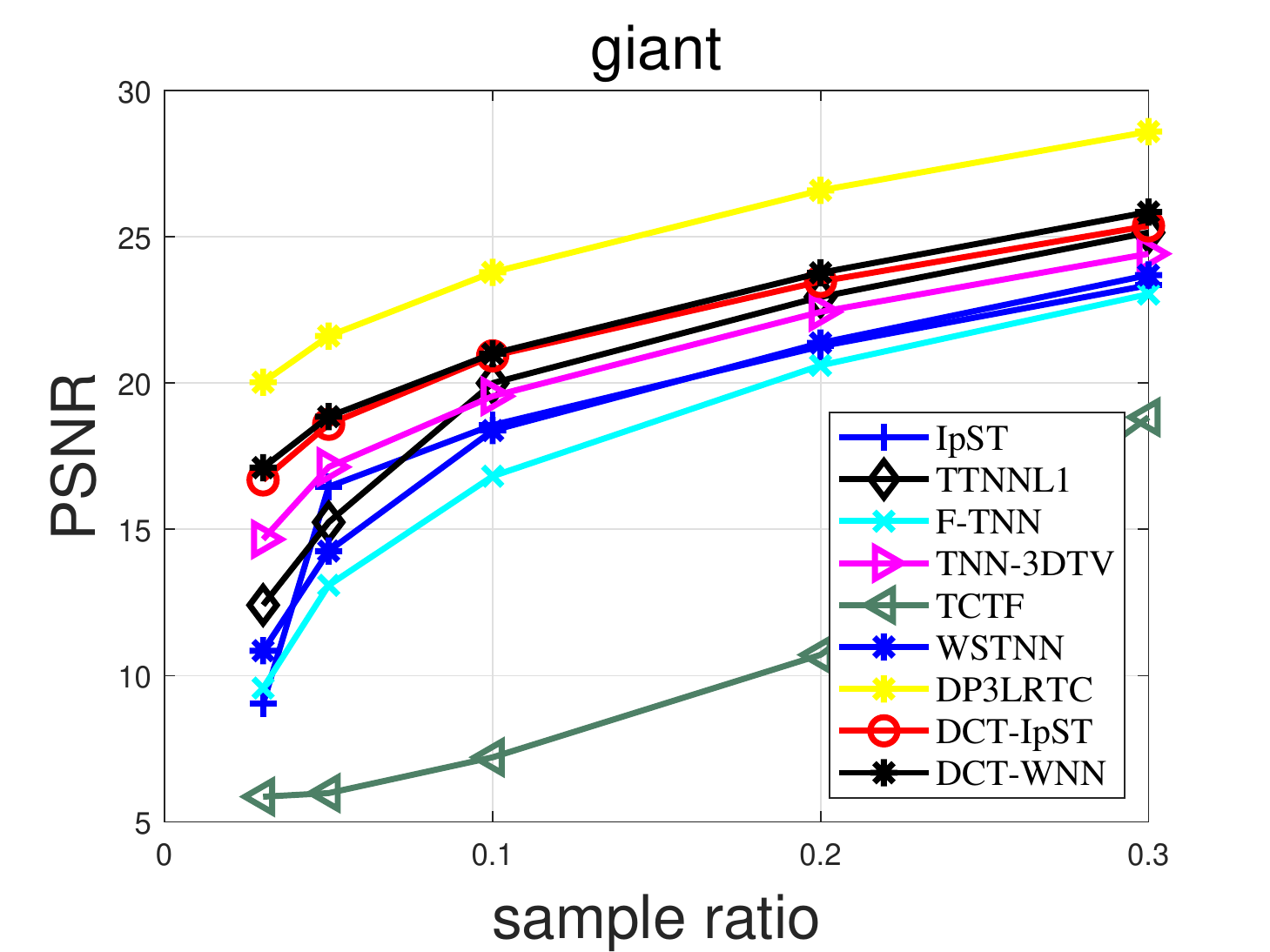}
	\includegraphics[width=0.24\textwidth]{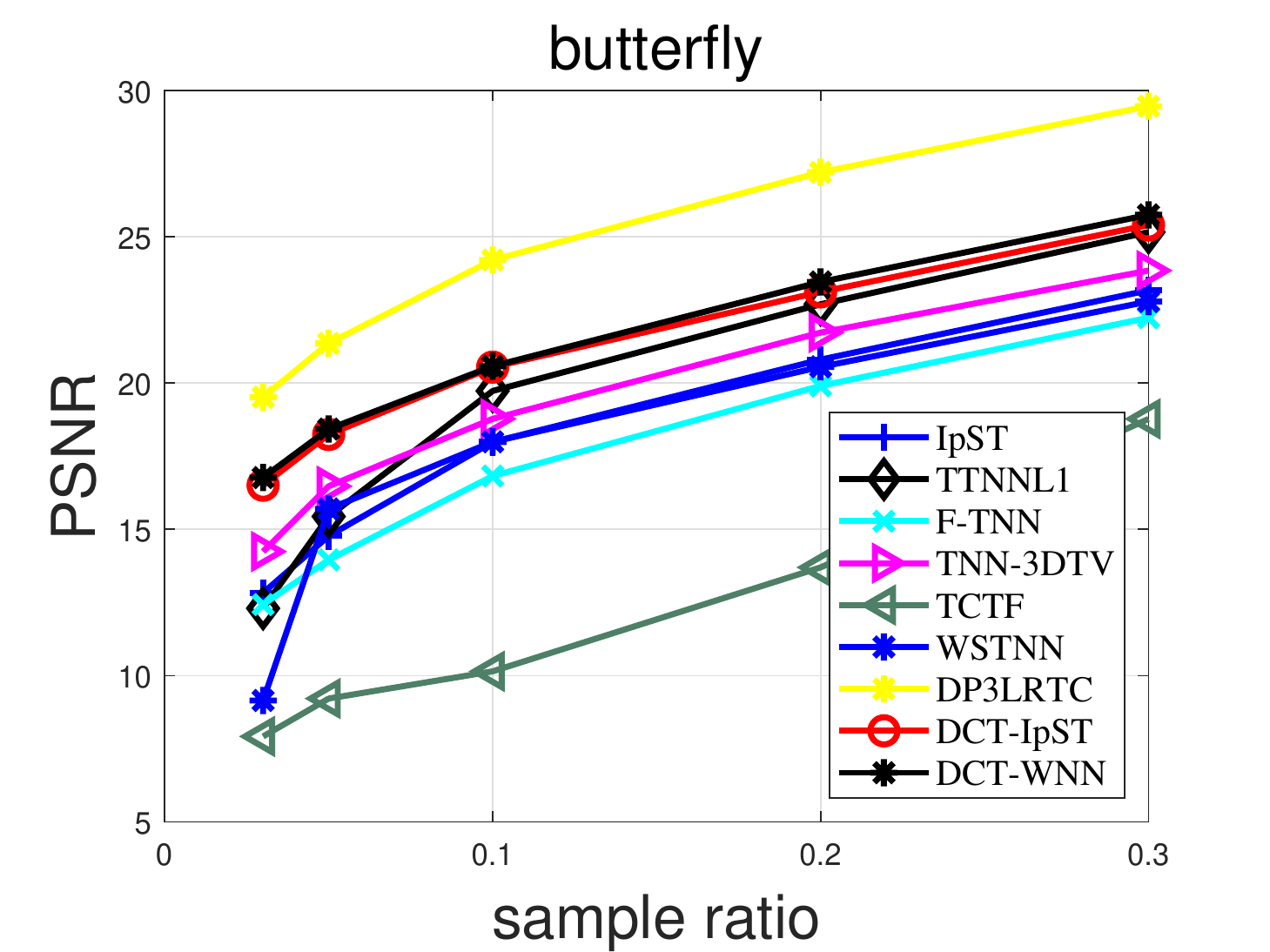}\\
		\includegraphics[width=0.24\textwidth]{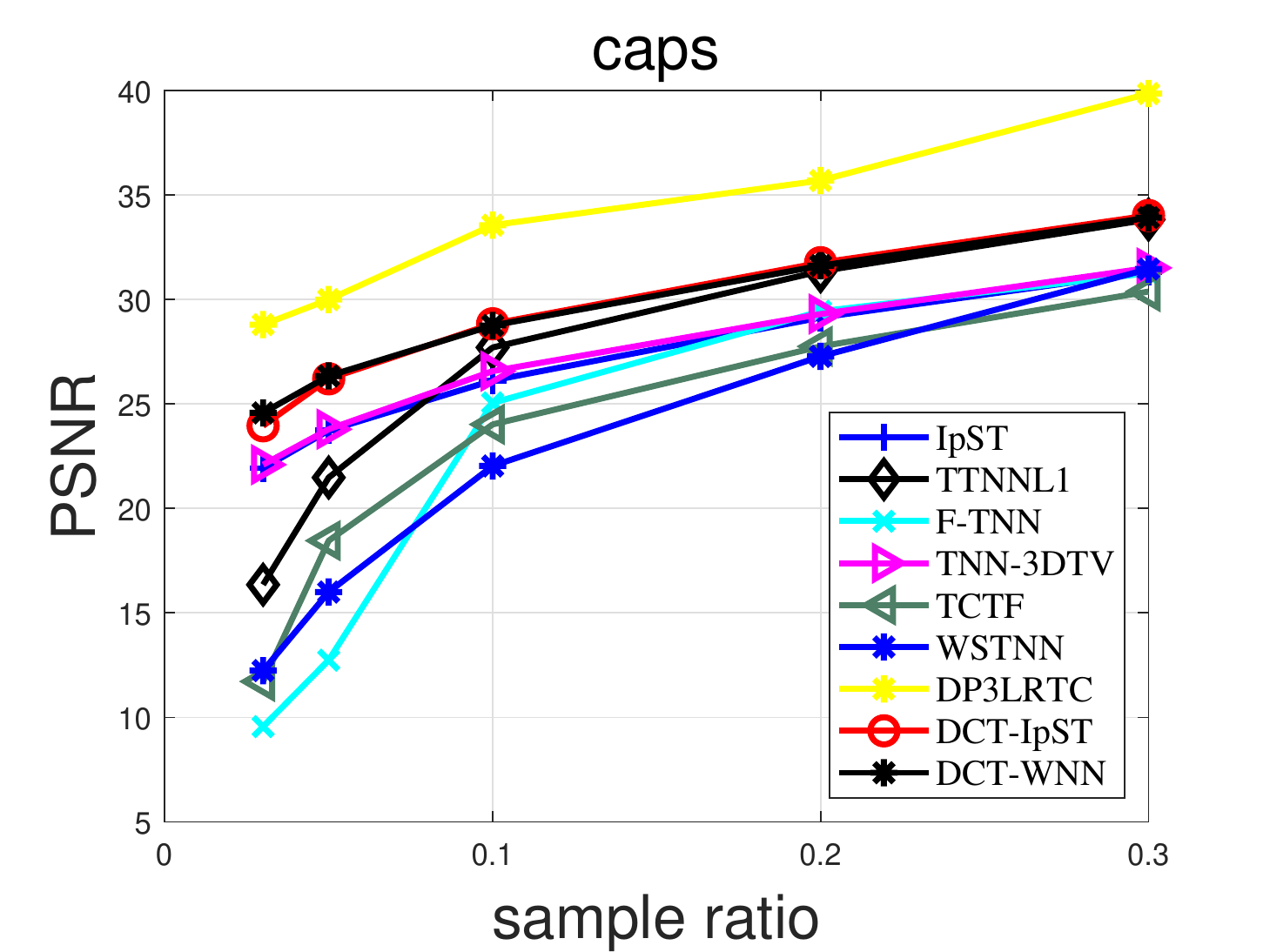}
	\includegraphics[width=0.24\textwidth]{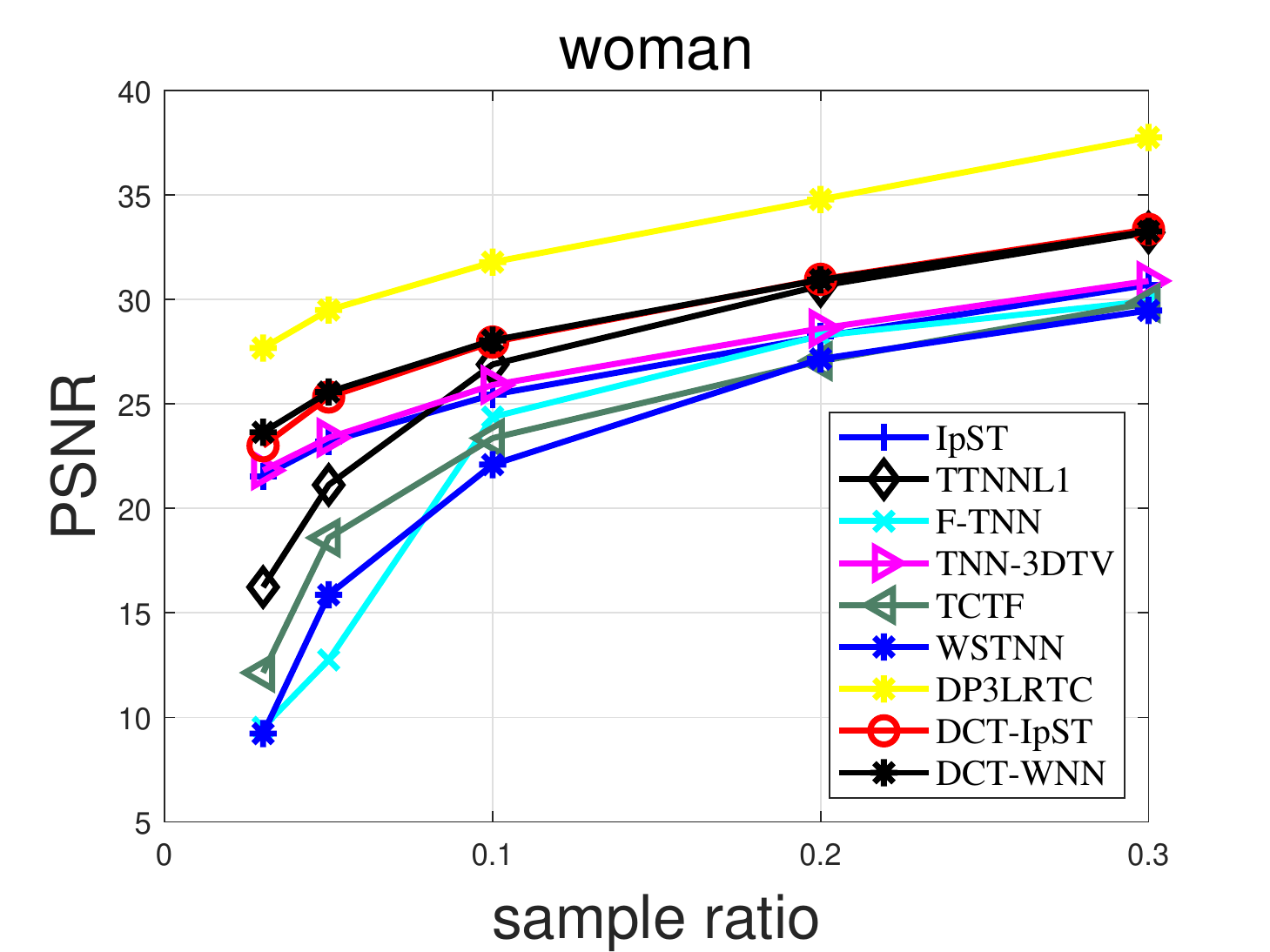}
	\includegraphics[width=0.24\textwidth]{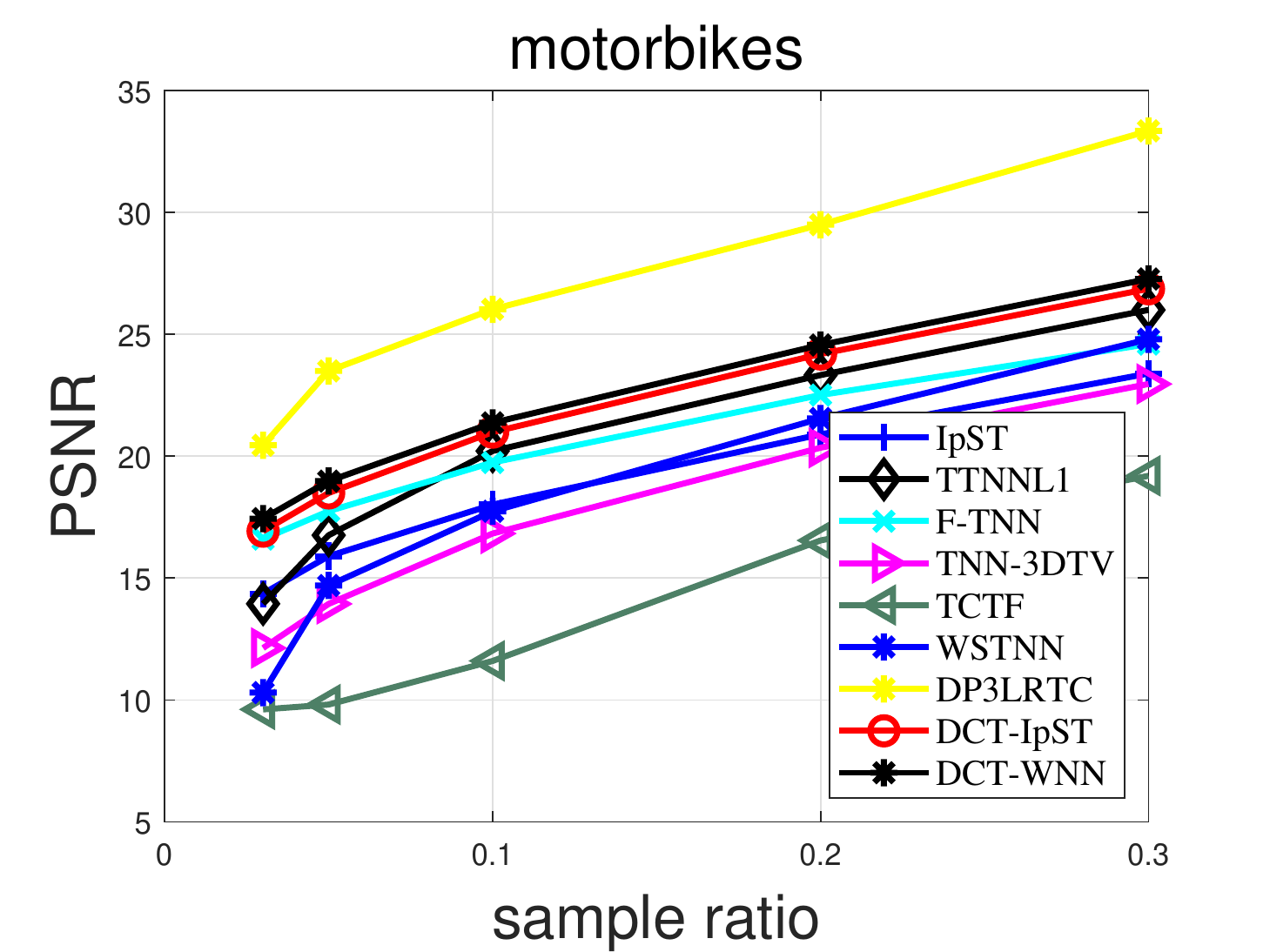}
	\includegraphics[width=0.24\textwidth]{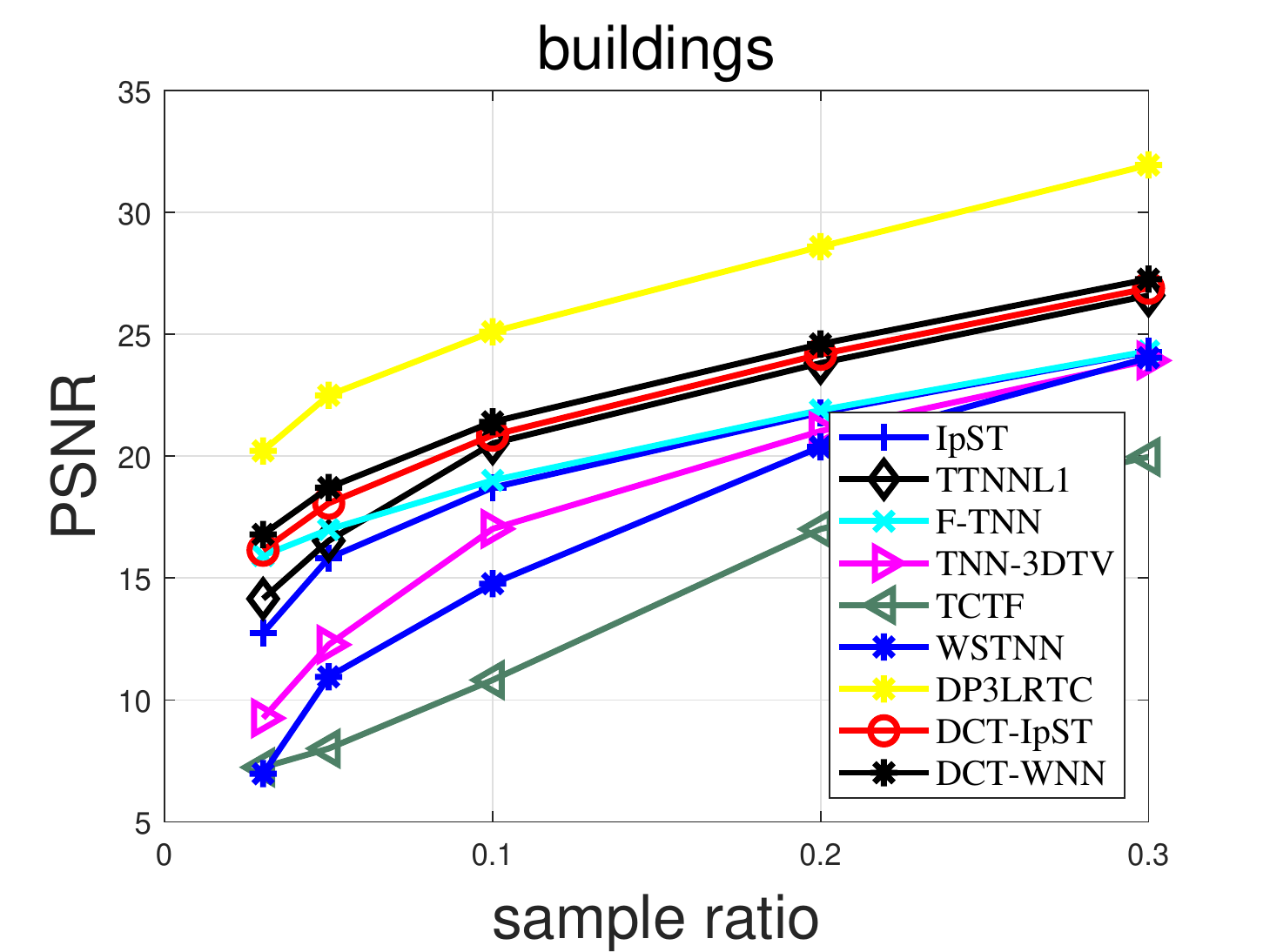}\\
	\includegraphics[width=0.24\textwidth]{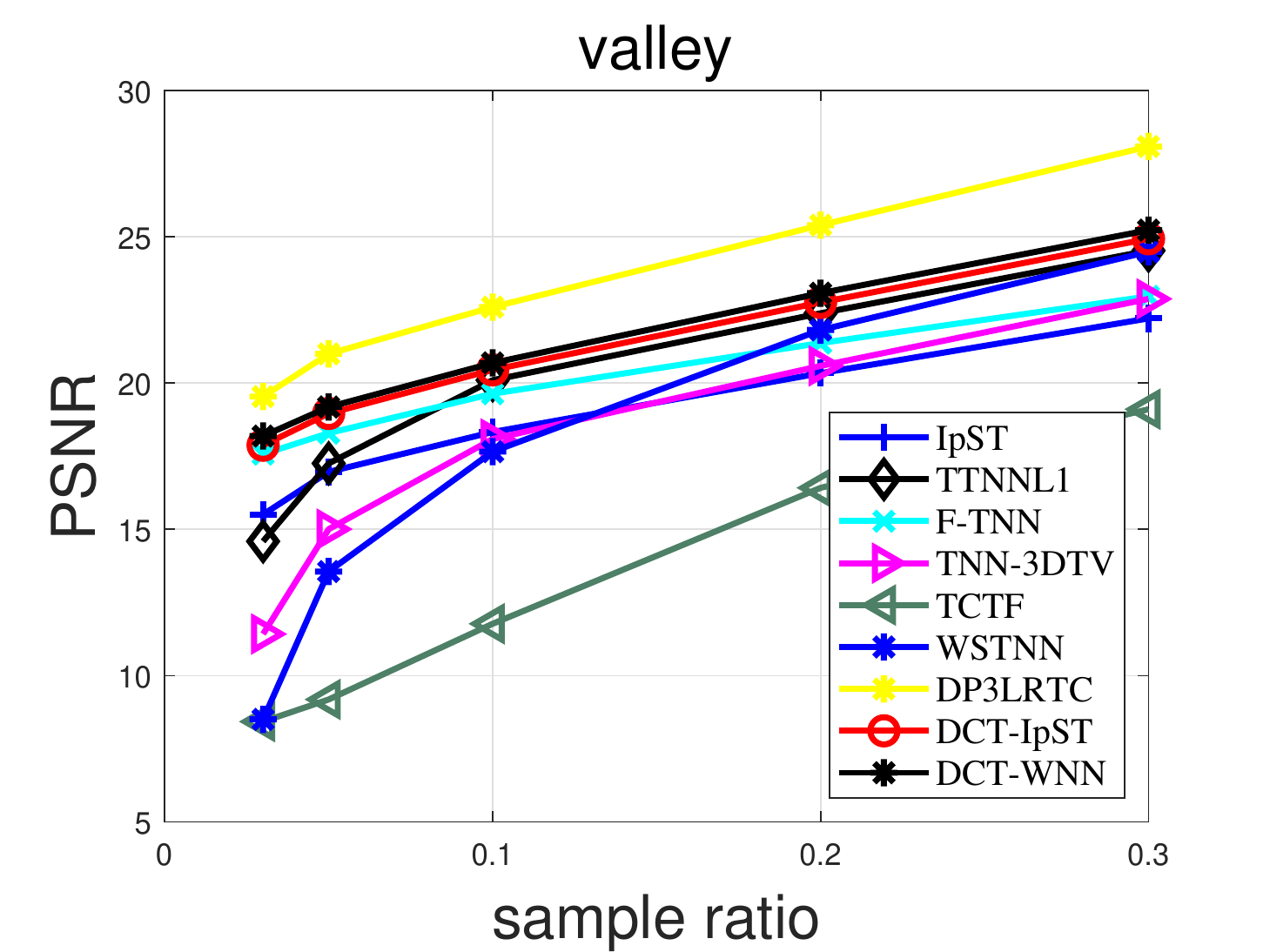}
	\includegraphics[width=0.24\textwidth]{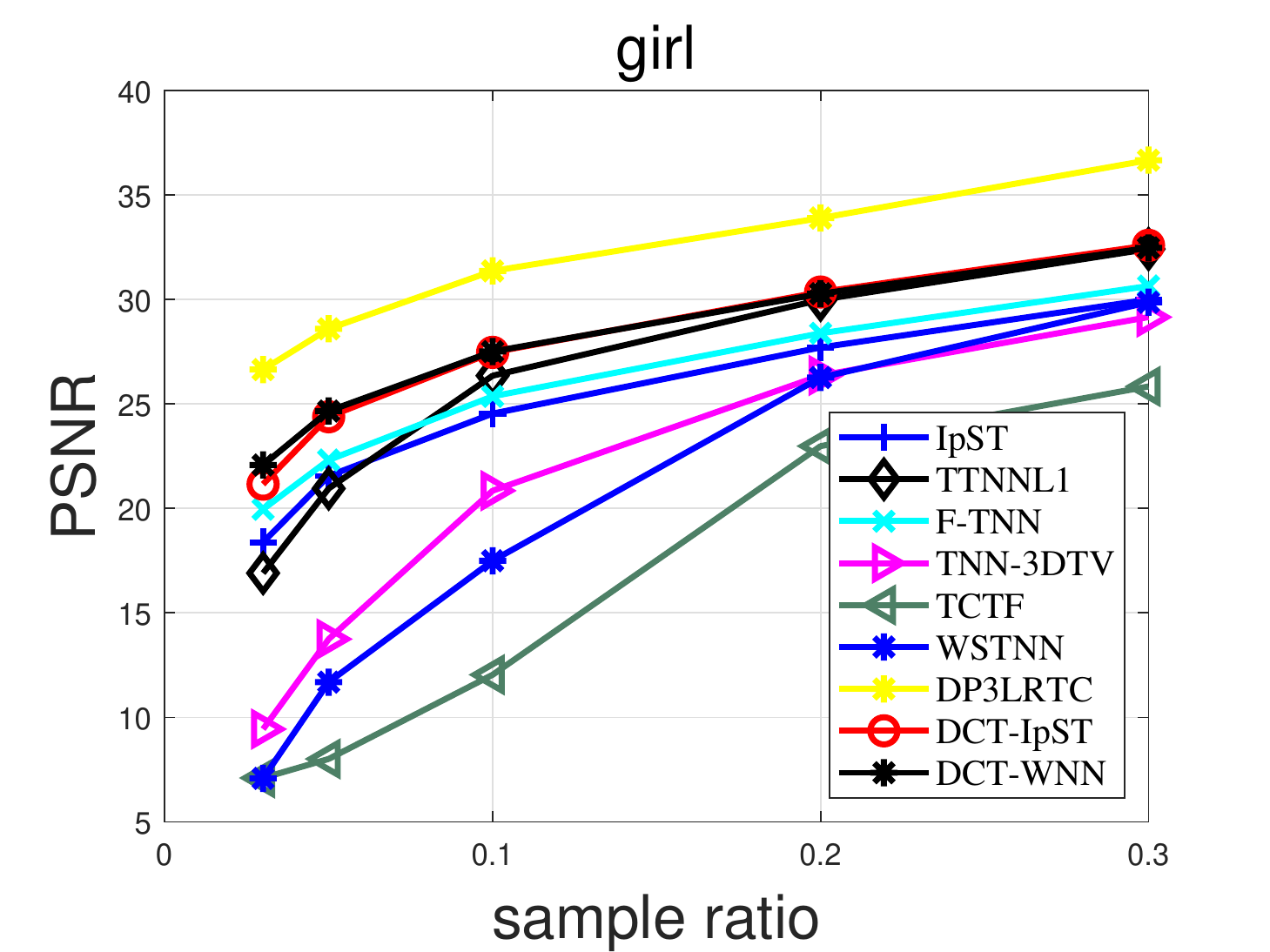}
	\includegraphics[width=0.24\textwidth]{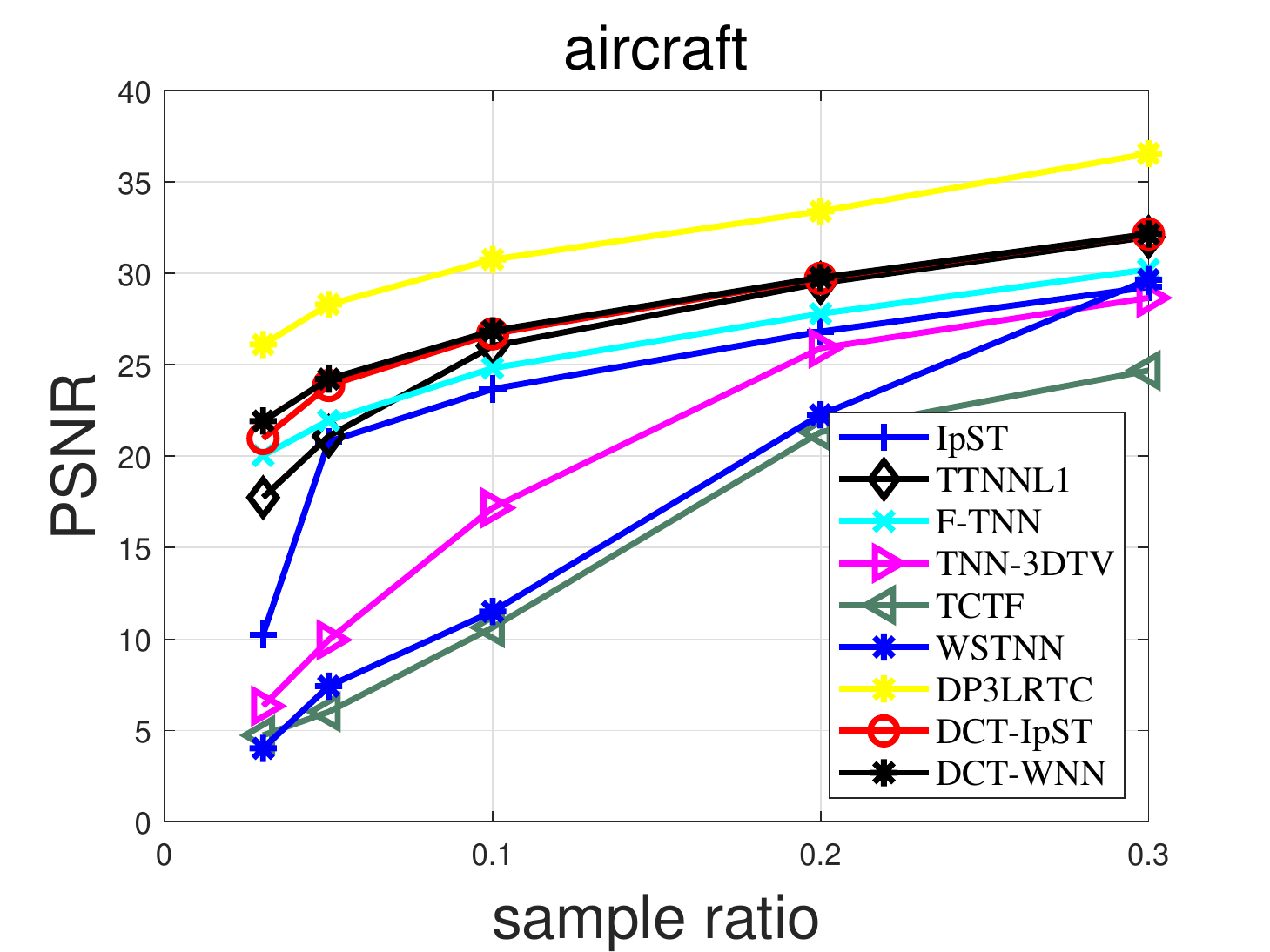}
	\includegraphics[width=0.24\textwidth]{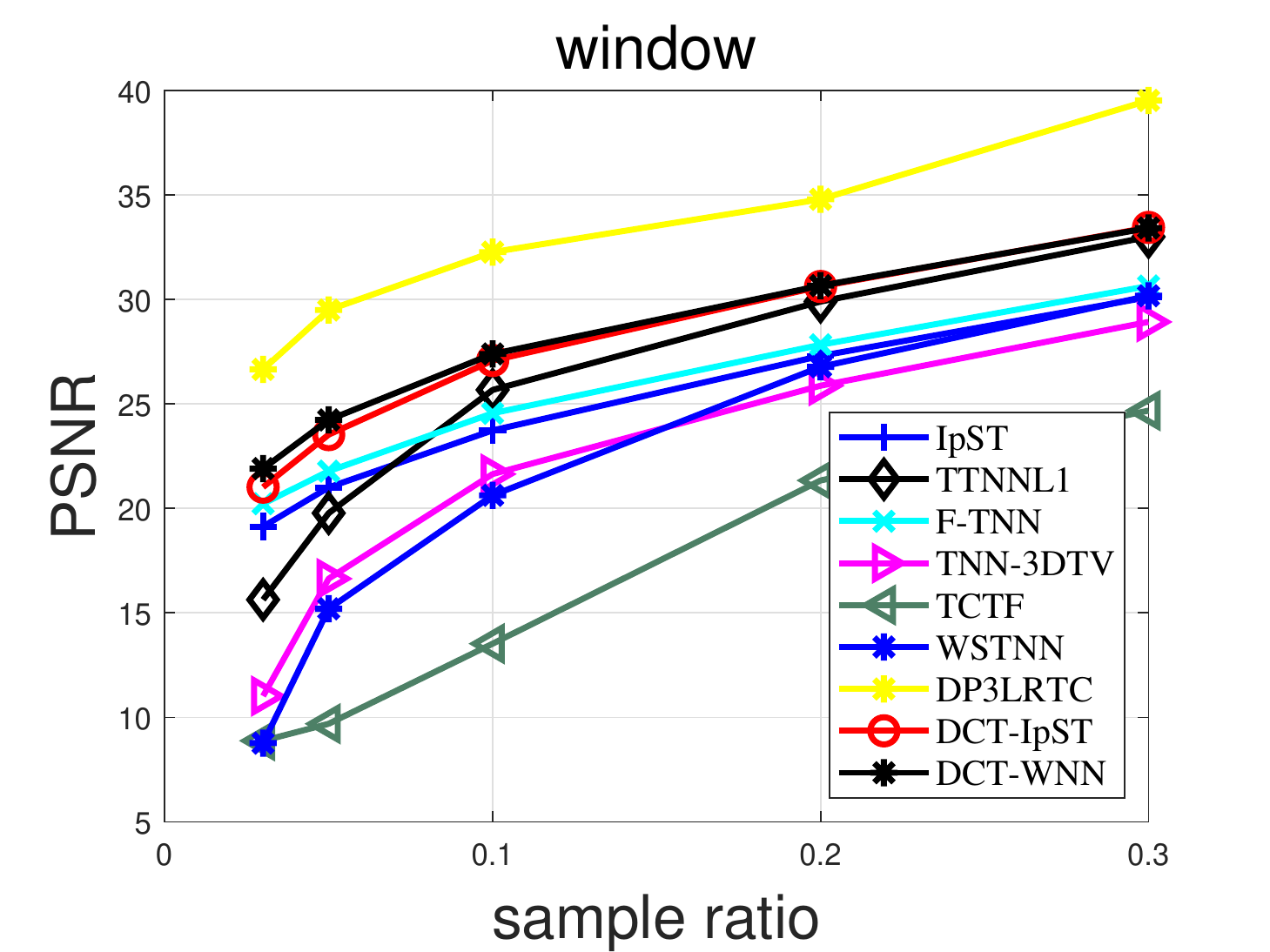}
	\caption{PSNR values with respect to different sampling rates for the sixteen RGB images in Fig. \ref{RGB_image}.}
	\label{RGB_image_psnr}
\end{figure*}


Below, we are interested in the numerical performance on structurally missing scenarios such as missing slices or shape/text mask (i.e., downsampling operator). Specifically, we consider five scenarios as shown by the first row of Fig. \ref{RGB_image_visual_stru}. Obviously, we see from the last two rows in Fig. \ref{RGB_image_visual_stru} that our DCT-WNN and DCT-IpST have promising ability to recover the structurally missing images.
 {Moreover, from the PSNR, SSIM,  and computing time in seconds (TIME(s) for short) reported in Table \ref{tab1}, we can see that the proposed DCT-WNN and DCT-IpST have better numerical performance than the compared methods except the deep learning method DP3LRTC. However, our DCT-WNN and DCT-IpST take less computing time than DP3LRTC to achieve comparable PSNR and SSIM values when all methods ran on a general personal computer without GPU acceleration. Those results efficiently demonstrate that the popular deep learning methods usually obtain better inpainting quality than the traditional model-driven approaches. However, model-driven approaches are still necessary for real-world applications due to their relatively low requirements on computer hardware.}

\begin{figure}[H]
	\centering
		\text{~~~~caps-1~~~~~~~caps-2~~~~~~~girl-1 ~~~~~~~girl-2~~~~~motorbikes~~~~~~}
	\includegraphics[width=.48\textwidth]{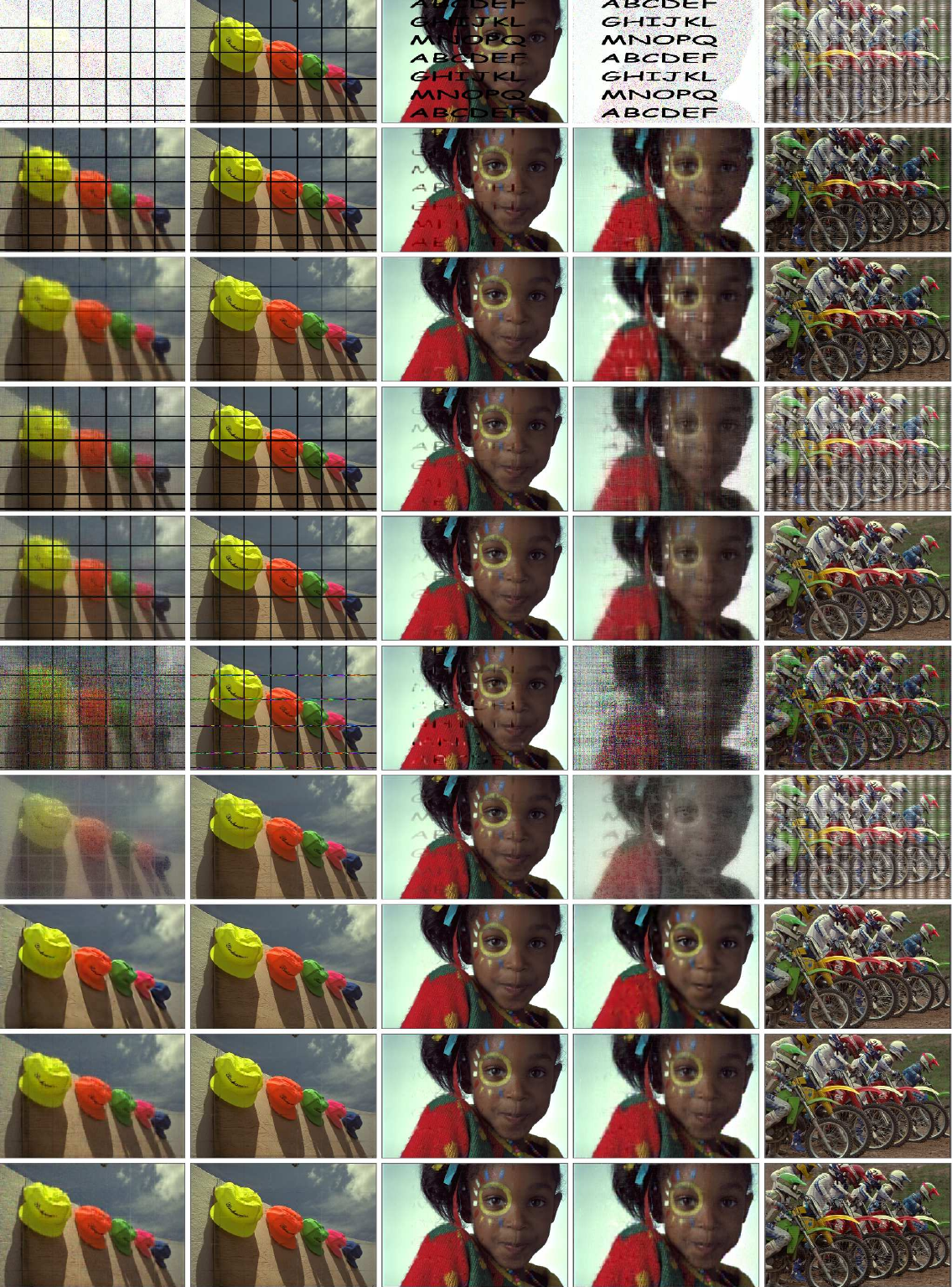}\\
	\caption{Images recovered by the nine algorithms for the cases with structurally missing information.  From the second row to bottom: Images recovered by IpST, TTNNL1, F-TNN, TNN-3DTV, TCTF, WSTNN, DP3LRTC, DCT-IpST and DCT-WNN, respectively.}
	\label{RGB_image_visual_stru}
\end{figure}

\begin{table*}
	\caption{Numerical results of the nine algorithms for recovering images with structurally missing pixels as tested in Fig. \ref{RGB_image_visual_stru}.}\label{tab1}
	\resizebox{\textwidth}{20mm}{
		\begin{tabular}{lllllllllllllllllllll}\toprule
			\multicolumn{1}{l}{\multirow{2}{*}{}}&\multicolumn{1}{l}{\multirow{2}{*}{Method}} &\multicolumn{3}{c}{caps-1} &&\multicolumn{3}{c}{caps-2}&&\multicolumn{3}{c}{girl-1}&&\multicolumn{3}{c}{girl-2}&&\multicolumn{3}{c}{motorbikes}\\
			\cline{3-5} \cline{7-9}\cline{11-13}\cline{15-17}\cline{19-21}
			&& PSNR & SSIM &  TIME && PSNR & SSIM& TIME && PSNR & SSIM & TIME&& PSNR & SSIM & TIME&& PSNR & SSIM & TIME\\\midrule		
			& IpST  &17.28  & 0.558 & 53.81 &&17.87  &0.749 & 52.21 && 25.01 & 0.892 & 55.29&& 23.24& 0.668 & 62.14&&16.01 &0.902 & 	29.03 \\
			
			& TTNNL  & 24.07& 0.711 & 55.37 && 26.86 & 0.857 & 55.31&& 31.22 & 0.942 & 53.25&&23.70& 0.724 & 56.93 && 21.62 &0.967&  46.87\\
			
			& F-TNN  & 17.00 & 0.540 & 135.08 && 17.87 & 0.749 & 25.13 && 27.61 & 0.916 & 97.64&&19.11&0.559 & 137.30 && 9.80&  0.861& 34.80\\
			
			& TNN-3DTV  & 18.96&0.602& 315.42 && 20.47  & 0.800 & 300.23 && 29.38 &0.934 & 361.91& & 22.47 & 0.712 & 323.01 && 31.96 & 0.990 & 322.12\\
			
			& TCTF  & 12.70 &0.130 &  26.43 && 18.22 & 0.753&30.23 && 25.59& 0.863 & 28.62& & 10.94& 0.135& 29.64 && 16.29 & 0.742 & 25.92\\
			
			& WSTNN  &19.27 &0.639 &  167.36 && 34.41 & 0.965& 154.31 && 29.05& 0.081 & 157.83& & 16.34& 0.576& 168.18 &&10.74& 0.884& 165.36\\
			
			& DP3LRTC  &29.04 &0.933 &  190.79 && 39.25 & 0.922& 218.53 && 33.93& 0.899 & 188.16& & 27.30 & 0.890 & 183.15 &&35.90 & 0.987& 182.67\\
			
			& DCT-IpST  &28.08& 0.842 & 102.59 &&35.84 & 0.973 & 100.13 &&31.79 & 0.946&99.67& &26.10 &0.821 & 101.77 &&32.90 & 0.991& 100.06\\
			
			& DCT-WNN  & 28.15 & 0.831 & 116.81&& 36.09 & 0.974& 102.33 && 31.76 & 0.943 &  100.07& & 26.22 &0.807 & 114.25 && 23.85 & 0.975& 101.55\\ \bottomrule	
	\end{tabular}}
\end{table*}

\subsection{Video recovery}
The above color images are actually low-dimensional third-order tensors. In this subsection, we consider higher dimensional cases, i.e., surveillance video data sets, which are natural third-order and fourth-order tensors corresponding to grayscale and color videos, respectively. Here, we choose six grayscale videos including  `\textbf{airport}', `\textbf{sidewalk}{\footnote{https://gr.xjtu.edu.cn/web/dymeng/3; the data can also be downloaded from: http://perception.i2r.a-star.edu.sg/bk\_model/bk\_index.html.}}', `\textbf{lunges}', `\textbf{biking}', `\textbf{baby}\footnote{http://crcv.ucf.edu/data/UCF101.php}' and `\textbf{basketball}\footnote{https://github.com/csyongdu/}' for experiments. Here, we add a `camera-shake' operation to  `\textbf{airport}' and `\textbf{sidewalk}'  so that their backgrounds are not static. Typical frames in these videos are depicted in Fig. \ref{video_depict}. 

We also first check the sparsity level of videos data sets when applying DCT to videos. The results in Fig. \ref{fig_videosparse} clearly demonstrate that videos also appear sparse structure under DCT, since the six video data sets have more than $60\%$ zeros, which means that DCT is an ideal technique for promoting sparsity of video data sets.
\begin{figure}[!htbp]
	\centering
	\text{airport}
	\vspace*{1 ex}
	
	\includegraphics[width=0.45\textwidth]{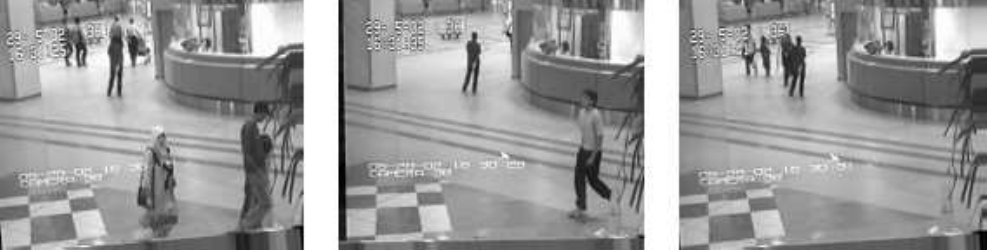}
	\vspace*{1 ex}
	\text{sidewalk}
	\vspace*{1 ex}
	\includegraphics[width=0.45\textwidth]{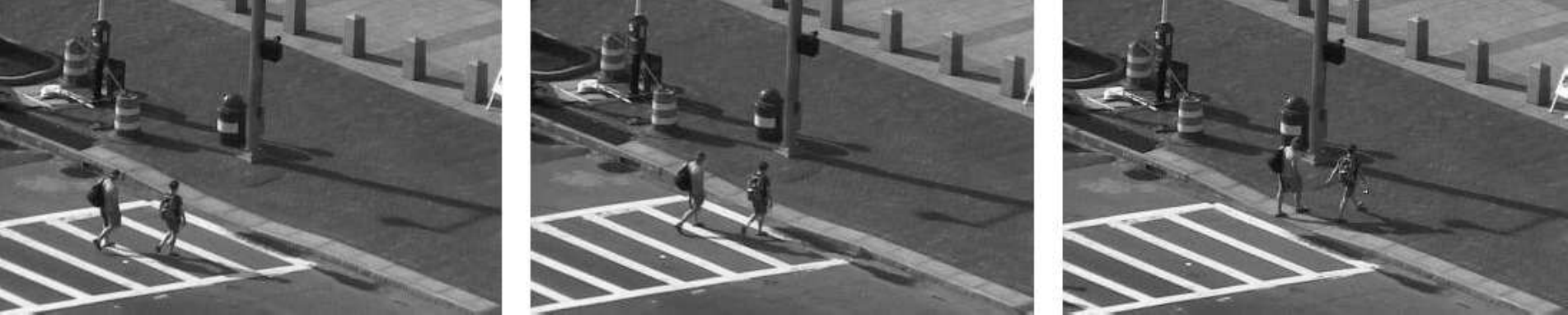}
	\vspace*{1 ex}
	\text{lunges}
	\vspace*{1 ex}
	\includegraphics[width=0.45\textwidth]{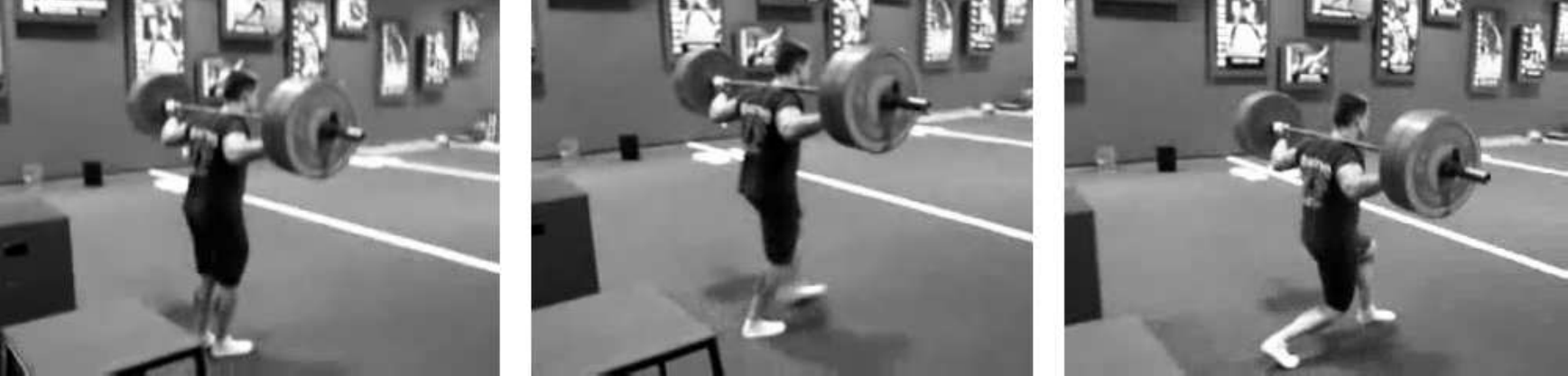}
\vspace*{1 ex}
	\text{biking}
	\vspace*{1 ex}
	\includegraphics[width=0.45\textwidth]{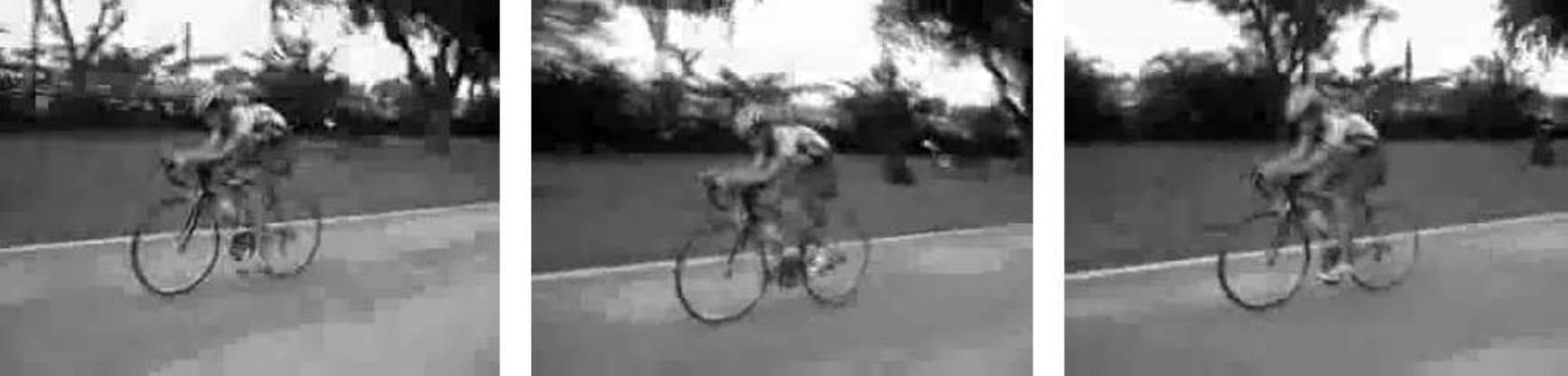}
\vspace*{1 ex}
	\text{baby}
	\vspace*{1 ex}
	\includegraphics[width=0.45\textwidth]{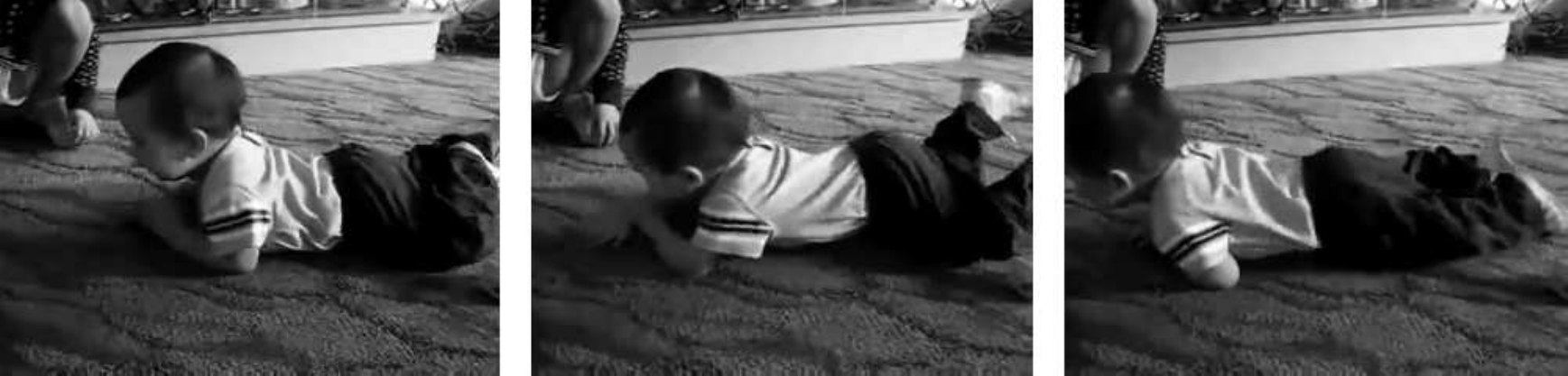}
	\vspace*{1 ex}
	\text{basketball}
	\vspace*{1 ex}
	\includegraphics[width=0.45\textwidth]{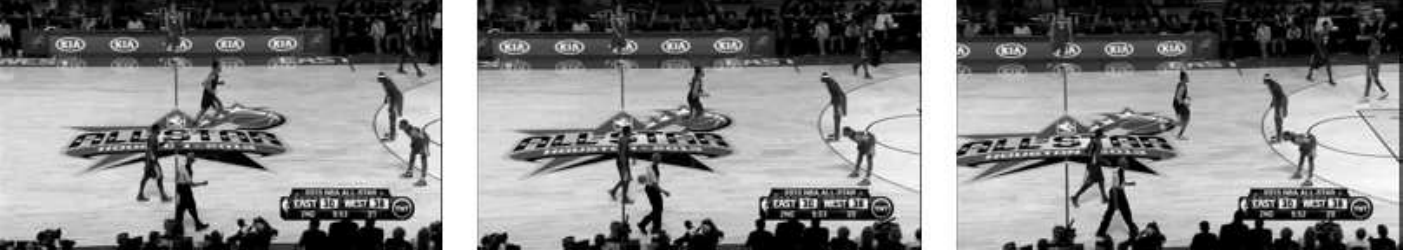}

	\caption{Some frames of six videos under test.}\label{video_depict}
\end{figure}

\begin{figure}[!htbp]
	\centering
	\includegraphics[width=0.45\textwidth]{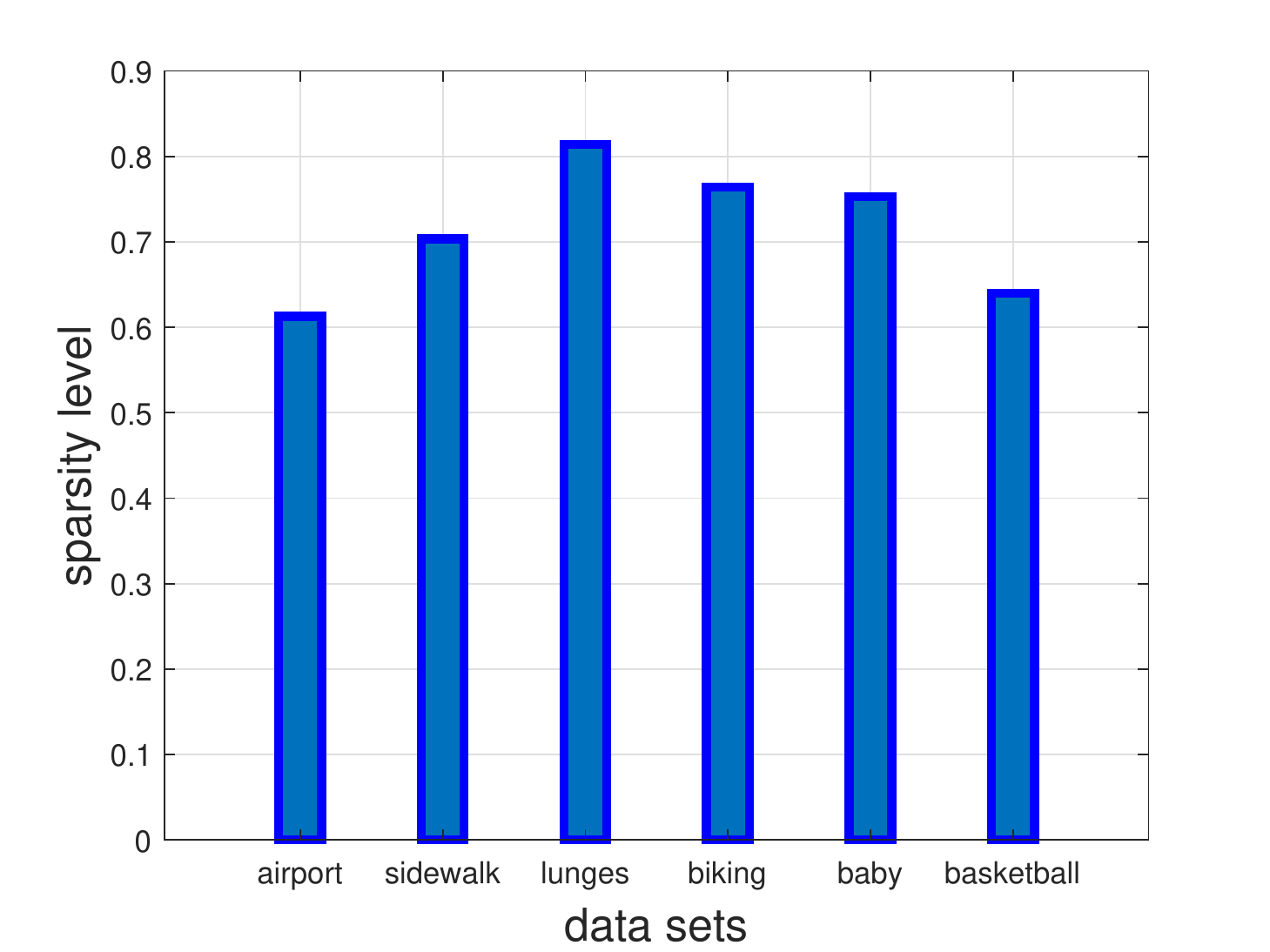}
	\caption{Sparsity level of video data sets under DCT.}\label{fig_videosparse}
\end{figure}

Now, we turn our attention to investigating the performance of our approach on incomplete videos data. In our experiments, we conduct four scenarios on the sampling rate, i.e., ${\sf sr}=\{5\%,10\%,20\%,30\%\}$, for the six video data sets. The PSNR, SSIM and TIME(s) values of the six videos obtained by the algorithms are summarized in Table \ref{video_result}, which shows that both DCT-IpST and DCT-WNN perform better than the other compared methods except the deep learning method DP3LRTC for the low sampling rate cases (i.e., ${\sf sr}\leq 20\%$). Here, it is noteworthy that the deep learning method DP3LRTC does not necessarily always achieve the best results. For example, our approach either DCT-WNN or DCT-IpST works better than DP3LRTC for ``\textbf{basketball}'' and ``\textbf{lunges}'' in some cases. 
More interestingly, our DCT-WNN and DCT-IpST always take much less computing time than DP3LRTC for video data sets to achieve comparable inpainting quality. In Fig. \ref{video_visual_sidewalk}, we first display  the destroyed (downsampled) frames (ten consecutive frames) are listed in the first row, and the recovered frames by different algorithms are displayed from the second row to the bottom. From the inpainted frames, it is obvious that both DCT-WNN and DCT-IpST can successfully obtain relatively ideal videos even for highly undersampled cases.

\begin{table*}[!htb]
	\caption{Numerical results of the ten algorithms for recovering video data sets.}\label{video_result}
	\centering
	\resizebox{\textwidth}{95mm}{
		\begin{tabular}{llllllllllllllllll}\toprule
			\multicolumn{1}{l}{\multirow{2}{*}{Videos Datasets}}&\multicolumn{1}{l}{\multirow{2}{*}{Method}} &\multicolumn{3}{c}{${\sf sr}=0.05$} &&\multicolumn{3}{c}{${\sf sr}=0.10$}&&\multicolumn{3}{c}{${\sf sr}=0.20$}&&\multicolumn{3}{c}{${\sf sr}=0.30$}\\
			\cline{3-5} \cline{7-9}\cline{11-13}	\cline{15-17}
			&&  PSNR& SSIM& TIME(s)  &&PSNR& SSIM& TIME(s) && PSNR& SSIM& TIME(s) && PSNR&SSIM&TIME(s)\\\midrule
			&IpST&17.62&0.371&16.68 &&19.6&0.485&14.01 &&22.23&0.661&12.41 &&24.32&0.778&10.64\\
			airport &TTNNL1&17.98&0.499&19.81 &&20.06&0.499&19.14 &&22.74&0.672&18.42 &&24.68&0.777&18.48\\
			$144\times176\times30$&F-TNN&17.94&0.372&111.83 &&20.68& 0.581&113.82 &&23.20&0.733&124.34 &&25.12&0.826&126.64\\
			&TNN-3DTV&17.44&0.297&169.11 &&19.15&0.426&136.12 &&21.47&0.593&124.34 &&23.46&0.713&120.98\\
			&TCTF&6.22&0.009&15.01 &&6.74&0.015&36.89 &&8.00&0.030& 38.51 &&11.51&0.091& 39.14\\
			&WSTNN&18.78&0.4403&35.34 &&20.28&0.533&27.91 &&22.29&0.669&19.64 &&23.99&0.761&16.71\\ 
			&SMF-LRTC&14.73&0.161&154.23 &&19.08&0.3967&183.62 &&23.25&0.685&196.20 &&24.93&0.763&210.60\\
			&DP3LRTC&21.14&0.629&394.33&&23.15&0.749& 384.42&&25.68&0.846& 371.67&& 27.82&0.911&364.87\\  
			&DCT-IpST&18.79&0.416&33.25 &&20.75&0.520&34.04 &&22.98&0.660&35.40 &&24.65&0.749&34.59\\ 
			&DCT-WNN&19.30& 0.445&27.41 &&21.31&0.559&31.09 &&23.38&  0.687&27.67 &&24.98&0.763&25.53\\ \midrule
			
			&IpST&19.71&0.537&64.84 &&23.30&0.684&62.70 &&27.22&0.823&50.37&&29.79&0.888&44.22\\
			sidewalk &TTNNL1&20.40&0.557&72.28 &&23.67&0.696&72.47 &&27.48&0.826&72.57 &&29.80&0.881&72.46\\
			$220\times352\times30$&F-TNN&20.12&0.551&478.43 &&21.47& 0.581& 475.44&&24.08&0.721&467.23&& 26.27& 0.800&462.11\\
			&TNN-3DTV&19.35&0.463&567.70 &&21.17&0.567&462.03&&23.76&0.691&307.24 &&26.00& 0.781&253.84\\
			&TCTF&6.22&0.013&121.53 &&9.36&0.019&121.11&&13.39&0.114& 129.77 &&18.58& 0.410& 121.26\\
			&WSTNN&23.30&0.730&100.96 &&25.72&0.813&81.01&&28.38&0.880&62.80&&30.32&0.915&48.75\\ 
			&SMF-LRTC&18.38&0.382&506.63 &&23.68&0.642&546.37 &&27.91&0.810&622.66 &&29.96&0.865&631.38\\ 
			&DP3LRTC&23.83&0.749&1230.33&&27.37&0.840& 1273.42&&30.57&0.906& 1272.67&& 32.81&0.937&1215.87\\ 
			&DCT-IpST&21.39&0.594&133.86 &&24.41& 0.711&127.81&&27.75&0.817&122.64&&29.85&  0.869&122.87\\ 
			&DCT-WNN&22.04& 0.624&106.72 &&25.14&0.740&105.36 &&28.31&  0.833&109.45&&30.29&0.880& 110.15\\ \midrule
			
			&IpST&15.16&0.324&29.98 &&17.08&0.451&27.26 &&19.81&0.632&23.11 &&21.90&0.751&19.97\\
			basketball &TTNNL1&15.95&0.405&31.87 &&18.32&0.570&28.33 &&21.21&0.704& 27.76 &&23.37&0.808&28.21\\
			$144\times 256\times30$&F-TNN&17.33&0.465&178.08&&19.09& 0.568&182.02 &&21.81&0.723&183.40 && 23.93&0.819&189.24\\
			&TNN-3DTV&16.12&0.351&290.96 &&17.86&0.494&217.19&&20.08&0.593&141.59&&21.95&0.749&116.23\\
			&TCTF&	5.87&0.015&26.65 &&6.56&0.029&36.89 && 8.34&0.083& 58.17 &&12.97&0.294& 59.13\\
			&WSTNN&16.49&0.430&45.76 &&19.02&0.566&46.02&&21.29&0.702&28.76&&23.31&0.797&23.99\\ 
			&SMF-LRTC&14.36&0.216&221.59 && 17.75&0.403&251.97 &&20.88&0.609&284.98&& 22.73&0.717&306.55\\ 
			&DP3LRTC&17.14&0.561&743.33&&19.17&0.688& 987.42&&22.01&0.817& 984.67&& 24.26&0.887&986.87\\ 
			&DCT-IpST&16.89&0.421&52.67&&20.75&0.559&53.45 &&22.10&0.708&53.15&&24.03&0.786&56.45\\ 
			&DCT-WNN&17.24& 0.438&41.45&&19.29&0.578&43.28 &&22.42&  0.717&44.86&& 24.32&0.790&46.26\\ \midrule
			
			&IpST& 21.30&0.663&95.15 &&25.38&0.812& 78.97 &&31.15&0.932&67.64 && 35.50&0.970&57.82\\
			lunges &TTNNL1&21.41&0.678&70.51 && 25.88&0.837&68.60 &&31.67&0.942&64.75&&35.42&0.971&64.12\\
			$240\times320\times30$&F-TNN&23.33&0.809&448.61&&25.60& 0.867&456.51&&28.91&0.922&460.82 &&31.70&0.950&462.13\\
			&TNN-3DTV&22.70&0.777&784.30 &&25.41&0.864& 565.78&&29.04&0.926&386.18&&32.19&0.957&310.20\\
			&TCTF&6.37&0.011&58.18 &&7.45&0.021&58.39 &&12.54&0.245& 59.79 &&20.92&0.511&  60.19\\
			&WSTNN&19.96&0.676&103.24&&23.77&0.808&78.69 && 27.82&0.669&68.75&& 31.25&0.903&83.01\\ 
			&SMF-LRTC&17.40&0.327&461.25&&22.98&0.675&532.40&&29.69&0.902&601.66 &&33.10&0.944&638.61\\ 
			&DP3LRTC&24.55&0.873&1226.33&&28.14&0.933& 1340.42&&32.71&0.971& 1378.67&& 35.94&0.984&1347.87\\ 
			&DCT-IpST&23.43&0.774&121.29 &&27.94&0.799&125.86&& 33.06&0.955&124.13 &&36.36&0.976&121.36\\ 
			&DCT-WNN&24.25& 0.799&105.47&&29.19&0.911&110.56&& 34.18&  0.963&115.50&&37.29&0.979&118.03\\ \midrule
			
			&IpST&19.54&0.588&78.98 &&23.25&0.714&70.72 &&29.26&0.886& 63.287 &&33.17&0.945&53.48\\
			biking &TTNNL1&20.57&0.631&70.81 &&24.16&0.750&68.96 &&29.16&0.885& 66.37 &&32.25&0.933&67.24\\
			$240\times320\times30$&F-TNN&20.81&0.640&455.86 &&22.83& 0.701&457.11 &&26.15&0.809&463.97&& 28.77&0.878&464.76\\
			&TNN-3DTV&19.06&0.529&741.98 &&22.08&0.659&575.65&&25.88&0.804& 412.79&&28.89&0.884&324.97\\
			&TCTF&	7.44&0.011&15.01 &&8.23&0.016&36.89 && 12.86&0.181& 38.51 &&17.18&0.500& 39.14\\
			&WSTNN&20.49&0.660&106.80&&24.13&0.781&80.386 &&27.85&0.873&68.03&&30.76&0.921& 72.38\\ 
			&SMF-LRTC&17.21&0.417&459.85 &&23.46&0.683&530.40 && 30.03&0.884&594.52&& 33.29&0.932&615.38\\ 
			&DP3LRTC&25.17&0.829&1219.33&&28.45&0.898& 1217.42&&31.80&0.946& 1222.67&& 34.17&0.966&1211.87\\ 
			&DCT-IpST&21.79&0.665&121.81 &&25.13&0.771&125.62 &&29.62&0.884&123.93 &&32.46&0.930&125.49\\ 
			&DCT-WNN&22.44& 0.693&104.37 &&26.17&0.810&108.59 &&30.55&  0.905&115.97 &&33.23&0.942&115.84\\ \midrule
			
			&IpST&20.46&0.478&81.32&&23.40&0.635&64.18 && 27.30&0.810&54.44 &&30.27&0.895& 45.72\\
			baby &TTNNL1&20.64&0.499&80.29 &&23.67&0.658&70.88&&27.43&0.826&65.43&& 30.16&0.887& 67.48\\
			$240\times320\times30$&F-TNN&21.32&0.542&503.42 &&23.13& 0.631&454.69&&25.96&0.757&463.50&&28.16&0.835&462.16\\
			&TNN-3DTV&20.50&0.467&758.46 && 22.83&0.610&463.90&&25.90&0.758&336.95 &&28.35&0.845&266.91\\
			&TCTF&	8.19&0.011&51.14&& 9.29&0.020&59.54&& 13.23&0.179& 60.43 &&19.95&0.477&  60.94\\
			&WSTNN&21.48&0.577&100.42&&23.66&0.695&80.13&&26.42&0.811&61.08&&28.73&0.879& 50.98\\ 
			&SMF-LRTC&18.46&0.347&489.33&&23.68&0.601& 540.42&&28.28&0.810& 596.67&& 30.51&0.873&625.87\\ 
			&DP3LRTC&24.60&0.723&1208.33&&27.52&0.822& 1213.42&&31.05&0.909& 1216.67&& 33.44&0.944&1237.87\\ 
			&DCT-IpST&21.88&0.551&121.29&&24.71&0.689&123.93&& 27.92&0.813&124.60&&30.39&0.881&125.18\\ 
			&DCT-WNN&22.63& 0.592&102.43&&25.49&0.729&105.54&&28.68&  0.839&112.11&&31.13&0.898&112.27\\ \midrule
	\end{tabular}	}
\end{table*}


\begin{figure*}[!htb]
	\centering
	\includegraphics[width=1\textwidth]{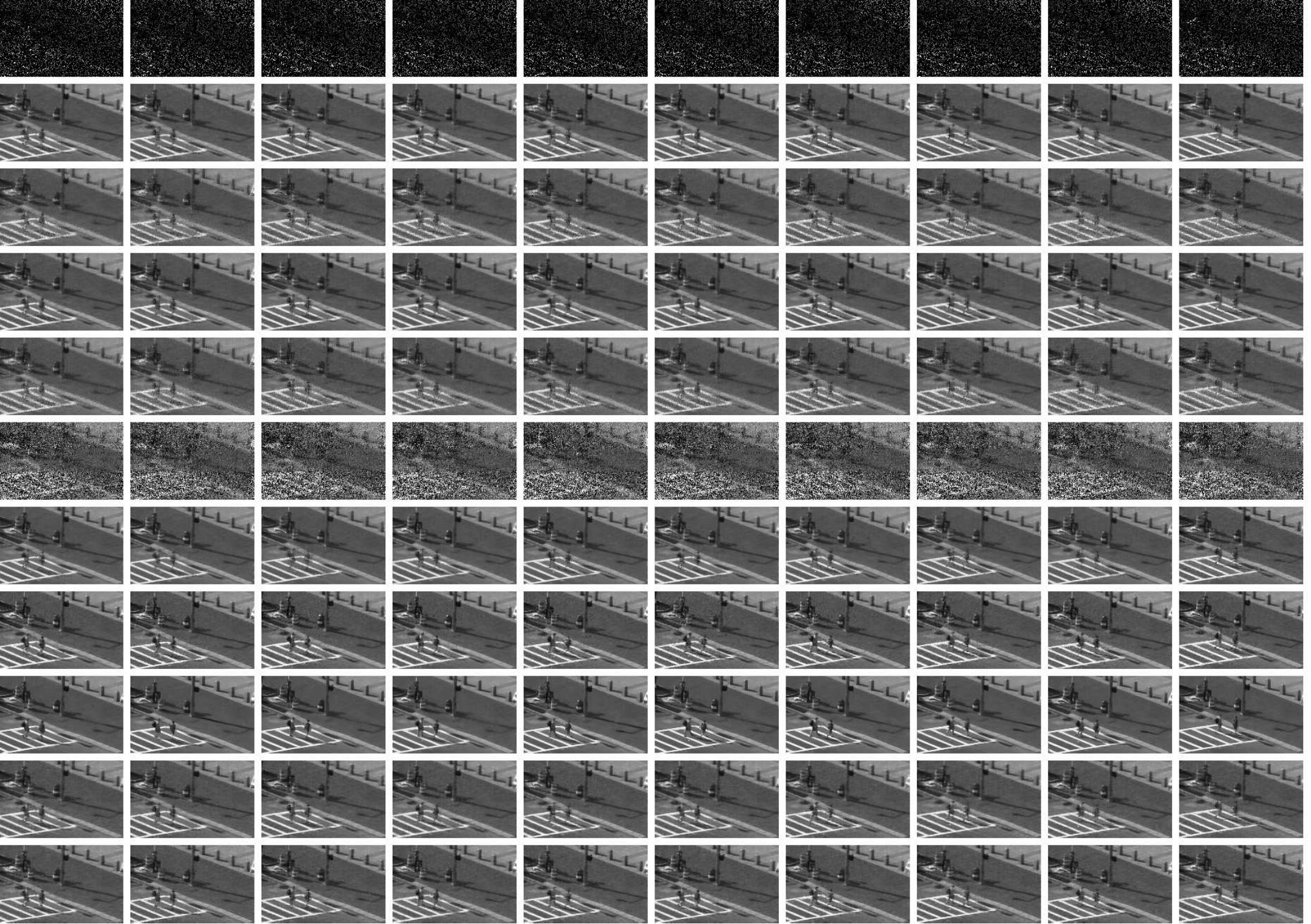}
	\caption{Visualizations of recovered `sidewalk' data sets by the ten algorithms. From left to right: ten consecutive observed frames with $20\%$ information. From top to bottom: observed frames, and frames recovered by IpST, TTNNL1, F-TNN, TNN-3DTV, TCTF, WSTNN, SMF-LRTC, DP3LRTC, DCT-IpST and DCT-WNN, respectively.}
	\label{video_visual_sidewalk}
\end{figure*}

{\subsection{Discussion}
\textbf{Parameters sensitivity}: Notice that there are some model parameters in our models, which are very important for the numerical performance of the models. However, it is well-known that setting model parameters is an incredibly difficult task in practice. Generally speaking, we can take constants empirically according to existing experiments in the literature. Here, we will conduct the numerical sensitivity of our models' parameters. Note that there is a triple of parameters in model \eqref{WTNN-question-1}, i.e., $\{(\alpha_1,\cdots,\alpha_N), \delta,\lambda\}$. According to \cite{SLH19,ZHZJM20}, we can easily determine the choice of $\alpha_i$'s.  So, we below are concerned with the sensitivity of $\delta$ and $\lambda$. Similarly, for model \eqref{question-1}, we focus on the sensitivity of $p$ and $\lambda$. In our experiments, we consider the ``\textbf{airport}'' video data set with ${\sf sr}=30\%$. The parameter sensitivity results are summarized in Fig. \ref{parameter_p}. It can be easily seen from Fig. \ref{parameter_p} that both models run relatively robust for $\lambda\in[0.01,1]$, which also provides us some suggestions on the settings for the above experiments.
\begin{figure}[!htb]
	\centering
	\includegraphics[width=.24\textwidth]{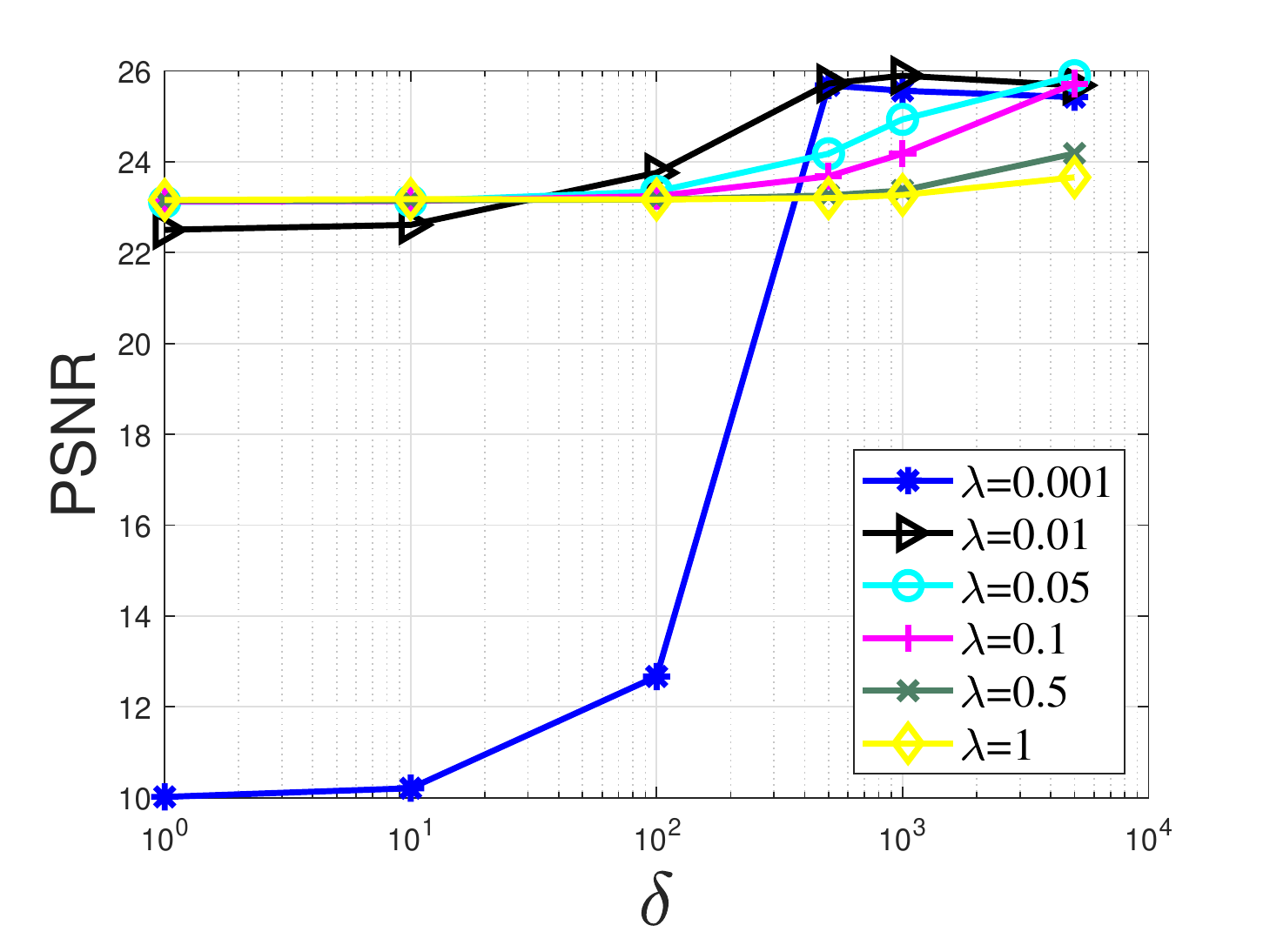}
	\includegraphics[width=.24\textwidth]{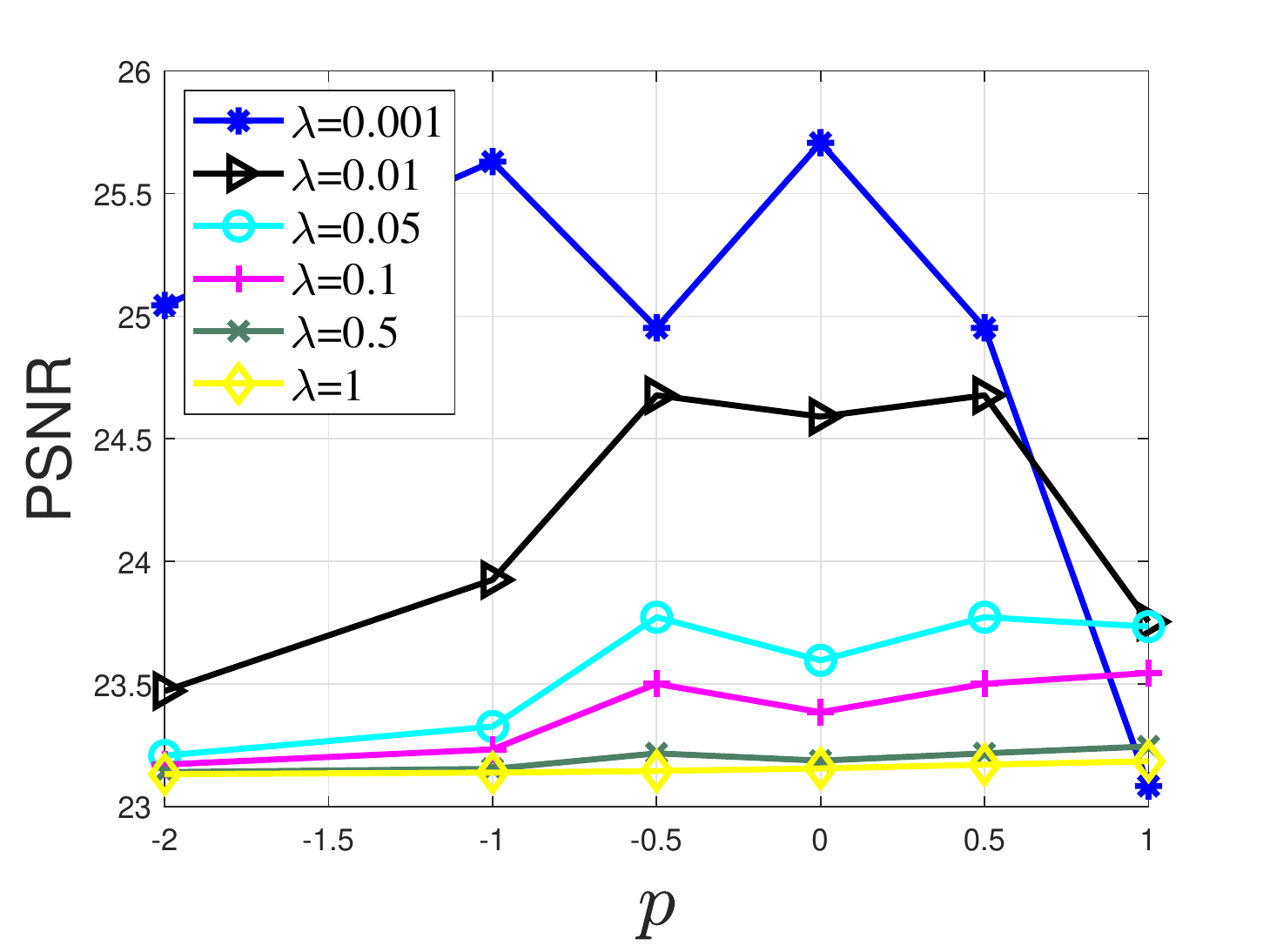}	
	\caption{Numerical sensitivity of model parameters $\lambda$, $\delta$, and $p$. The left figure corresponds to our first model \eqref{WTNN-question-1}. The right one corresponds to the second model \eqref{question-1}.}
	\label{parameter_p}
\end{figure}

\textbf{Model discussion}: Notice that the objective functions of our models consists of two nonsmooth parts. Ones will be interested in what will happen if we remove one of the nonsmooth parts. To answer this question, we shall investigate the numerical performance of our models when removing the low-rank regularization term or the sparse term. Here, we also focus on the ``\textbf{airport}''  video data set with ${\sf sr}=\{5\%,10\%,20\%,30\%\}$ in our experiments. In Fig. \ref{low_rank_sparse}, ``\textbf{only low-rank term}'' and ``\textbf{only sparse term}'' correspond to our models without sparse term and without low-rank regularization, respectively. From the plots in Fig. \ref{low_rank_sparse}, we observe that our models equipped with only one low-rank term works better than the cases with only one sparse term. However, the proposed models with sparse and low-rank terms have the best numerical performance, which also sufficiently supports the idea of this paper.
\begin{figure}[!htb]
	\centering
	\includegraphics[width=.24\textwidth]{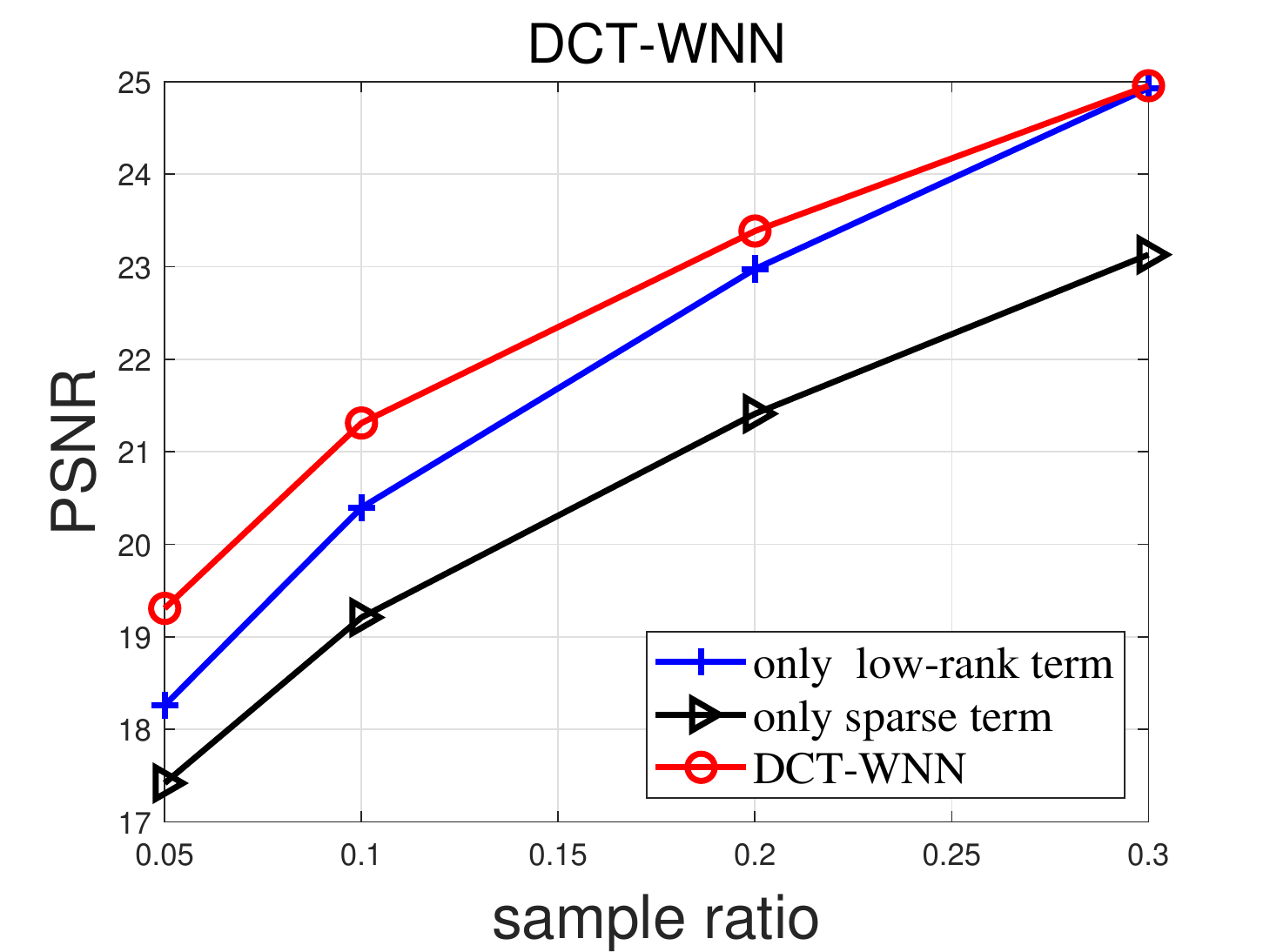}
	\includegraphics[width=.24\textwidth]{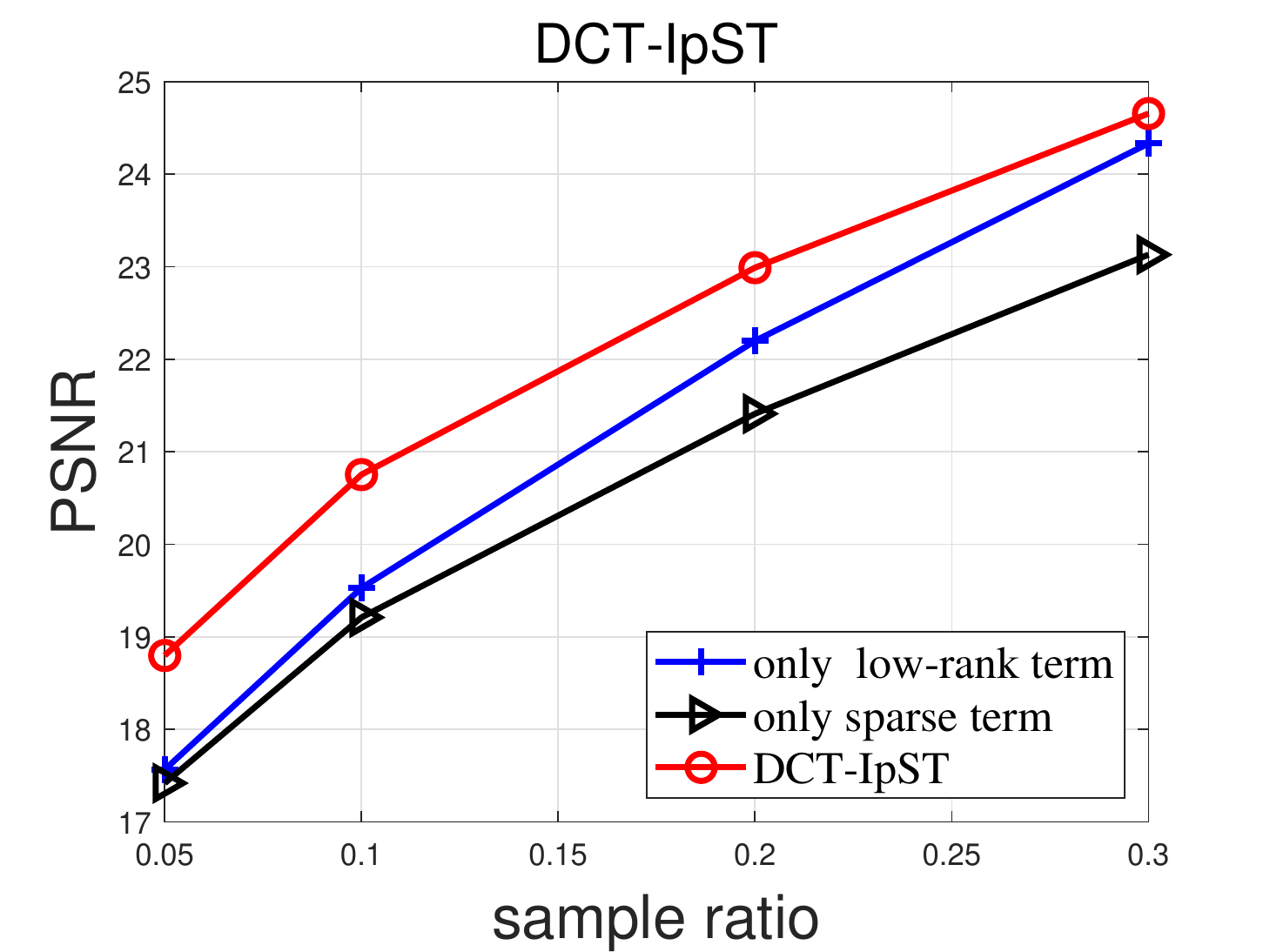}	
	\caption{Numerical performance of our models with only one low-rank objective or with only one sparse regularization term, respectively.}
	\label{low_rank_sparse}
\end{figure}
}

In summary, the above computational results show that the proposed DCT-WNN and DCT-IpST are powerful to complete different types of downsampled images and videos data sets. {Although our DCT-WNN and DCT-IpST cannot obtain the same inpainting quality of the state-of-the-art deep learning method DP3LRTC, DCT-WNN and DCT-IpST run much faster for large-scale data sets than DP3LRTC on a general personal computer. From the computing time and computer hardware perspectives, those model-driven approaches still have their own advantages. Moreover, we believe that our models could be further improved if some deep learning techniques are embedded in. We leave it as one of our future concerns.}

\section{Conclusion}\label{ConRemark}
In this paper, we introduced a unified {``sparse + low-rank''} tensor completion approach, which includes two structured optimization models for recovering images and videos from highly undersampled data. By introducing auxiliary variables, we gainfully separated the two nonsmooth terms appeared in the objective functions, thereby being of benefit for designing easily implementable algorithms for the underlying separable models. A series of computational experiments on RGB images and surveillance videos data sets demonstrated that our approach performs better than some state-of-the-art {model-driven} tensor completion approaches.

\bigskip
\noindent{\bf Declaration of Competing Interest}: The authors declared that they have no known competing financial interests or personal relationships that could have appeared to influence the work reported in this paper.

\section*{Acknowledgments}
{The authors are grateful to the editor and two anonymous referees for their close reading and valuable comments, which led to great improvements of the paper. They also would like to thank all authors of \cite{HLHS17,JLLS18,ZLLZ18,JNZH20,ZHZJM20,SLH19,ZHJZJM19,ZXJWN20} for sharing their code, especially thank Dr. Yufan Li and Mr. Zhiyuan Zhang for their kind help on our experiments.} C. Ling and H. He were supported in part by National Natural Science Foundation of China (Nos. 11971138 and 11771113) and Natural Science Foundation of Zhejiang Province (Nos. LY19A010019 and LD19A010002).



\end{document}